\documentclass{article}
\usepackage[margin=1in]{geometry}
\usepackage{microtype}
\usepackage{needspace}

\usepackage[table,dvipsnames]{xcolor}

\usepackage{amsmath}
\usepackage{amssymb}
\usepackage{amsthm}

\definecolor{darkgreen}{RGB}{0,100,0}

\usepackage{graphicx}
\usepackage{svg}
\usepackage{wrapfig}
\usepackage{subcaption}
\usepackage{booktabs}
\usepackage{siunitx}
\usepackage{tabularx}
\usepackage{array}

\usepackage{float}
\usepackage{placeins}

\usepackage{tikz}
\usepackage{tikz-cd}
\usetikzlibrary{positioning,calc,arrows.meta,fit}

\usepackage{multicol}

\usepackage[most]{tcolorbox}

\usepackage[numbers,sort&compress]{natbib}

\usepackage{aliascnt}

\usepackage{amsmath,amsfonts,bm,amssymb}

\def\eqref#1{equation~\ref{#1}}

\def\1{\bm{1}}

\DeclareMathAlphabet{\mathsfit}{\encodingdefault}{\sfdefault}{m}{sl}
\SetMathAlphabet{\mathsfit}{bold}{\encodingdefault}{\sfdefault}{bx}{n}

\newcounter{setting}[section]
\renewcommand{\thesetting}{\thesection.\arabic{setting}}

\newcounter{model}[section]
\renewcommand{\themodel}{\thesection.\arabic{model}}

\newcounter{initialization}[section]
\renewcommand{\theinitialization}{\thesection.\arabic{initialization}}

\usepackage[
    colorlinks=true,
    citecolor=MidnightBlue,
    linkcolor=MidnightBlue,
    urlcolor=MidnightBlue
]{hyperref}

\usepackage{cleveref}

\AddToHook{cmd/subsection/before}{\Needspace{8\baselineskip}}
\AddToHook{cmd/subsubsection/before}{\Needspace{6\baselineskip}}

\newcommand{\experimentfigure}[2][]{%
  \IfFileExists{#2}{%
    \includegraphics[#1]{#2}%
  }{%
    \fbox{%
      \parbox[c][0.22\textheight][c]{0.88\linewidth}{%
        \centering
        Figure file not found:\\[0.5em]
        \texttt{\detokenize{#2}}%
      }%
    }%
  }%
}

\newcounter{datasetting}
\renewcommand{\thedatasetting}{\arabic{datasetting}}
\crefname{datasetting}{Data Setting}{Data Settings}
\Crefname{datasetting}{Data Setting}{Data Settings}

\newcounter{trainingsetup}
\renewcommand{\thetrainingsetup}{\arabic{trainingsetup}}
\crefname{trainingsetup}{Training Setup}{Training Setups}
\Crefname{trainingsetup}{Training Setup}{Training Setups}

\newcounter{initializationsetting}
\renewcommand{\theinitializationsetting}{\arabic{initializationsetting}}
\crefname{initializationsetting}{Initialization}{Initializations}
\Crefname{initializationsetting}{Initialization}{Initializations}

\tcbset{
  theorysetting/.style={
    enhanced,
    breakable,
    boxrule=0.8pt,
    arc=2pt,
    left=7pt,
    right=7pt,
    top=5pt,
    bottom=5pt,
    before skip=8pt,
    after skip=8pt,
    fonttitle=\bfseries,
    before upper={%
      \setlength{\parskip}{3pt}%
      \setlength{\abovedisplayskip}{6pt plus 2pt minus 2pt}%
      \setlength{\belowdisplayskip}{6pt plus 2pt minus 2pt}%
    },
  }
}

\NewDocumentEnvironment{datasetting}{O{}}
{
  \refstepcounter{datasetting}
  \begin{tcolorbox}[
    theorysetting,
    colback=blue!3,
    colframe=blue!45!black,
    title={
      Data Setting~\thedatasetting
      \if\relax\detokenize{#1}\relax\else\ --- #1\fi
    }
  ]
}
{
  \end{tcolorbox}
}

\NewDocumentEnvironment{trainingsetup}{O{}}
{
  \refstepcounter{trainingsetup}
  \begin{tcolorbox}[
    theorysetting,
    colback=green!3,
    colframe=green!40!black,
    title={
      Training Setup~\thetrainingsetup
      \if\relax\detokenize{#1}\relax\else\ --- #1\fi
    }
  ]
}
{
  \end{tcolorbox}
}

\NewDocumentEnvironment{initializationsetting}{O{}}
{
  \refstepcounter{initializationsetting}
  \begin{tcolorbox}[
    theorysetting,
    colback=orange!4,
    colframe=orange!55!black,
    title={
      Initialization~\theinitializationsetting
      \if\relax\detokenize{#1}\relax\else\ --- #1\fi
    }
  ]
}
{
  \end{tcolorbox}
}

\newtheorem{theorem}{Theorem}

\newaliascnt{lemma}{theorem}
\newtheorem{lemma}[lemma]{Lemma}
\aliascntresetthe{lemma}

\newaliascnt{corollary}{theorem}
\newtheorem{corollary}[corollary]{Corollary}
\aliascntresetthe{corollary}

\crefname{theorem}{Theorem}{Theorems}
\Crefname{theorem}{Theorem}{Theorems}

\crefname{lemma}{Lemma}{Lemmas}
\Crefname{lemma}{Lemma}{Lemmas}

\crefname{corollary}{Corollary}{Corollaries}
\Crefname{corollary}{Corollary}{Corollaries}

\crefname{lemma}{Lemma}{Lemmas}

\begin{document}

\begin{center}

\makebox[\textwidth][c]{%
  \parbox[t]{0.84\paperwidth}{%
    \centering
    \Large\bfseries
    Revenge of Monosemanticity: Neuron Specialization\\[-0.05em]
    as a New Form of Feature Learning in MLPs
    \par
  }%
}

\vspace{1.1em}

{\large
\begin{tabular}{cc}
Amirhesam Abedsoltan$^{1}$
&
Enric Boix-Adsera$^{2}$ \\[0.35em]
Fivos Kalogiannis$^{1}$
&
Mikhail Belkin$^{3,1}$
\end{tabular}
}

\vspace{0.9em}

{\small
$^{1}$Department of Computer Science and Engineering, UC San Diego
\par
\makebox[\linewidth][c]{%
$^{2}$Department of Statistics and Data Science, The Wharton School, University of Pennsylvania}
\par
$^{3}$Halıcıoğlu Data Science Institute, UC San Diego
\par}

\end{center}

\vspace{1.0em}

\begin{abstract}
Understanding how neural networks learn and organize features is central to understanding their behavior.
Much existing theory of feature learning has focused on the emergence of a global low-dimensional representation.
We show that this picture is incomplete. In regression problems with clustered data,
we demonstrate that multilayer perceptrons (MLPs) develop specialized neurons: individual
neurons become strongly aligned with specific features relevant
to a particular region of the input space. A substantial fraction of these neurons exhibit monosemanticity: they activate predominantly on inputs from the corresponding cluster.
Thus, rather than learning a single global low-dimensional representation, MLPs learn a collection of local low-dimensional representations.
 Furthermore, we show that this ability to specialize gives MLPs a provable data-efficiency advantage over feature-learning methods based on a global low-dimensional representation.
\end{abstract}

\vspace{1.0em}

\section{Introduction}

A major advance in understanding neural networks was the recognition that neural networks can perform feature learning by exploiting low-dimensional structure in the target, where prediction depends on only a small number of directions in the input space. In such settings, neural networks can adapt their representations to identify these predictive directions~\citep{bach2017breaking,soltanolkotabi2017learning,yehudai2019power,wei2019regularization,ghorbani2019limitations,ghorbani2020outperform,malach2021quantifying,abbe2022merged,abbe2023sgd,ba2022highdimensional,damian2022neural,dandi2024giant,bruna2025survey}.

In this setting, the response takes the form
\[
f(x)
=
g\!\left(U^{\top}x\right),
\qquad
U\in\mathbb{R}^{d\times r},
\qquad
r\ll d,
\]
where \(g:\mathbb{R}^r\to\mathbb{R}\) is the \emph{link function}. Although \(x\in\mathbb{R}^d\), the response
depends only on the \(r\)-dimensional predictive subspace spanned by the
columns of $U$.
During training, a neural network can learn features aligned with that predictive subspace, effectively reducing the dimensionality of the problem from the ambient dimension $d$ to the intrinsic dimension $r$.

This perspective views feature learning primarily as a mechanism for discovering low-dimensional predictive structure: supervision reshapes the representation so that learning is governed by the intrinsic dimension of the task rather than the ambient dimension of the input.

In our work, we identify a qualitatively different mechanism of feature learning
in multilayer perceptrons (MLPs) that is not tied to recovering a
single global low-dimensional predictive subspace.

\Needspace{17\baselineskip}
\noindent\begin{minipage}[t]{0.605\textwidth}
\vspace{0pt}
\setlength{\parskip}{0.3\baselineskip}
In this setting, the data are drawn from multiple clusters, and each cluster can
have its own predictive directions and link function. Thus, prediction may be
low-dimensional within each cluster even when there is no low-dimensional
global structure. The diagram on the right illustrates the \emph{single-index}
case, where each cluster has one predictive direction.

As we will see, MLPs can simultaneously discover the cluster structure and learn the predictive functions relevant to each cluster. This uncovers a form of feature learning that goes beyond recovering a single global low-dimensional predictive subspace.
\end{minipage}\hfill
\begin{minipage}[t]{0.38\textwidth}
    \vspace{-3pt}
    \raggedleft
    \includegraphics[width=\linewidth]
    {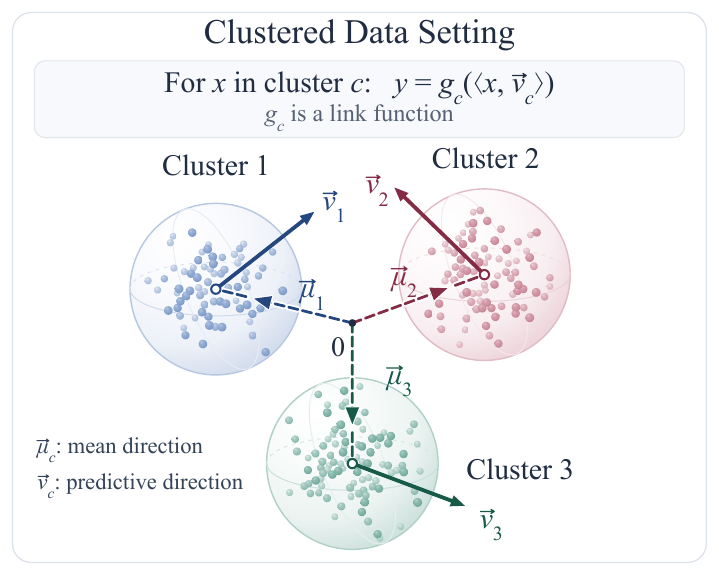}
\end{minipage}
\par\medskip

\noindent\textbf{Main contributions.}

We show that, at inference time, a trained MLP both:

\noindent\textbf{(a)  routes an input} to its cluster.

\noindent\textbf{(b)  predicts its response} using the learned cluster-specific prediction function.

MLPs do this through neuron \textbf{specialization}: during training, a substantial fraction of neurons
develop weights aligned with the subspace spanned by a single cluster's mean
and predictive directions.

\noindent Additionally, many of these specialized neurons exhibit \textbf{monosemanticity}: these neurons  primarily activate on inputs
from the cluster to which they are specialized, and not on other clusters.

This mechanism resembles a mixture of experts. However, it
emerges in a standard MLP without cluster labels or additional structure,
such as an explicit routing module.

We show empirically that this specialization leads to increasingly favorable sample-complexity 
scaling as the number of clusters grows, allowing MLPs to maintain strong performance even when 
no single global low-dimensional structure exists. 
While specialization occurs for standard ReLU and GELU~\citep{hendrycks2016gaussian}, 
modern architectures with multiplicative gating, such as ReGLU and SwiGLU~\citep{shazeer2020glu}, 
can further improve sample complexity
in this setting.

We also establish a provable advantage of MLPs over methods based on a global low-dimensional representation. In particular, in our theoretical setting, as the number of clusters $K$ increases, the MLP has the ability to achieve optimality (consistency) with only polynomially many training samples, $n=\mathrm{poly}(K)$. In contrast, to achieve consistency, methods based on a single global low-dimensional representation have sample complexity which is super-polynomial in $K$.

\subsection{Discussion}
The form of feature learning identified in this work differs from classical feature learning, where the main goal
is to recover a single global low-dimensional predictive subspace. In such
settings, methods designed for learning low-dimensional features can perform as well as or
better than MLPs. A prominent example is the Recursive Feature Machine (RFM), a supervised kernel-based method that uses input-response pairs to learn a global low-dimensional predictive subspace~\citep{radhakrishnan2024mechanism,radhakrishnan2025linear}.
RFM is therefore a particularly informative comparison: like an MLP, it learns
features from the data rather than operating with a fixed representation.

To study this phenomenon systematically, we consider a mixture of single-index or
multi-index models. Suppose the data are drawn from \(K\) clusters. For an input
\(x\in\mathbb{R}^d\) belonging to cluster \(c\in[K]\), the target is
\[
f(x)=g_c\!\left(U_c^\top x\right),
\qquad
U_c\in\mathbb{R}^{d\times r}.
\]
Thus, prediction is $r$-dimensional within each cluster, but both the
predictive subspace spanned by the columns of \(U_c\) and the link function
\(g_c\) may vary across clusters. The case \(r=1\) is a single-index model,
whereas \(r>1\) gives a multi-index model.

Note that the standard single- or multi-index model corresponds to \(K=1\). When
\(K>1\), however, the column spaces of \(U_1,\ldots,U_K\) may not generally admit a common
low-dimensional subspace. Hence, prediction can be low-dimensional locally within
each cluster even though there is no useful global low-dimensional predictive
subspace to recover.

\subsection{Prior Work}
\label{sec:prior_work}

\paragraph{Feature learning of global low-dimensional structure.}
A large body of work studies feature learning for multi-index
models. These works have shown that neural-network models can adapt to an unknown
low-dimensional subspace of the inputs, succeeding where fixed-kernel methods are sample-inefficient
\citep{bach2017breaking,soltanolkotabi2017learning,
yehudai2019power,ghorbani2020outperform}. Recently, the literature has characterized how gradient-based training succeeds in recovering this subspace for isotropic data
\citep{abbe2022merged,abbe2023sgd,damian2022neural,ba2022highdimensional,dandi2024giant,damian2025generative,bruna2025survey}, and in the related setting of classifying Gaussian-mixture data with a constant number of clusters
\citep{pmlr-v139-refinetti21b,arous2024highdimensional}.

However, in our work, we prove that neuron specialization is a different form of feature learning, allowing MLPs to efficiently tackle settings where a single global low-dimensional representation is insufficient.

\paragraph{Neuron specialization.} The emergence of
specialized neurons has been observed in several settings. For instance, training neural networks to do modular arithmetic leads to specialized neurons, each representing different Fourier components
\citep{nanda2023progress,gromov2023grokking,morwani2024feature,
he2026mechanism,he2026neural}. When learning XOR-type targets, it has also been shown that neurons specialize to four clusters rather than being distributed evenly across a predictive low-dimensional subspace \citep{frei2023random,glasgow2024sgd}. In teacher--student settings, student neurons have also been shown to align frequently with teacher neurons \citep{tian2020student,oostwal2021hidden,zhu2025gradient}, although this depends on initialization \citep{jarvis2025theory}.

These works mainly study specialization to globally relevant features or to cluster directions that themselves determine the target. In our work, neurons specialize to locally predictive features. Thus, an MLP must not only learn the predictive directions, but also preserve their association with the clusters in which they are relevant. Moreover, these works do not isolate a sample-complexity advantage of specialization over an adaptive feature-learning method based on a single global representation, such as the Recursive Feature Machine \citep{radhakrishnan2024mechanism}. We establish such an advantage: when cluster-specific predictive directions collectively span a high-dimensional space, neuron specialization allows an MLP to preserve these local feature-cluster associations and achieve better sample complexity.

\paragraph{Monosemanticity.}
Interpretability works on language models have found that some individual neurons are monosemantic, responding primarily to a single interpretable concept or pattern, while others are \emph{polysemantic}, responding to multiple distinct concepts or patterns \citep{bills2023language,elhage2022solu}. More recent work
has argued that monosemantic features need not align with individual neurons,
and has used sparse dictionary learning to recover more monosemantic feature
directions from neural representations
\citep{bricken2023monosemanticity,templeton2024scaling}. Our work shows that,
at least in simple MLPs, individual neurons can exhibit both specialization and monosemanticity.
We also show that neuron specialization can improve sample efficiency
compared with methods based on a global low-dimensional representation.
\paragraph{Mixture-of-experts models.} A complementary line of work studies specialization and latent-cluster recovery in mixture-of-experts (MoE) architectures. An MoE router can learn cluster-center features that partition a clustered classification problem into simpler subproblems handled by different experts \citep{chen2022towards}. Related theoretical and empirical results show that a learned router can route inputs according to latent clusters \citep{dikkala-etal-2023-benefits}. Particularly close to our statistical setting, recent work studies nonlinear regression with an underlying cluster structure of single-index models and shows that an MoE trained by SGD can detect the latent organization and divide the task into cluster-specific subproblems; under its assumptions, a vanilla neural network does not detect this organization within the same polynomial complexity regime \citep{pmlr-v267-kawata25a}.

These works build specialization into the architecture through an explicit router and separate experts. In contrast, we study a standard MLP with no explicit routing or expert decomposition, and show that both the implicit clustering and the corresponding local predictive-feature specialization can emerge through the specialization of individual neurons.

\subsection{Paper Structure}
The remainder of the paper studies neuron specialization in MLPs empirically and theoretically. Section~\ref{sec:prelim} introduces the setting. Section~\ref{sec:Experiments} shows empirically that MLPs jointly learn cluster structure and cluster-specific predictive functions, achieving increasingly favorable sample complexity over global low-rank methods as the number of clusters grows. We also show that gated activations such as ReGLU and SwiGLU improve sample complexity over ReLU and GELU. Finally, Section~\ref{sec:specialization-theory} theoretically analyzes the emergence of neuron specialization in simplified settings.

\section{Preliminaries}\label{sec:prelim}

Throughout our synthetic experiments and theoretical analysis, we consider
data drawn from a Gaussian mixture model with \(K\) clusters. The cluster
index \(c\) is drawn uniformly from \(\{1,\ldots,K\}\). For each cluster \(c\),
let \(\mu_c\in\mathbb{R}^d\) denote its mean,
\(\Sigma_c\in\mathbb{R}^{d\times d}\) its covariance matrix, and
\(v_c\in\mathbb{R}^d\) its unit predictive direction.
The means lie in the \emph{routing subspace}, and the predictive directions
lie in the orthogonal \emph{predictive subspace}. Each cluster's
\emph{routing--predictive span} is \(\operatorname{span}\{\mu_c,v_c\}\).
Conditioned on cluster \(c\), the covariates \(x\in\mathbb{R}^d\) are distributed as
\[
x \mid c \sim \mathcal{N}(\mu_c, \Sigma_c).
\]
The majority of our experiments, as well as the
theoretical analysis, use a cluster-specific single-index model with additive
Gaussian noise:
\[
y = g_c\!\left(\langle x, v_c\rangle\right) + \varepsilon,
\qquad
\varepsilon \sim \mathcal{N}(0,\sigma^2),
\]
where \(\varepsilon\) is independent of \(x\) and \(c\), and
\(g_c:\mathbb{R}\to\mathbb{R}\) is a nonlinear link function that may vary
across clusters.

In Section~\ref{sec:clustered-multi-index}, we also consider a multi-index
setting in which each cluster \(c\) has three orthonormal predictive
directions \(v_{c,1},v_{c,2},v_{c,3}\) in the predictive subspace. Its
routing--predictive span then includes all three directions together with
\(\mu_c\). The same cluster-specific link
function is applied along each direction and the resulting components are
aggregated:
\[
y
=
\frac{1}{\sqrt 3}
\sum_{r=1}^{3}
g_c\!\left(\langle x,v_{c,r}\rangle\right)
+\varepsilon.
\]

Our experiments use well-separated clusters.

Throughout, we use mean squared error (MSE) as the regression loss,
\[
\mathcal{L}(f)
=
\mathbb{E}\!\left[(f(x)-y)^2\right],
\]
where \(f:\mathbb{R}^d\to\mathbb{R}\) denotes the learned predictor.

\paragraph{Activations.}
We use ReLU and GELU~\citep{hendrycks2016gaussian} as standard activations,
and ReGLU and SwiGLU~\citep{shazeer2020glu} as gated activations.
Our MLPs have one hidden layer of width \(m\), with output
\(f(x)=\sum_{j=1}^m a_jh_j(x)+b_{\mathrm{out}}\),
with hidden units
\begin{align}
    \text{ReLU/GELU:}\quad h_j(x)
    &=\phi(w_j^\top x+b_j),
    \label{eq:standard-hidden-activation}\\
    \text{ReGLU/SwiGLU:}\quad h_j(x)
    &=\phi(w_{g,j}^\top x+b_{g,j})(w_{v,j}^\top x+b_{v,j}).
    \label{eq:gated-hidden-activation}
\end{align}
The vectors \(w_j,w_{g,j},w_{v,j}\in\mathbb R^d\) are the standard, gate,
and value weights, respectively. The \(b\)'s are biases and \(a_j\) is the
readout weight.
Here \(\phi\) is the corresponding activation for standard units,
ReLU for ReGLU, and \(\operatorname{SiLU}(t)=t/(1+e^{-t})\) for SwiGLU.

\section{Experiments}
\label{sec:Experiments}
In Section~\ref{sec:coordinate-specialization}, we study neuron specialization
through weight alignment with cluster-specific routing--predictive spans and
monosemanticity through activation concentrated on the cluster to which a neuron is specialized.
Sections~\ref{sec:local-routing} and~\ref{sec:clustered-multi-index} cover
neuron specialization and the sample efficiency of MLPs on single- and
multi-index models, with comparisons to global and cluster-specific kernel
predictors.
Section~\ref{sec:learned-gating} focuses on local prediction using cluster
structure recovered from trained MLP representations.

\subsection{A Substantial Fraction of MLP Neurons Are Specialized and Monosemantic}
\label{sec:coordinate-specialization}

\paragraph{Main findings.}
A substantial fraction of active neurons exhibit \emph{specialization}: their
weights align with the span of a single cluster's mean and predictive direction.
In this setup, each cluster's mean \(\mu_c\) and predictive direction
\(v_c\) lie along individual coordinate axes. Their span therefore consists
of the two matched routing and predictive coordinates.
Figure~\ref{fig:coordinate-specialization-hermite3-neurons} illustrates
representative specialized neurons whose weights are concentrated on these two coordinates
and small elsewhere. Using the specialization score
(defined below), the specialization column of Table~\ref{tab:coordinate-specialization}
quantifies the percentage of active neurons that specialize.

We show that a substantial fraction of specialized neurons are
\emph{monosemantic}: they activate predominantly on inputs from the
cluster to which they are specialized (Table~\ref{tab:coordinate-specialization} and
Fig.~\ref{fig:k25-axis-q2}). This finding is consistent across activations.
To quantify this firing exclusivity, we
introduce a monosemanticity score (defined below): the fraction of a neuron's total activation energy
attributable to the cluster to which it is specialized. These results connect
specialization to monosemanticity, while showing that specialization alone does
not guarantee monosemanticity.

\begin{figure}[!htb]
    \centering
    \includegraphics[width=\linewidth]
    {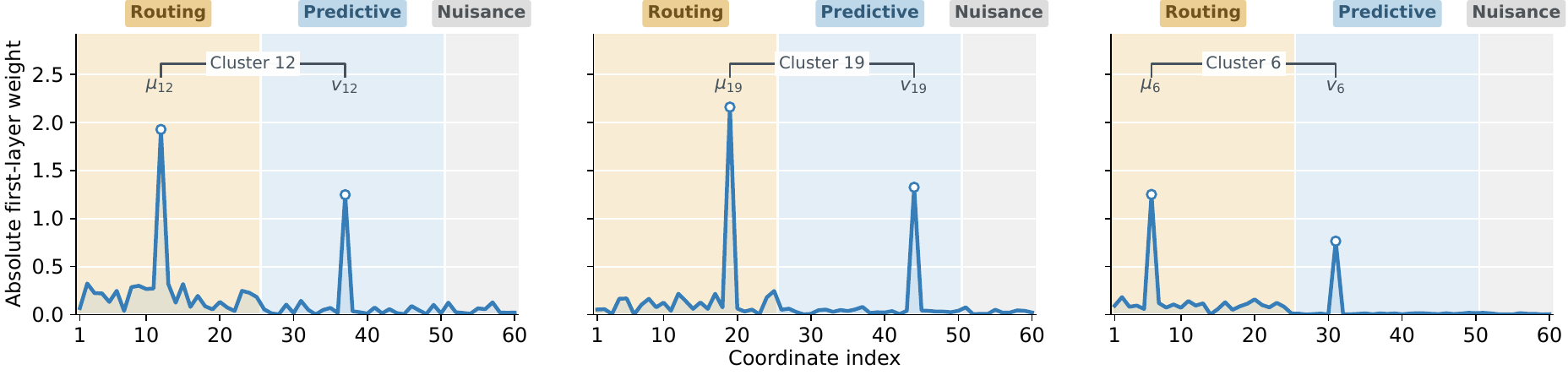}
    \caption{\textbf{Representative specialized neurons.}
    Representative ReLU weights with 25 clusters. Cluster means and predictive
    directions are coordinate axes, so concentration on the two matched
    coordinates illustrates specialization to a single cluster. Setup and selection:
    Appendix~\ref{app:coordinate-specialization-details}; all activations:
    Appendix Fig.~\ref{fig:k25-all-activation-neurons}.}
    \label{fig:coordinate-specialization-hermite3-neurons}
\end{figure}

\begin{table}[!htbp]
    \centering
    \small
    \setlength{\tabcolsep}{3.5pt}
    \renewcommand{\arraystretch}{1.12}
    \begin{tabular}{lcccc}
        \toprule
        & & \multicolumn{1}{c}{Specialization measure}
        & \multicolumn{2}{c}{Monosemanticity measures} \\
        \cmidrule(lr){3-3}\cmidrule(lr){4-5}
        \shortstack[l]{Model /\\baseline}
        & \shortstack{Active neurons\\(out of 4096)}
        & \shortstack{Score \(S_j\geq0.9\)\\(\% active) \(\uparrow\)}
        & \shortstack{Mean monosemanticity\\score (\(M_j\)) \(\uparrow\)\\(over specialized neurons)}
        & \shortstack{Neurons with \(M_j>1/2\)\\(\% specialized) \(\uparrow\)} \\
        \midrule
        ReLU
        & \(3879\pm30\)
        & \(15.71\pm0.66\)
        & \(0.654\pm0.012\)
        & \(65.74\pm1.69\) \\
        GELU
        & \(431\pm19\)
        & \(77.66\pm0.49\)
        & \(0.442\pm0.014\)
        & \(45.74\pm2.02\) \\
        ReGLU
        & \(406\pm3\)
        & \(44.30\pm0.69\)
        & \(0.945\pm0.006\)
        & \(95.57\pm1.04\) \\
        SwiGLU
        & \(249\pm16\)
        & \(57.68\pm4.90\)
        & \(0.665\pm0.009\)
        & \(71.80\pm2.72\) \\
        \midrule
        Baseline & --- & \(0\) & \(0.04\) & \(0\) \\
        \bottomrule
    \end{tabular}
    \caption{\textbf{Specialization and monosemanticity.}
    Measures of neuron specialization and monosemanticity across activations.
    Values are mean \(\pm\) standard deviation across three runs varying only
    random initialization and optimization minibatch sampling, with the data fixed.
    Baselines use random weight orientations for
    specialization (maximum mean across activations) and uniform firing for
    monosemanticity. Setup and baseline details: Appendices~\ref{app:coordinate-specialization-details}
    and~\ref{app:random-orientation-specialization}.}
    \label{tab:coordinate-specialization}
\end{table}

\paragraph{Data setting.}
We use \(K=25\) balanced clusters in \(d=60\). The routing subspace consists of the first 25 coordinates,
the predictive subspace of the next 25, and the nuisance coordinates of the
last 10. For cluster \(c\), the mean \(\mu_c\) is \(1\)
in coordinate \(c\) and \(0\) elsewhere. The predictive direction \(v_c\)
is the unit vector along coordinate \(25+c\). We draw
\[
    X\mid C=c\sim\mathcal N\!\left(\mu_c,
        \operatorname{diag}(0.125^2I_{25},I_{35})\right),
    \qquad Y=g_c(X_{25+c})+\varepsilon.
\]
Within each cluster, coordinates are independent: routing coordinates have
standard deviation \(0.125\), and all other coordinates have standard
deviation \(1\). The response depends only on coordinate \(25+c\), with
independent noise \(\varepsilon\sim\mathcal N(0,0.02^2)\).
The links are normalized Hermite-2, normalized Hermite-3, sine, or hyperbolic
tangent.

\paragraph{Specialization (weight alignment).}
We measure alignment with the span of a cluster's mean and predictive
direction, \(\operatorname{span}\{\mu_c,v_c\}\). In our specialization
experiments, these directions are orthogonal and unit length. Let \(w_j\)
denote neuron \(j\)'s weight vector. For \(w_j\neq0\), we define
\begin{equation}
    S_{j,c}=
    \frac{\sqrt{\langle w_j,\mu_c\rangle^2+
                       \langle w_j,v_c\rangle^2}}{\|w_j\|_2}.
    \label{eq:matched-subspace-specialization}
\end{equation}
For a standard neuron, \(S_{j,c}\) is the cosine of the smallest angle
between its weight vector and this cluster's span. The neuron's
\emph{specialization score} is
\[
    S_j=\max_c S_{j,c}.
\]
% A score of \(0.9\) means that at least \(81\%\) of the squared weight norm
% lies in that cluster's routing--predictive span. 

% The two components may have
% different magnitudes and signs; neither is required to carry substantial
% weight on its own.

For ReGLU and SwiGLU, we evaluate the same expression for the gate weight
\(w_{g,j}\) and the value weight \(w_{v,j}\), and take the smaller of the
two scores as \(S_{j,c}\) for each cluster. Thus both branches must align
with the same cluster's span. We call an active neuron geometrically
specialized when \(S_j\geq0.9\). Active neurons form the smallest set
accounting for \(99.9\%\) of total importance
(Appendix~\ref{app:shared-specialization-analysis}). In the two single-index
experiments, retaining these neurons with hidden features frozen and
refitting only the readout largely recovers full-model test performance,
supporting this selection criterion
(Appendix~\ref{app:active-neuron-readout-refit}).

\paragraph{Monosemanticity.}
Among specialized neurons, we measure how many are monosemantic, meaning
they activate predominantly on inputs from the cluster to which they are specialized.
Let \(h_j(x)\) denote neuron \(j\)'s activation
(Eqs.~\ref{eq:standard-hidden-activation} and \ref{eq:gated-hidden-activation}).
Its \emph{activation energy within cluster \(r\)} is
\(\mathbb E[h_j(X)^2\mid C=r]\), the mean squared activation on inputs
from that cluster. Let \(c\) denote the cluster to which neuron \(j\) is specialized.
We divide the activation energy within cluster \(c\) by the sum across all clusters to define the
\emph{monosemanticity score}
\begin{equation}
    M_j=
    \frac{\mathbb E[h_j(X)^2\mid C=c]}
         {\sum_{r=1}^{K}\mathbb E[h_j(X)^2\mid C=r]}.
    \label{eq:activation-energy-concentration}
\end{equation}
For equally likely clusters, \(M_j\) is the
fraction of the neuron's total activation energy contributed by cluster \(c\).
Under uniform firing, \(M_j=1/K\), where \(K\) is the total number of clusters.
Values near one indicate nearly exclusive firing on cluster \(c\).
The tables report the mean monosemanticity score over specialized neurons
and the percentage with \(M_j>1/2\), meaning that more than half of their
activation energy comes from the cluster to which they are specialized.

\begin{figure}[!htb]
    \centering
    \includegraphics[width=0.74\linewidth]
    {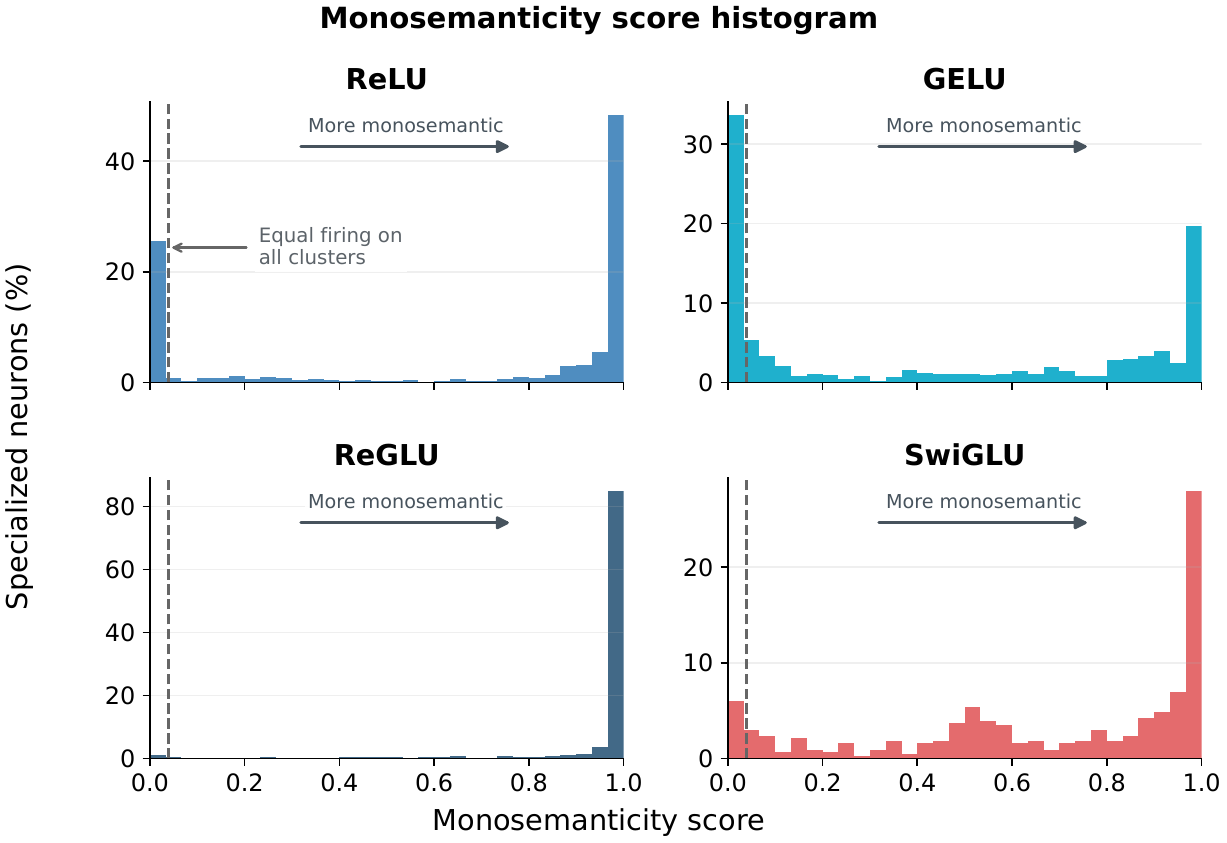}
    \caption{\textbf{Monosemanticity of specialized neurons in MLPs.}
    Distribution of monosemanticity scores \(M_j\) among specialized neurons.
    A score of \(0\) means no activation energy comes from the cluster to which
    the neuron is specialized; a score of \(1\) means all activation energy
    comes from that cluster. The dashed line at \(M_j=1/25=0.04\) marks
    equal activation energy across all 25 clusters.
    Experimental details:
    Appendix~\ref{app:coordinate-specialization-details}.}
    \label{fig:k25-axis-q2}
\end{figure}

\paragraph{Details and additional results.}
Appendix~\ref{app:coordinate-specialization-details} gives the training
protocol, link assignment, and predictive-performance results.
To assess how much specialization could arise by chance, we independently
randomize each active neuron's weight directions within the routing and
predictive subspaces, preserving the norm in each subspace and leaving
nuisance weights unchanged. This baseline yields effectively zero
specialization at the same threshold (\(S_j\geq0.9\))
(Appendix~\ref{app:random-orientation-specialization}).
\FloatBarrier

\subsection{Non-orthogonal Means and Predictive Directions}
\label{sec:local-routing}

We next extend our setting beyond the
axis-aligned setting of Section~\ref{sec:coordinate-specialization}.
Cluster means \(\mu_c\) and predictive directions \(v_c\) are no longer
restricted to individual coordinate axes: cluster means can be non-orthogonal
to one another within the routing subspace, and predictive directions can be
non-orthogonal to one another within the predictive subspace. The routing subspace
carries information about cluster identity and remains orthogonal to the
predictive subspace.

\paragraph{Main findings.}
\begin{enumerate}
\item \textbf{Specialization and monosemanticity.}
Specialization persists beyond the axis-aligned setting: across activations,
a substantial fraction of active neurons specialize to a single cluster
(Table~\ref{tab:single-index-specialization}).
Many of these specialized neurons also activate predominantly on inputs from
the cluster to which they are specialized
(Table~\ref{tab:single-index-specialization} and
Fig.~\ref{fig:q2-histograms-single-index}). Specialization thus remains
closely associated with monosemanticity, without guaranteeing monosemanticity.

\item \textbf{Sample efficiency.}
With one or two clusters, global RFM performs comparably to MLPs and
substantially outperforms the global Laplace kernel, which does not learn
features (Fig.~\ref{fig:local-sample-complexity}). Thus, learning a single
global feature representation is effective in this regime. However, as the
number of clusters increases, the performance gap between global RFM and
global Laplace shrinks, while MLPs continue to outperform both methods.

We also compare with \emph{local methods}, which use true cluster identities
to fit a separate Laplace kernel or RFM predictor within each cluster.
Local Laplace uses cluster information without learning features, whereas
local RFM additionally learns a separate feature metric for each cluster.
With sufficient training data, MLPs outperform even local Laplace and are
noticeably closer to local RFM, despite receiving no cluster labels.
These results demonstrate a form of feature learning that goes beyond
recovering a single global low-dimensional predictive subspace.
Gated activations further improve sample efficiency, bringing MLPs closer
to local RFM.

\end{enumerate}

% Requires: booktabs, graphicx
\begin{table}[!htb]
    \centering
    \small
    \setlength{\tabcolsep}{3.5pt}
    \renewcommand{\arraystretch}{1.10}
    \begin{tabular}{lcccc}
        \toprule
        & & \multicolumn{1}{c}{Specialization measure}
        & \multicolumn{2}{c}{Monosemanticity measures} \\
        \cmidrule(lr){3-3}\cmidrule(lr){4-5}
        \shortstack[l]{Model /\\baseline}
        & \shortstack{Active neurons\\(out of 4096)}
        & \shortstack{Score \(S_j\geq0.9\)\\(\% active) \(\uparrow\)}
        & \shortstack{Mean monosemanticity\\score (\(M_j\)) \(\uparrow\)\\(over specialized neurons)}
        & \shortstack{Neurons with \(M_j>1/2\)\\(\% specialized) \(\uparrow\)} \\
        \midrule
        ReLU
        & \(3332\pm69\)
        & \(22.19\pm0.62\) & \(0.969\pm0.001\) & \(99.50\pm0.08\) \\
        GELU
        & \(502\pm5\)
        & \(74.59\pm1.38\) & \(0.660\pm0.009\) & \(65.40\pm1.64\) \\
        ReGLU
        & \(399\pm8\)
        & \(75.25\pm1.77\) & \(0.985\pm0.001\) & \(99.22\pm0.18\) \\
        SwiGLU
        & \(692\pm15\)
        & \(39.89\pm1.37\) & \(0.655\pm0.024\) & \(68.86\pm3.24\) \\
        \midrule
        Baseline & --- & \(\approx 0\) & \(0.02\) & \(0\) \\
        \bottomrule
    \end{tabular}
    \caption{\textbf{Specialization and monosemanticity with non-orthogonal means and predictive directions.}
    Measures of neuron specialization and monosemanticity across activations.
    Values are mean \(\pm\) standard deviation across three runs varying only
    random initialization and optimization minibatch sampling, with the data fixed.
    Baselines use random weight orientations for
    specialization (maximum mean across activations) and uniform firing for
    monosemanticity. Setup and baseline details: Appendices~\ref{app:experiments-3-2-3-4-details}
    and~\ref{app:random-orientation-specialization}.}
    \label{tab:single-index-specialization}
\end{table}

\begin{figure}[!htb]
    \centering
    \includegraphics[width=0.74\linewidth]
    {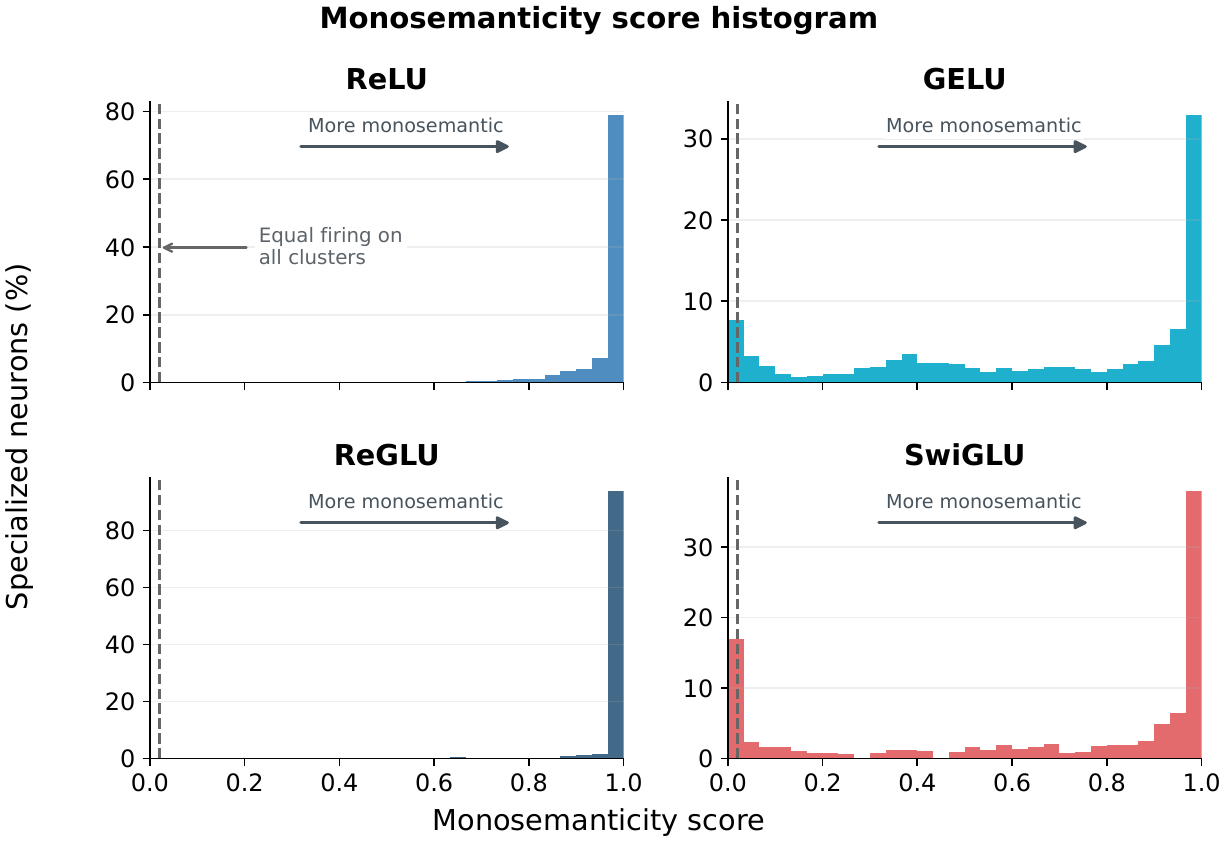}
    \caption{\textbf{Monosemanticity of specialized neurons with non-orthogonal means and predictive directions.}
    Distribution of monosemanticity scores \(M_j\) among specialized neurons.
    A score of \(0\) means no activation energy comes from the cluster to which
    the neuron is specialized; a score of \(1\) means all activation energy
    comes from that cluster. The dashed line at \(M_j=1/50=0.02\) marks
    equal activation energy across all 50 clusters.
    Experimental details:
    Appendix~\ref{app:experiments-3-2-3-4-details}.}
    \label{fig:q2-histograms-single-index}
\end{figure}

\begin{figure}[!htb]
    \centering
    \includegraphics[width=\linewidth]
    {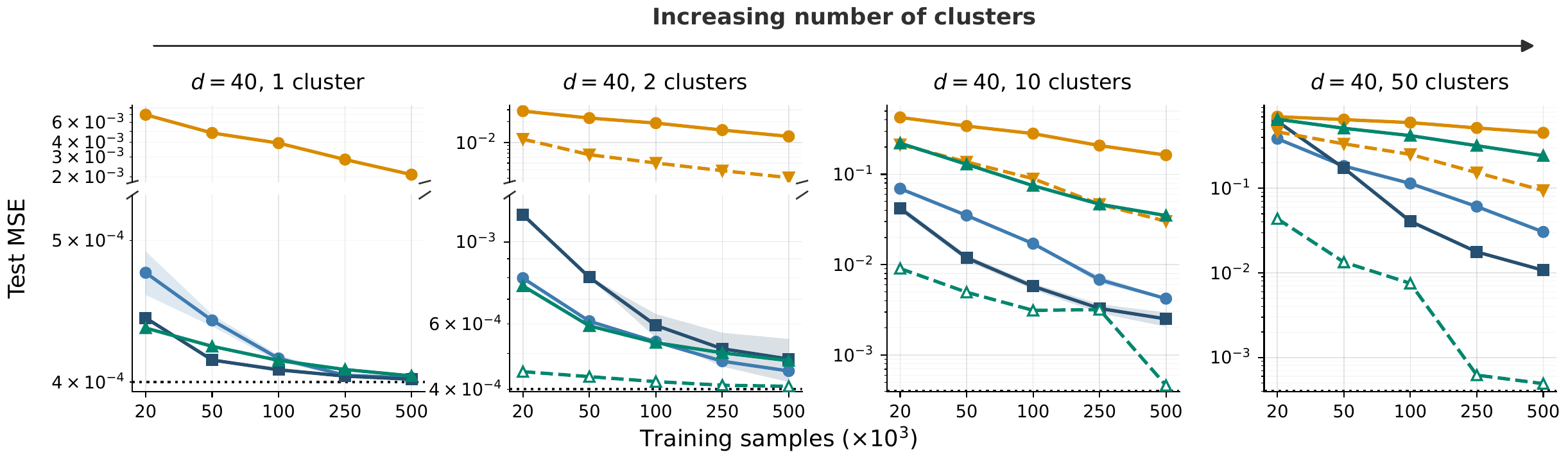}
    \par\smallskip
    \includegraphics[width=\linewidth]
    {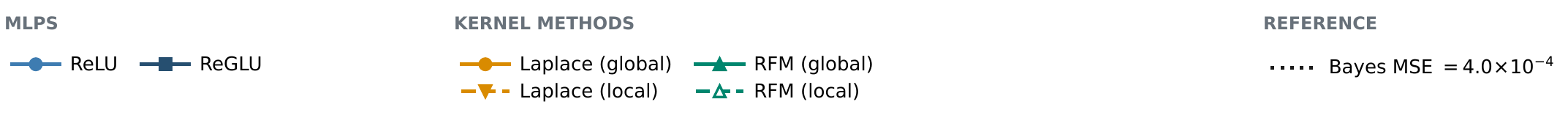}
    \caption{\textbf{MLPs learn clustered single-index models efficiently.}
    Test MSE as training size and number of clusters vary. Shading shows one
    standard deviation across three MLP runs varying only random initialization
    and optimization minibatch sampling; data, geometry, and link assignments
    are fixed across runs. Local methods use true cluster identities to fit
    a separate predictor within each cluster; the dotted line marks the noise floor. The first two panels
    use broken vertical axes. Setup: Appendix~\ref{app:experiments-3-2-3-4-details};
    all activations: Appendix Fig.~\ref{fig:single-index-all-activations}.}
    \label{fig:local-sample-complexity}
\end{figure}

\paragraph{Data setting.}
We use \(K\in\{1,2,10,50\}\) balanced Gaussian clusters in \(d=40\),
with twenty routing and twenty predictive coordinates.
Cluster means \(\mu_c\) lie in the routing subspace. For \(K>1\), the
means and routing noise scale \(\sigma_K\) are chosen so that the clusters
are well separated; Appendix~\ref{app:experiments-3-2-3-4-details} gives
the full construction.
Each predictive direction \(v_c\) is sampled independently by normalizing
a standard Gaussian vector in the twenty-dimensional predictive subspace.
These directions have unit norm but need not be orthogonal to one another.
For a sample from cluster \(c\),
\[
    X\mid C=c\sim\mathcal N\!\left(\mu_c,
        \operatorname{diag}(\sigma_K^2I_{20},I_{20})\right),
    \qquad Y=g_c(v_c^\top X)+\varepsilon.
\]
Links are assigned from normalized Hermite-2, normalized Hermite-3, sine, and hyperbolic
tangent, with independent noise \(\varepsilon\sim\mathcal N(0,0.02^2)\).
We vary training size from \(20{,}000\) to \(500{,}000\), using width-\(4096\)
MLPs. Geometry, links, and all data splits are fixed across the three MLP
runs; initialization and minibatch sampling vary. Hyperparameters and
checkpoints are selected using validation MSE.

\paragraph{Specialization and monosemanticity.}
We use the measures introduced in Section~\ref{sec:coordinate-specialization}:
the specialization score \(S_j\), based on alignment with
\(\operatorname{span}\{\mu_c,v_c\}\) in Eq.~\ref{eq:matched-subspace-specialization},
and the monosemanticity score \(M_j\) in Eq.~\ref{eq:activation-energy-concentration}.
Table~\ref{tab:single-index-specialization} reports the specialization and
monosemanticity measures.
Figure~\ref{fig:q2-histograms-single-index} shows monosemanticity score histograms
among specialized neurons.

\paragraph{Details and additional results.}
Appendix~\ref{app:experiments-3-2-3-4-details} gives the data-generation,
training, and evaluation protocols, including the active-neuron criterion
and validation-based model selection.
Random weight orientations with unchanged block norms yield effectively zero
specialization at the same threshold
(Appendix~\ref{app:random-orientation-specialization}).
\FloatBarrier

\subsection{MLPs Retain Their Advantage for Clustered Multi-Index Models}
\label{sec:clustered-multi-index}

\paragraph{Main findings.}
The findings of Section~\ref{sec:local-routing} also hold for multi-index
models introduced in Section~\ref{sec:prelim}, where prediction within each
cluster depends on several directions.
Table~\ref{tab:multi-index-specialization} and
Fig.~\ref{fig:q2-histograms-multi-index} show that a substantial fraction of
active neurons specialize and that many of these specialized neurons are
monosemantic. MLPs continue to have a sample-efficiency advantage over global
kernel methods and perform close to local RFM
(Fig.~\ref{fig:clustered-multi-index-results}).

% Requires: booktabs, graphicx
\begin{table}[!htb]
    \centering
    \small
    \setlength{\tabcolsep}{3.5pt}
    \renewcommand{\arraystretch}{1.10}
    \begin{tabular}{lcccc}
        \toprule
        & & \multicolumn{1}{c}{Specialization measure}
        & \multicolumn{2}{c}{Monosemanticity measures} \\
        \cmidrule(lr){3-3}\cmidrule(lr){4-5}
        \shortstack[l]{Model /\\baseline}
        & \shortstack{Active neurons\\(out of 4096)}
        & \shortstack{Score \(S_j\geq0.9\)\\(\% active) \(\uparrow\)}
        & \shortstack{Mean monosemanticity\\score (\(M_j\)) \(\uparrow\)\\(over specialized neurons)}
        & \shortstack{Neurons with \(M_j>1/2\)\\(\% specialized) \(\uparrow\)} \\
        \midrule
        ReLU
        & \(3378\pm20\)
        & \(24.28\pm0.37\) & \(0.981\pm0.001\) & \(99.80\pm0.07\) \\
        GELU
        & \(1068\pm16\)
        & \(70.64\pm2.09\) & \(0.718\pm0.005\) & \(79.84\pm0.62\) \\
        ReGLU
        & \(387\pm5\)
        & \(95.18\pm0.24\) & \(0.996\pm0.000\) & \(100.00\pm0.00\) \\
        SwiGLU
        & \(483\pm6\)
        & \(63.08\pm0.71\) & \(0.588\pm0.002\) & \(56.67\pm2.07\) \\
        \midrule
        Baseline & --- & \(\approx 0\) & \(0.02\) & \(0\) \\
        \bottomrule
    \end{tabular}
    \caption{\textbf{Specialization and monosemanticity in multi-index models.}
    Measures of neuron specialization and monosemanticity across activations.
    Values are mean \(\pm\) standard deviation across three runs varying only
    random initialization and optimization minibatch sampling, with the data fixed.
    Baselines use random weight orientations for
    specialization (maximum mean across activations) and uniform firing for
    monosemanticity. Setup and baseline details: Appendices~\ref{app:experiments-3-2-3-4-details}
    and~\ref{app:random-orientation-specialization}.}
    \label{tab:multi-index-specialization}
\end{table}

\begin{figure}[!htb]
    \centering
    \includegraphics[width=0.74\linewidth]
    {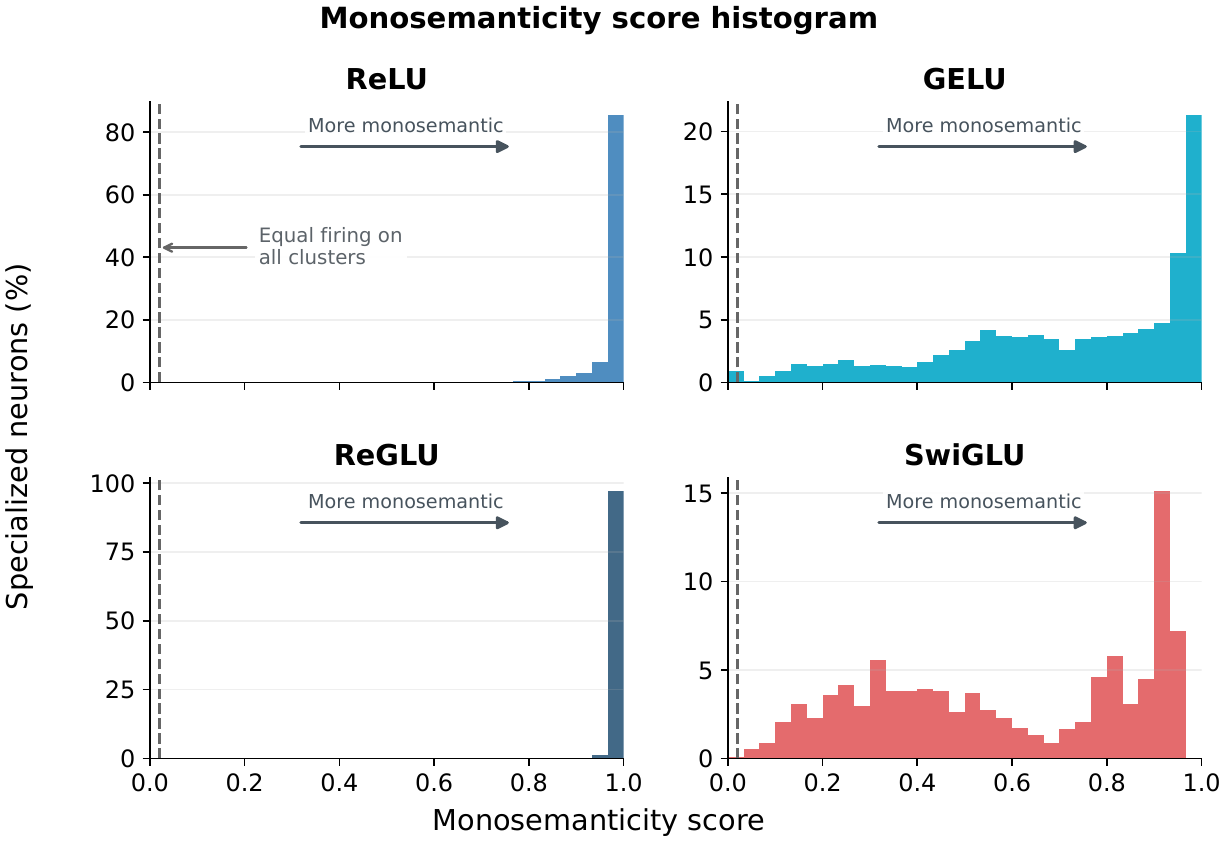}
    \caption{\textbf{Monosemanticity of specialized neurons in MLPs for multi-index models.}
    Distribution of monosemanticity scores \(M_j\) among specialized neurons.
    A score of \(0\) means no activation energy comes from the cluster to which
    the neuron is specialized; a score of \(1\) means all activation energy
    comes from that cluster. The dashed line at \(M_j=1/50=0.02\) marks
    equal activation energy across all 50 clusters.
    Histograms pool specialized neurons across three runs varying only random
    initialization and optimization minibatch sampling; data, geometry, and
    link assignments are fixed across runs.
    Experimental details:
    Appendix~\ref{app:experiments-3-2-3-4-details}.}
    \label{fig:q2-histograms-multi-index}
\end{figure}

\begin{figure}[!htb]
    \centering
    \includegraphics[width=0.9\linewidth]
    {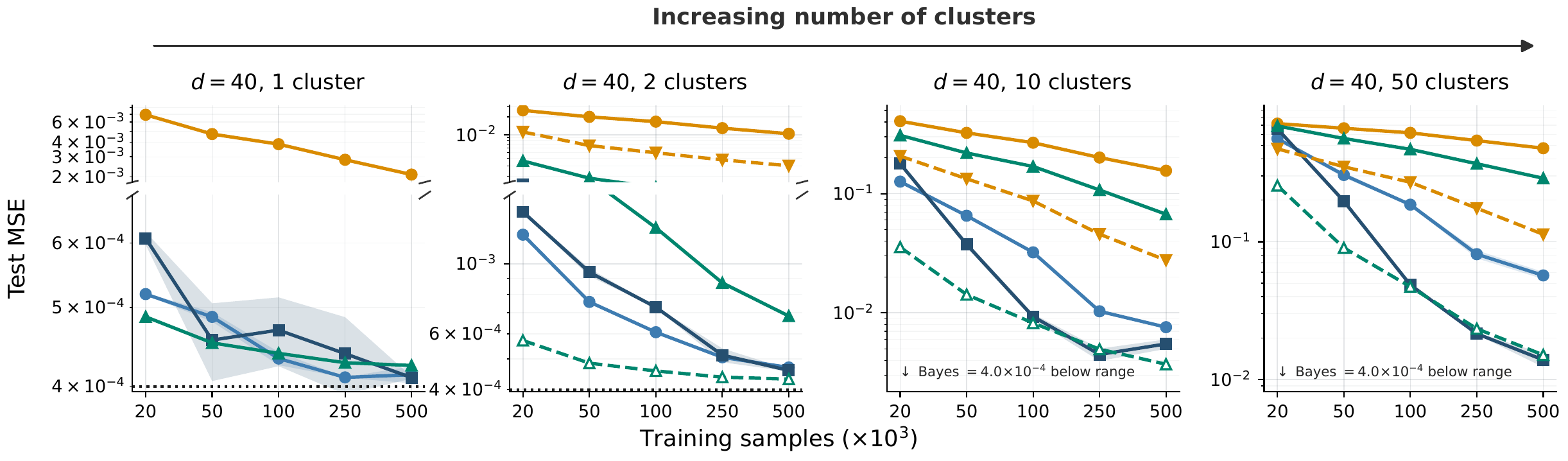}
    \par\smallskip
    \includegraphics[width=0.9\linewidth]
    {figs/experiments_revision/polished/relu_reglu_legend_only.pdf}
    \caption{\textbf{MLPs retain their advantage for clustered multi-index models.}
    Test MSE as training size and number of clusters vary. Shading shows one
    standard deviation across three MLP runs. Local methods use true cluster
    identities to fit a separate predictor within each cluster. The first two panels use broken vertical axes; the noise floor
    is below the plotted range with 10 and 50 clusters.
    Setup: Appendix~\ref{app:experiments-3-2-3-4-details}; all activations:
    Appendix Fig.~\ref{fig:multi-index-all-activations}.}
    \label{fig:clustered-multi-index-results}
\end{figure}

\paragraph{Data setting.}
We use the same \(d=40\) split into twenty routing and twenty predictive
coordinates and the same \(K\in\{1,2,10,50\}\) as in
Section~\ref{sec:local-routing}. For each \(K\), we reuse the same cluster
means \(\mu_c\) and routing noise scale \(\sigma_K\).
For each cluster \(c\), we independently sample three standard Gaussian vectors in the predictive
subspace and orthonormalize them using a reduced QR decomposition to obtain
\(v_{c,1},v_{c,2},v_{c,3}\). Directions belonging to different clusters
need not be orthogonal. The response is
\[
    Y=\frac{1}{\sqrt3}\sum_{\ell=1}^{3}
        g_c(v_{c,\ell}^\top X)+\varepsilon,
    \qquad \varepsilon\sim\mathcal N(0,0.02^2).
\]
Each cluster's link function \(g_c\) is randomly selected once from
normalized Hermite-2, normalized Hermite-3, sine, and hyperbolic tangent,
then held fixed. Within each cluster, \(g_c\) is applied
to all three predictive projections \(v_{c,\ell}^\top X\).

\paragraph{Specialization and monosemanticity.}
We use the same specialization and monosemanticity measures introduced in
Section~\ref{sec:coordinate-specialization}. For specialization, the relevant
span now includes all three predictive directions:
\[
\operatorname{span}\{\mu_c,v_{c,1},v_{c,2},v_{c,3}\}.
\]
In Eq.~\ref{eq:matched-subspace-specialization}, we replace
\(\langle w_j,v_c\rangle^2\) by
\(\sum_{\ell=1}^{3}\langle w_j,v_{c,\ell}\rangle^2\).
Table~\ref{tab:multi-index-specialization} reports the specialization and
monosemanticity measures.
Figure~\ref{fig:q2-histograms-multi-index} shows monosemanticity score histograms
among specialized neurons.

\paragraph{Details.}
Appendix~\ref{app:experiments-3-2-3-4-details} specifies the data, training,
model-selection, and specialization protocols for both index-model experiments.

\subsection{Trained MLPs Encode Cluster Structure}
\label{sec:learned-gating}

We next ask whether the cluster structure learned by an MLP can be used
by a separate predictor. To test this, we extract clusters and local features
from a trained MLP, then fit a Laplace kernel or RFM predictor within each
learned cluster, without supplying the true cluster assignments.

\paragraph{Procedure.}
We first train a ReLU MLP without cluster labels and freeze its weights.
For each input, we form a representation describing the relative contributions
of its active first-layer neurons. We apply K-means to these representations,
using the known number of clusters \(K\). Within each learned cluster, we
construct local features from the first-layer weights, weighted by the neurons'
contributions to that cluster, and fit a separate Laplace kernel or RFM
predictor. We call these \emph{MLP-gated} predictors. At test time, we compute
the input's MLP representation, assign it to the nearest learned centroid,
and apply the corresponding local predictor.
Figure~\ref{fig:mlp-gating-procedure} summarizes this procedure.

\begin{figure}[!htb]
  \centering
  \experimentfigure[width=0.85\linewidth]{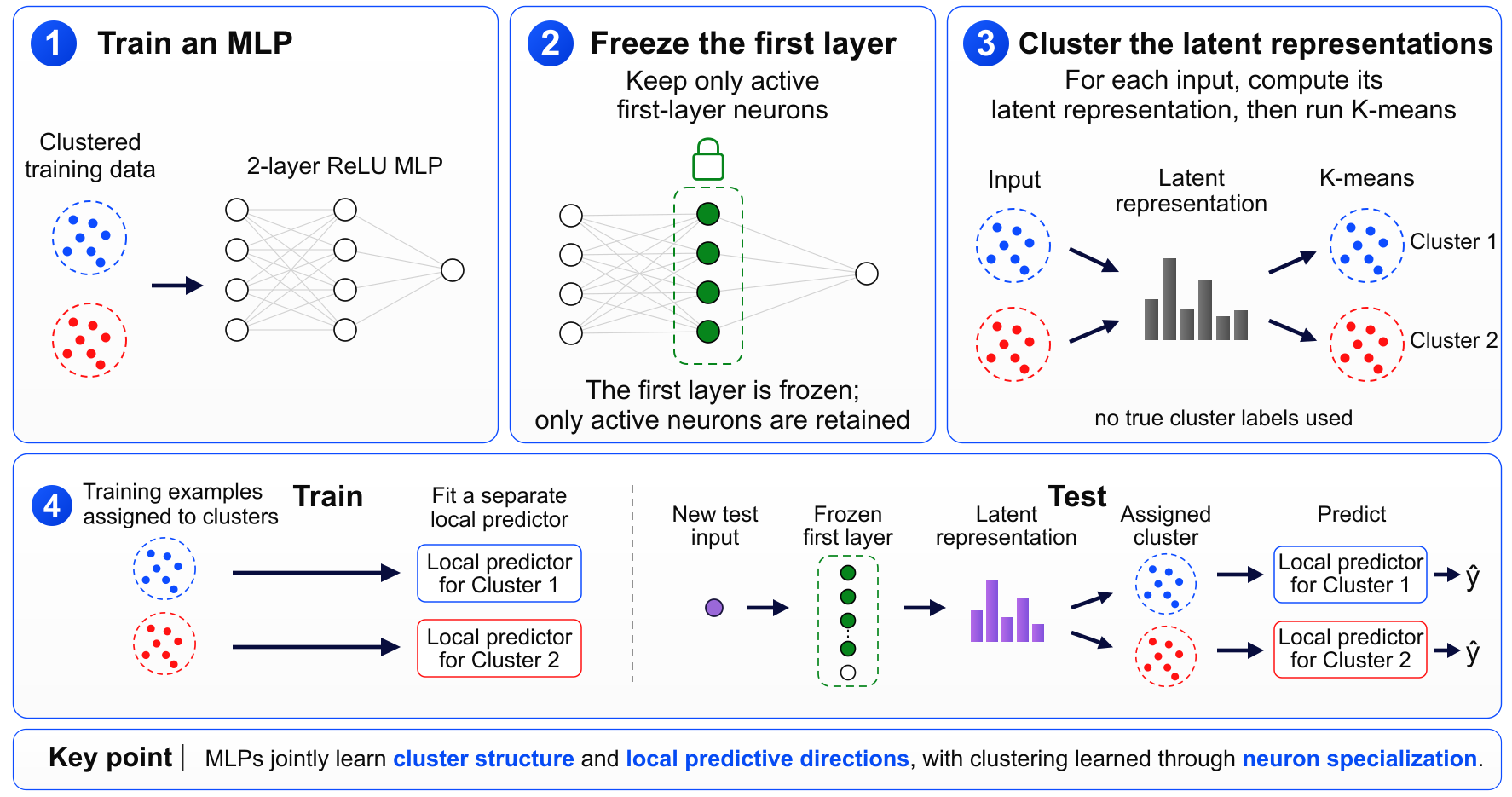}
  \caption{\textbf{Overview of the MLP-gated local-prediction procedure.} A ReLU
  MLP is trained without cluster labels, K-means is applied to its first-layer
  representations, and a local predictor is fit within each learned cluster.
  Test inputs are routed using the nearest learned centroid.}
  \label{fig:mlp-gating-procedure}
\end{figure}

\paragraph{Main findings.}
Figure~\ref{fig:local-sample-complexity_kmeans} shows that using clusters
extracted from the MLP representation substantially improves both Laplace and
RFM over their global counterparts. In particular, MLP-gated RFM consistently
outperforms global RFM and approaches local RFM, which is fitted using the
true clusters. Thus, the learned representation contains both the information
needed to route inputs to the appropriate cluster and the predictive features
relevant to that cluster.

\begin{figure}[!htb]
    \centering
    \experimentfigure[width=0.87\linewidth]
    {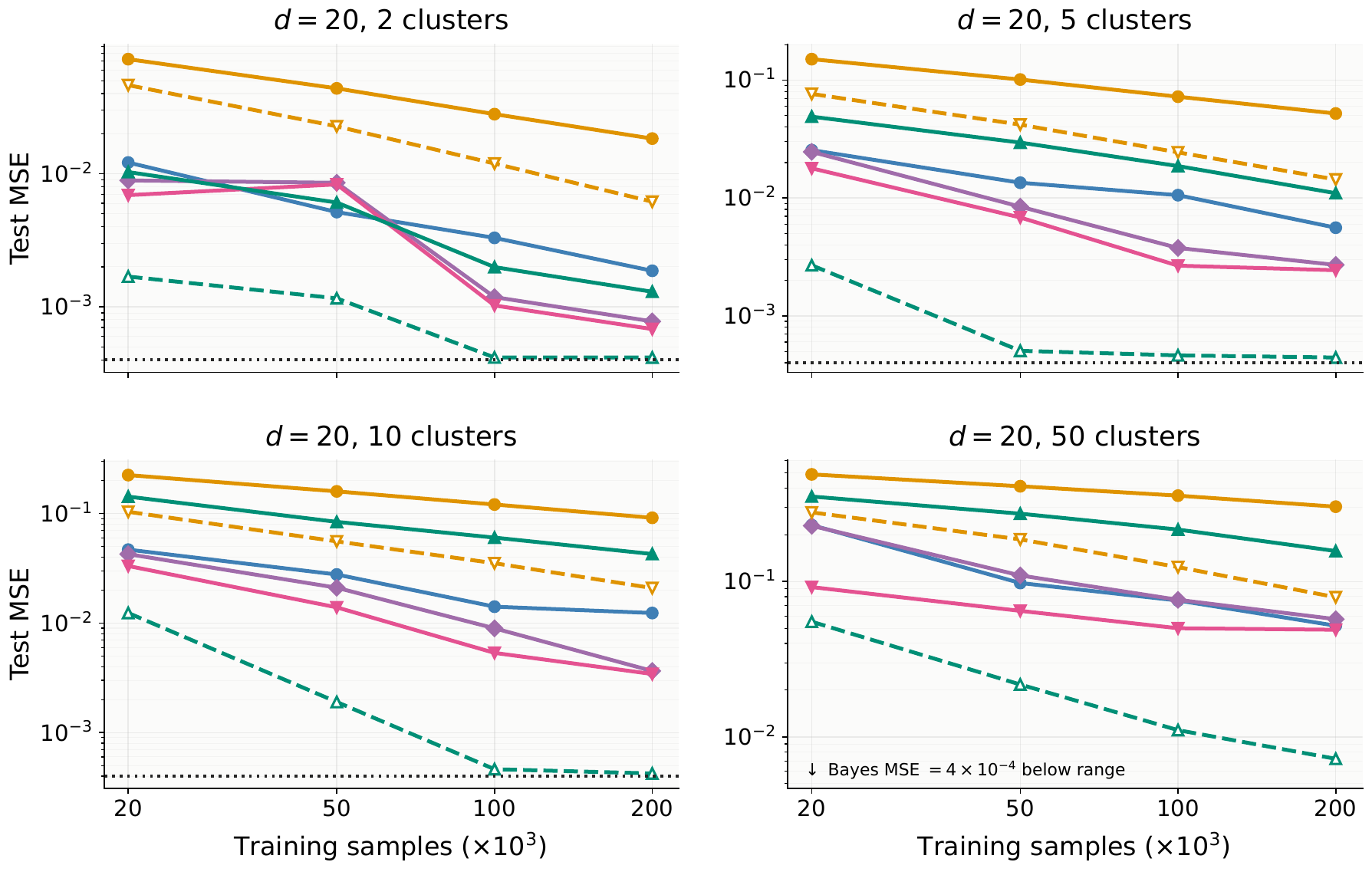}

    \vspace{0.3em}

    \experimentfigure[width=0.82\linewidth]
    {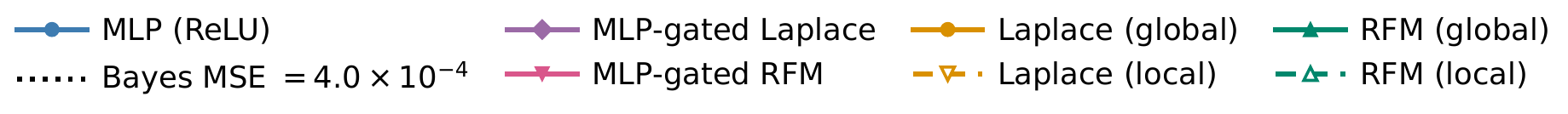}

    \caption{\textbf{Trained MLP representations encode cluster structure.}
    RFM fitted within MLP-derived clusters improves on global RFM and
    approaches local RFM fitted with true cluster identities. Curves average
    three runs. Experimental details: Appendix~\ref{app:learned-gating-details}.}
    \label{fig:local-sample-complexity_kmeans}
\end{figure}

\paragraph{Data setting.}
This experiment uses a separate dataset in \(d=20\), with ten routing
and ten predictive coordinates. We consider \(K\in\{2,5,10,50\}\). Each cluster's link function is sampled uniformly from
the normalized third-Hermite polynomial, sine, and hyperbolic tangent. We
average results over three independent seeds and evaluate each run on
\(10{,}000\) test examples.

\paragraph{Details.}
Appendix~\ref{app:learned-gating-details} specifies active-neuron selection,
training, representation clustering, and local-predictor evaluation.
\FloatBarrier

\Needspace{12\baselineskip}
\section{Theory}
\label{sec:specialization-theory}

In this section, we prove two main results describing how
MLPs learn data with cluster structure. In
Theorem~\ref{thm:neurons-specialize}, we prove that neurons in the MLP
specialize during training. In Theorem~\ref{thm:rfm-mlp-gap}, we prove that,
as the number of clusters tends to infinity, MLPs outperform standard kernel
methods and Recursive Feature Machines (RFM) in terms of sample complexity. This comes as a consequence of the fact that the latter two methods cannot compute specialized features for each
cluster.

\subsection{MLP Neurons Specialize When Learning on Gaussian Mixture Model Data}
\label{sec:early-specialization}

In this subsection, we show that, under small initialization,
neurons trained on Gaussian mixture data align with the span of a single
cluster's mean and predictive direction. We use unit-length means and do
not require the Gaussian components to be well separated.

We consider $K$-cluster Gaussian mixture data in $d=2K$ dimensions. Let
$e_1,\ldots,e_K$ denote the standard basis of $\mathbb R^K$. The routing subspace consists of the first
$K$ coordinates and contains the cluster means; the predictive subspace
consists of the last $K$ coordinates and contains the cluster-specific
predictive directions.\footnote{Note that independent orthogonal changes of basis in the
two blocks preserve the isotropic Gaussian noise and all conclusions below.}

\begin{datasetting}[Symmetric Gaussian mixture]
\label{data:neuron-specialization}

For $c\in[K]$, let
\[
    \mu_c=
    \begin{bmatrix}
        e_c\\
        0
    \end{bmatrix},
    \qquad
    v_c=
    \begin{bmatrix}
        0\\
        e_c
    \end{bmatrix}.
\]
Set each cluster to be isotropic with $\Sigma_c=I_d$ and use the common link function $g_c=h_3$, where $h_3(t)=(t^3-3t)/\sqrt6$ is the third Hermite polynomial.

In the notation of Section~\ref{sec:prelim}, the data distribution is
\[
    c\sim\operatorname{Unif}([K]),
    \qquad
    x\mid c\sim\mathcal N(\mu_c,\Sigma_c),
    \qquad
    y=g_c\!\left(\langle x,v_c\rangle\right)
      =h_3\!\left(\langle x,v_c\rangle\right).
\]
\end{datasetting}

The routing subspace carries information about the latent cluster through its
mean, while within cluster $c$ the response depends only on $v_c$.
The Gaussian components overlap, so the cluster index need not be
recoverable from an individual input. All population expectations below
are taken under the joint distribution of $(c,x,y)$, and training uses
only $(x,y)$. The cubic Hermite link gives an explicit expression for
$\mathbb{E}[y\,\phi(\omega^\top x)]$, allowing us to characterize the
directions selected by the early training dynamics.
Next, we consider training a neural network to learn this data distribution.

\begin{trainingsetup}[Two-layer ReLU population gradient flow]
\label{train:neuron-specialization}
We train the two-layer ReLU network
\[
    f_\theta(x)=\sum_{j=1}^m a_j\phi(w_j^\top x),
    \qquad
    \phi(t)=t_+,
\]
by population gradient flow on the squared loss
\[
    \mathcal L(\theta)
    =
    \frac12\mathbb E\left[
        \bigl(f_\theta(x)-y\bigr)^2
    \right].
\]
\end{trainingsetup}

\begin{initializationsetting}[Small random hidden weights and zero output layer]
\label{init:neuron-specialization}

The hidden directions are initialized independently and uniformly at random, scaled by a parameter $\varepsilon > 0$,
while the output layer is initialized at zero:
\[
    \omega_j^0
    \stackrel{\mathrm{iid}}{\sim}
    \operatorname{Unif}(\mathbb S^{d-1}),
    \qquad
    w_j(0)=\varepsilon\omega_j^0,
    \qquad
    a_j(0)=0.
\]

\end{initializationsetting}

Define the cluster-specific span and its orthogonal projector by
\[
    \mathcal S_c:=\operatorname{span}\{\mu_c,v_c\},
    \qquad
    P_c:=\mu_c\mu_c^\top+v_cv_c^\top.
\]
Because $\mu_c$ and $v_c$ are orthonormal, the specialization score for cluster $c$ in
Eq.~\ref{eq:matched-subspace-specialization} is
$\|P_c w\|_2/\|w\|_2$. The following theorem establishes this same
geometric notion of specialization in the isotropic setting above.

\begin{theorem}[Randomly initialized neurons specialize]
\label{thm:neurons-specialize}
Under \cref{data:neuron-specialization}, \cref{train:neuron-specialization},
and \cref{init:neuron-specialization}, fix positive integers $K,m$.
For almost every draw of $\omega_1^0,\ldots,\omega_m^0$, there are
cluster labels $J_1,\ldots,J_m\in[K]$ such that, for every
$\delta\in(0,1)$, there exist a finite time $T_\delta$ and
$\varepsilon_0>0$ for which
\[
    0<\varepsilon\leq\varepsilon_0
    \quad\Longrightarrow\quad
    \left\|
        (I_d-P_{J_j})
        \frac{w_j(T_\delta)}{\|w_j(T_\delta)\|_2}
    \right\|_2
    \leq
    \delta
    \qquad\text{for every }j\in[m].
\]
Equivalently, every neuron's alignment with its selected cluster's span
satisfies
\[
    \frac{\|P_{J_j}w_j(T_\delta)\|_2}{\|w_j(T_\delta)\|_2}
    \geq\sqrt{1-\delta^2}.
\]
The labels depend only on the initial directions. The time $T_\delta$
and the bound $\varepsilon_0$ may depend on these directions, $K$, $m$,
and $\delta$; $T_\delta$ is independent of $\varepsilon$.

The selected labels $J_1,\ldots,J_m$ are independent and uniformly
distributed on $[K]$. Consequently,
\[
    \mathbb P\left(
        \{J_1,\ldots,J_m\}=[K]
    \right)
    \geq
    1-Ke^{-m/K}.
\]
In particular, for $\eta\in(0,1)$, if
\[
    m\geq K\log\left(\frac{K}{\eta}\right),
\]
then this label-coverage event has probability at least $1-\eta$.
On this event, outside the null set above, the specialization bounds cover
every cluster for each $\delta\in(0,1)$, at the time $T_\delta$ and for
$0<\varepsilon\leq\varepsilon_0$ specified above. These quantities may
depend on the realized initial directions and the tolerance.
\end{theorem}

The conclusion concerns geometric specialization during early training
as the initialization scale decreases. The limiting directions of the
associated independent-neuron dynamics have nonzero components along
both $\mu_c$ and $v_c$ (Lemma~\ref{lem:specialized-maxima}).
Monosemanticity is a separate property, measured empirically
by $M_j$ in Eq.~\ref{eq:activation-energy-concentration}.

\paragraph{Proof sketch.}
We employ a proof strategy of studying feature learning
through effectively independent neuron dynamics, which has been used in prior
work on early-time and small-initialization regimes
\citep{abbe2022merged,min2024early,glasgow2024sgd}.
At initialization, the network output is zero and the hidden weights are
stationary. The initial derivative of each output weight is proportional to
the population correlation
\[
    \Phi(\omega)
    =\mathbb E\left[y\,
      \phi(\omega^\top x)\right]
\]
of its own randomly initialized hidden direction. Except on a measure-zero
set, this correlation is nonzero, so the output weight immediately acquires
the appropriate sign. In the teacher-only dynamics, ReLU homogeneity then
turns the neuron's directional dynamics into a positive time
reparameterization of gradient ascent on the corresponding signed
correlation objective. Lemma~\ref{lem:early-time-tracking} subsequently
transfers this behavior to the full network by controlling the interaction
terms on fixed time intervals.

For the cubic Hermite link function, the Hermite calculation makes this objective a sum of
clusterwise cubic terms. Its positive local maxima are one-cluster solutions:
each lies exactly in $\mathcal S_c$ for one cluster $c$, with nonzero
mean and predictive components. Every positive mixed-cluster critical
point has an unstable direction. There are only finitely many positive
critical points, so gradient-flow convergence and strict-saddle avoidance
imply that a random neuron specializes almost surely. These arguments
use the orthogonal mean--predictive pairs and the correlation identity;
they do not require the Gaussian components to be well separated.

Permutation symmetry makes the selected cluster uniform on $[K]$, and
independence of the initial hidden directions makes the selected labels
independent. The coverage estimate is then the usual coupon-collector union
bound. Finally, on every fixed early-time interval the full network output is
$O(m\varepsilon^2)$, so the coupled network dynamics are a vanishing
perturbation of the isolated-neuron dynamics. A Gr\"onwall argument transfers
the specialization result to population gradient flow. The complete proof is
given in Appendix~\ref{app:proof-neurons-specialize}.

\subsection{Specialization Yields a Sample-Complexity Gap Between MLPs and \mbox{Adaptable} Kernel Methods}
\label{sec:mlp-rfm-separation}

In this subsection, we establish a sample-complexity
separation between two-layer MLPs and the Recursive Feature Machines (RFM) adaptive kernel method. Our clustered data distribution is similar to that of the previous subsection, but we let the number of clusters $K$ tend to infinity.

We show that fitting a norm-constrained MLP with a
polynomial number of training samples yields a consistent estimator. Our proof
of this fact relies on showing that the MLP can have specialized neurons that
approximate the predictive function separately within each cluster. In
contrast, we show that RFM is not consistent under any fixed polynomial scaling of
the sample size in $K$.

Concretely, the data distribution we use for this result is similar to that in the previous subsection, but we suppose full separation between clusters for
ease of analysis.

\begingroup
\tcbset{theorysetting/.append style={breakable=false}}
\begin{datasetting}[Gaussian mixture with noiseless routing]
\label{data:rfm-mlp-gap}
We retain the means $\mu_c=[e_c;0]$ and predictive directions
$v_c=[0;e_c]$ from \cref{data:neuron-specialization} in $d=2K$
dimensions, and make the first $K$ coordinates
noiseless by setting
\[
    \Sigma_c=
    \begin{bmatrix}
        0_{K\times K} & 0\\
        0 & I_K
    \end{bmatrix}.
\]

Additionally, we consider a non-constant link function $g:\mathbb R\to\mathbb R$ which is independent of $K$, and satisfies boundedness and Lipschitzness
\[
    \|g\|_\infty\leq G_0,
    \qquad
    |g(s)-g(t)|\leq L|s-t|
    \quad\text{for all }s,t\in\mathbb R,
\]
for constants $G_0<\infty$ and $L>0$. We use the same link in every
cluster, $g_c=g$. Thus, in the notation of Section~\ref{sec:prelim},
\[
    c\sim\operatorname{Unif}([K]),
    \qquad
    x\mid c\sim\mathcal N(\mu_c,\Sigma_c),
    \qquad
    y=g\!\left(\langle x,v_c\rangle\right).
\]

\end{datasetting}
\endgroup

The routing subspace now reveals the cluster exactly, while the response within
cluster $c$ depends only on its cluster-specific predictive direction.

Instead of studying a network trained through gradient-based
updates, we study a two-layer MLP fit by empirical risk minimization subject
to a Frobenius-norm constraint on its weights.

\begingroup
\tcbset{theorysetting/.append style={breakable=false}}

\begin{trainingsetup}[Frobenius-constrained MLP ERM]
\label{train:mlp-gap}
We fit the two-layer ReLU network
\[
    f_\theta(x)=\sum_{j=1}^m a_j\phi(w_j^\top x),
    \qquad
    \phi(t)=t_+,
\]
using empirical risk minimization. For a width $m$ and budget $B>0$, define
\[
    \mathcal F_{B,m}
    :=
    \left\{
        f_\theta:
        \|\theta\|_{\mathrm F}^2
        =\frac12\sum_{j=1}^m
        \left(a_j^2+\|w_j\|_2^2\right)
        \leq B
    \right\}.
\]
Given $n$ samples, squared-loss empirical risk minimization (ERM) gives
\[
    \widehat f_{\mathrm{MLP}}
    \in
    \arg\min_{f\in\mathcal F_{B,m}}
    \frac1n\sum_{i=1}^n\bigl(f(x_i)-y_i\bigr)^2.
\]
\end{trainingsetup}

\endgroup

We compare this MLP ERM with standard rotationally invariant kernel ridge
regression and with RFM, a supervised kernel method
\citep{radhakrishnan2024mechanism,radhakrishnan2025linear}.\footnote{
We give the RFM comparator the ground-truth population average gradient outer
product (AGOP), thereby removing metric-estimation error and isolating the
limitation of using one global feature metric. An analysis of RFM using an
empirically estimated AGOP would require controlling empirical fluctuations
and is not implied by our theorem.}

\begingroup
\tcbset{theorysetting/.append style={breakable=false}}

\begin{trainingsetup}[Kernel methods and RFM with ground-truth AGOP]
\label{train:rfm-gap}

Let $\mathcal K$ be a rotationally invariant kernel. Given $n$ samples,
$\widehat f_{\mathrm{Kernel}}$ is kernel ridge regression with kernel
$\mathcal K(x,x')$. We give the RFM comparator the ground-truth population
AGOP\footnotemark
\[
    M
    :=\mathbb E\left[
        \nabla_x f^\star(x)\nabla_x f^\star(x)^\top
      \right].
\]
For $\rho\geq0$, define the regularized AGOP and corresponding kernel
\[
    M_\rho:=M+\rho I_d,
    \qquad
    \mathcal K_{M_\rho}(x,x')
    :=\mathcal K\bigl(\sqrt{M_\rho}x,\sqrt{M_\rho}x'\bigr).
\]
$\widehat f_{\mathrm{RFM}}$ is kernel ridge regression with kernel
$\mathcal K_{M_\rho}$. Both estimators may use any kernel ridge parameter
$\lambda_n\geq0$.

\end{trainingsetup}

\footnotetext{Because $P_x$ is supported on a union of affine subspaces,
the regression function $f^\star$ is intrinsically defined only on that
support, and its derivatives normal to the support are not determined.
Throughout, $\nabla_x f^\star(x)$ denotes the gradient along the affine
support component containing $x$, embedded in $\mathbb R^d$ with zero normal
component. Since $g$ is Lipschitz, it is differentiable almost everywhere,
and for an input $x$ from cluster $c$,
\[
    \nabla_x f^\star(x)
    =g'\!\left(\langle x,v_c\rangle\right)v_c
\]
almost everywhere.}

\endgroup

In the theorem below, we show that a two-layer MLP learns the
cluster-structured target with polynomial sample complexity for every fixed,
bounded, nonconstant Lipschitz link function $g$. In contrast, the standard rotationally
invariant kernel estimator and RFM with the ground-truth AGOP remain
inconsistent under every fixed polynomial sample-size scaling. The appendix
also proves the MLP upper bound in a more general fixed-dimensional
multi-index setting; the single-index result stated here is a
corollary for $r=1$, which is simpler to state.

\Needspace{18\baselineskip}
\begin{theorem}[Specialization enables data-efficient learning]
\label{thm:rfm-mlp-gap}
Suppose \cref{data:rfm-mlp-gap} holds. Then the target function $f^\star$ admits approximation by cluster-specialized two-layer
ReLU networks with polynomially growing width and Frobenius norm,
and with approximation error tending to zero as $K\to\infty$.

Consequently, there are polynomially growing width
and norm-budget sequences $(m_K)$ and $(B_K)$ such that, with polynomially many samples $n_K=\Omega(K^{7})$,
the MLP estimator in \cref{train:mlp-gap} is consistent:
\[
    \lim_{K\to\infty}
    \mathbb E\left[
        \|\widehat f_{\mathrm{MLP}}-f^\star\|_{L^2(P_x)}^2
    \right]
    =0.
\]

On the other hand, for every fixed $A<\infty$, every sequence
$n_K=O(K^A)$, and every sequence $\rho_K\geq0$,
\[
    \liminf_{K\to\infty}
    \mathbb E\left[
        \|\widehat f-f^\star\|_{L^2(P_x)}^2
    \right]
    >0,
    \qquad
    \widehat f\in
    \left\{
        \widehat f_{\mathrm{Kernel}},
        \widehat f_{\mathrm{RFM}}
    \right\}.
\]
\end{theorem}

\paragraph{Proof sketch.}
The MLP claim follows from a more general multi-index result proved in
Appendix~\ref{app:proof-mlp-rfm-separation}. We first show that the link function can be approximated arbitrarily well
by a finite shallow network. By allowing the approximation accuracy to
improve sufficiently slowly with $K$, we obtain a sequence of
cluster-specialized MLPs whose approximation error tends to zero while
their width and Frobenius norm grow only polynomially in $K$. This construction follows by using neurons specialized to each cluster as follows. Since
$g$ is Lipschitz and bounded, it can be approximated in Gaussian $L^2$ by a
finite shallow network of ridge-ramp functions, each of which is a
difference of two ReLUs. We can thus approximate $g$ on each cluster by using routing
coordinates to convert every ridge-ramp unit into a pair of ReLU neurons whose
contributions cancel on all nonselected clusters.

The construction that we provide implies that there is a Frobenius-norm-bounded network with good approximation error. Thus, standard Barron- and path-norm Rademacher-complexity bounds control the
estimation error
\citep{e2022barron,bach2017breaking,golowich2018size}.
Because the Gaussian covariates are unbounded, we derive the required
sample-dependent version carefully, using Gaussian concentration to control
the MLP predictions on typical inputs.

For the kernel lower bound, expand the link function in terms of normalized Hermite
polynomials,
\[
    g(t)=\sum_{q=0}^\infty\widehat g_qh_q(t),
\]
where $h_q$ denotes the normalized probabilists' Hermite polynomial of degree
$q$, and $\widehat g_q$ is the corresponding scalar Hermite coefficient of $g$.
A bounded, continuous, nonconstant function cannot have only finitely many
nonzero Hermite coefficients: otherwise it would equal a polynomial, and a
bounded polynomial on $\mathbb R$ is constant. Thus every finite Hermite tail
of $g$ has positive energy.

The ground-truth population AGOP is a rescaling of
$P_V:=\sum_{c=1}^K v_cv_c^\top$, the orthogonal projector onto the span of
the cluster-specific predictive directions:
\[
    M=\frac{\alpha}{K}P_V,
    \qquad
    \alpha=\mathbb E[g'(G)^2]>0.
\]
Hence the standard kernel metric and every regularized AGOP metric remain
rotationally invariant within the predictive coordinates. At Hermite order
$q$, the clusterwise target contains a component in an irreducible harmonic
subspace of dimension $\Theta_q(K^q)$. The representer theorem restricts a
kernel predictor to the span of at most $n$ kernel sections, and the projected
span is an invariant random subspace. Therefore, when $n=O(K^A)$, every fixed
order $q>A$ leaves asymptotically all of its target energy unrecovered. Since
$g$ has nonzero Hermite energy beyond every finite order, both kernel
estimators retain nonzero error under every fixed polynomial sample scaling. This degree-by-degree harmonic obstruction is closely related to prior
high-dimensional analyses of rotationally invariant kernel methods and
invariant kernel models
\citep{ghorbani2021linearized,ghorbani2020outperform,mei2021invariances}.
The complete proof is given in
Appendix~\ref{app:proof-mlp-rfm-separation}.

\clearpage

\section*{Acknowledgements}

We gratefully acknowledge support from the National Science Foundation (NSF)
under grants CCF-2112665 and MFAI 2502258, the Office of Naval Research
(ONR N000142412631), and the Defense Advanced Research Projects Agency
(DARPA) under Contract No. HR001125CE020.
This work used the Delta system at the National Center for Supercomputing Applications through allocation TG-CIS220009 from the Advanced Cyberinfrastructure Coordination Ecosystem: Services \& Support (ACCESS) program, which is supported by National Science Foundation grants \#2138259, \#2138286, \#2138307, \#2137603, and \#2138296.

AI tools were used to assist with aspects of the experiments and theoretical analysis.

\bibliographystyle{plainnat}
\bibliography{references_feature_learning}

\clearpage
\appendix

\section{Additional Experimental Details}
\label{app:experimental-details}

\subsection{Active Neurons, Specialization, and Monosemanticity}
\label{app:shared-specialization-analysis}

For each trained model, we rank hidden neurons by importance and retain the
smallest prefix accounting for at least \(99.9\%\) of total importance.
These are the \emph{active neurons} used in the specialization analysis.
The importance of an ordinary ReLU or GELU neuron is
\[
    I_j=|a_j|\|w_j\|_2,
\]
where \(a_j\) is the final readout coefficient. For gated neurons in all three specialization experiments, it is
\[
    I_j=|a_j|\sqrt{\|w_{g,j}\|_2\|w_{v,j}\|_2}.
\]
The \(99.9\%\) rule is a descriptive selection criterion, not a significance
threshold or a guarantee about the fraction of prediction error retained.

We measure specialization using the score in
Eq.~\ref{eq:matched-subspace-specialization}, with a two-dimensional span
for single-index targets and a four-dimensional span for rank-three targets
when the cluster mean is nonzero. Biases are excluded from the
specialization score but included in the hidden activation used to compute
\(M_j\). The cluster means and predictive directions used in the specialization
analyses have unit norm and lie in orthogonal routing and predictive subspaces. For gated neurons the two branch scores are
combined by taking their minimum for each cluster before maximizing over
clusters. Individual-branch statistics maximize only that branch's score.

The specialization score measures weight alignment and allows different
amounts of weight on routing and prediction. Monosemanticity concerns
cluster-selective activation, which we separately quantify using
Eq.~\ref{eq:activation-energy-concentration}. Activation energy within a
cluster is expected squared activation conditional on cluster membership.
For balanced clusters, the monosemanticity score \(M_j\) is the fraction
of a neuron's total activation energy contributed by inputs from the cluster
to which it is specialized. The reported mean averages these
per-neuron fractions over specialized neurons.

\subsection{Coordinate-Based Subspace Specialization}
\label{app:experiment-3-1-details}
\label{app:coordinate-specialization-details}

The experiment in Section~\ref{sec:coordinate-specialization} uses \(K=25\)
balanced clusters in \(d=60\). For cluster \(c\), the mean \(\mu_c\) is
\(1\) in coordinate \(c\) and \(0\) elsewhere, and \(v_c\) is the unit
vector along predictive coordinate \(25+c\). The last ten coordinates are
nuisance coordinates. Routing noise has standard deviation \(0.125\); predictive
and nuisance coordinates are independent standard Gaussians. Label noise
has standard deviation \(0.02\). The fixed link assignment contains nine
normalized Hermite-2, seven normalized Hermite-3, three sine, and six
hyperbolic-tangent clusters.

One fixed balanced dataset contains \(500{,}000\) training, \(20{,}000\)
validation, and \(20{,}000\) test examples. All three training runs use the
same data, geometry, link assignments, and label noise. Only initialization
and optimization minibatches vary.

\paragraph{Training and selection.}
All MLPs have width \(4096\) and use the same small-initialization
parameterization as the index-model experiments in
Appendix~\ref{app:experiments-3-2-3-4-details}. Learning rate and weight decay
are selected using validation MSE only. All four architectures select
learning rate \(3\times10^{-3}\) and weight decay \(10^{-6}\). The checkpoint
with the lowest validation MSE is retained; the test set is evaluated only
after checkpoint and hyperparameter selection. Table~\ref{tab:k25-axis-performance}
reports predictive performance.

\paragraph{Specialization and monosemanticity.}
For the full weight vector, the specialization score for cluster \(c\) reduces to
\[
    S_{j,c}=\frac{\sqrt{[w_j]_c^2+[w_j]_{25+c}^2}}{\|w_j\|_2}.
\]
The denominator includes routing, predictive, and nuisance coordinates.
For gated units, we take the minimum of the gate and value scores for each
cluster before maximizing over clusters. Active neurons are the smallest
importance-ranked set accounting for \(99.9\%\) of total importance.

We evaluate the monosemanticity score \(M_j\) only among active neurons satisfying
\(S_j\geq0.9\), using the cluster to which each neuron is specialized.
Hidden activations include biases and, for gated units, the gate--value
product. The uniform baseline is \(1/25=0.04\). Specialization and
monosemanticity are evaluated separately, as shown in
Table~\ref{tab:coordinate-specialization} and Fig.~\ref{fig:k25-axis-q2}.
Table entries summarize each run and then report the mean and sample
standard deviation across runs; histograms pool selected neurons across
runs separately for each activation. Bar heights give the percentage of
these neurons in each score interval, summing to \(100\%\) within each panel.
The ReLU examples in
Fig.~\ref{fig:coordinate-specialization-hermite3-neurons} are selected for
high specialization and monosemanticity scores, positive routing
orientation, and visible weight on both matched coordinates.
Figure~\ref{fig:k25-all-activation-neurons} extends the visualization to
GELU, ReGLU, and SwiGLU.

\begin{figure}[!htbp]
    \centering
    \includegraphics[width=\linewidth]
    {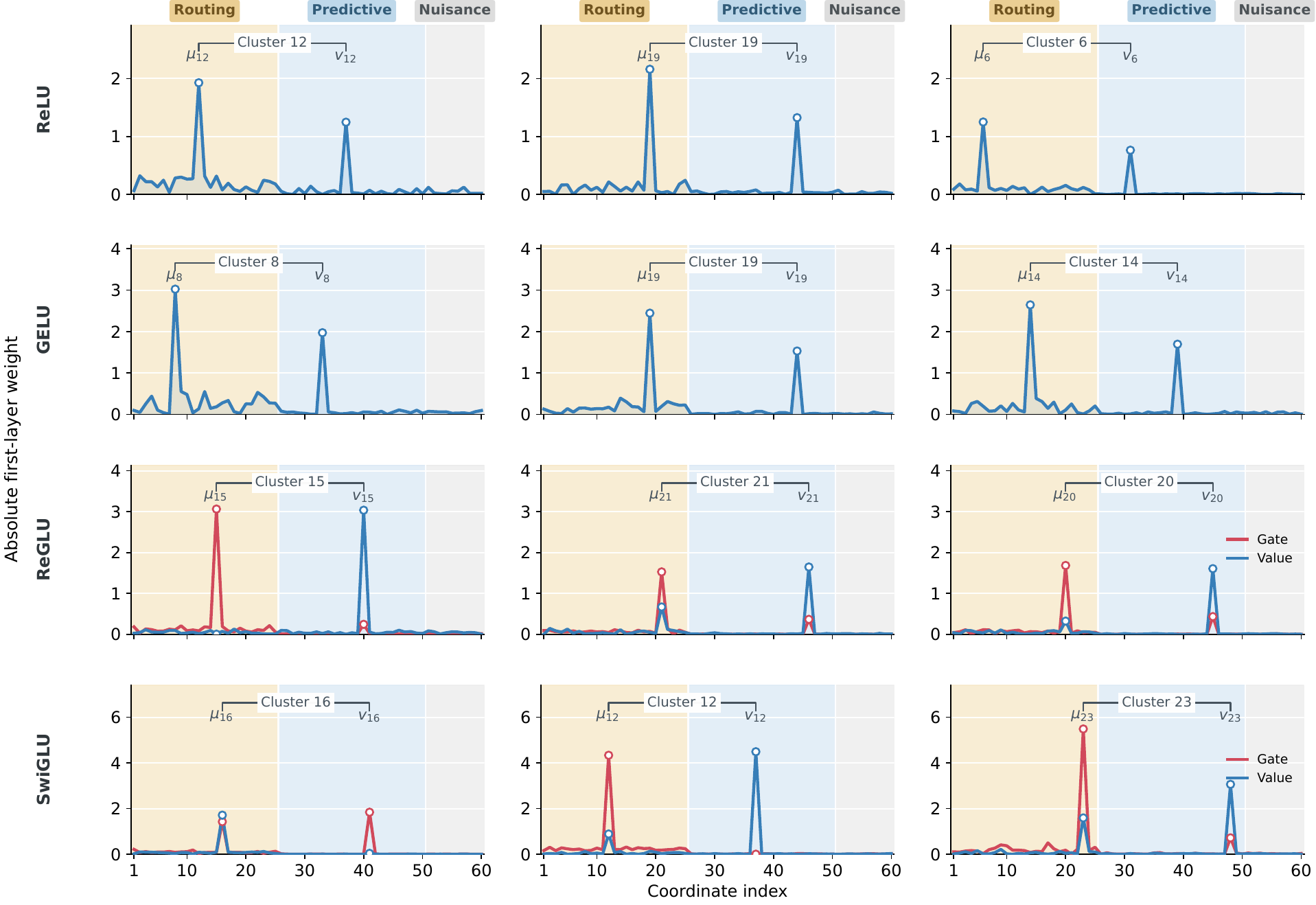}
    \caption{\textbf{Representative specialized neurons across all four activations.}
    Absolute first-layer weights in the \(K=25\), \(d=60\) experiment of
    Section~\ref{sec:coordinate-specialization}. Rows show ReLU, GELU,
    ReGLU, and SwiGLU, with three selected neurons per activation.
    Shading distinguishes routing, predictive, and nuisance coordinates.
    For ReGLU and SwiGLU, gate and value weights are shown separately.
    These examples visualize specialization to individual clusters;
    population specialization and monosemanticity statistics
    appear in Table~\ref{tab:coordinate-specialization}.}
    \label{fig:k25-all-activation-neurons}
\end{figure}

\begin{table}[htbp]
    \centering
    \small
    \setlength{\tabcolsep}{5pt}
    \renewcommand{\arraystretch}{1.12}
    \begin{tabular}{lccc}
        \toprule
        Model & Train MSE & Validation MSE & Test MSE \\
        \midrule
        ReLU
        & \(0.00401\pm0.00039\)
        & \(0.03588\pm0.00138\)
        & \(0.05436\pm0.00219\) \\
        GELU
        & \(0.00638\pm0.00073\)
        & \(0.01726\pm0.00256\)
        & \(0.02349\pm0.00146\) \\
        ReGLU
        & \(0.00136\pm0.00015\)
        & \(0.00613\pm0.00009\)
        & \(0.01185\pm0.00099\) \\
        SwiGLU
        & \(0.00265\pm0.00021\)
        & \(0.00537\pm0.00044\)
        & \(\mathbf{0.00510\pm0.00103}\) \\
        \bottomrule
    \end{tabular}
    \caption{\textbf{Predictive performance in the \(K=25\), \(d=60\)
    axis-aligned experiment.}
    Results are mean \(\pm\) standard deviation over three model seeds.}
    \label{tab:k25-axis-performance}
\end{table}
\FloatBarrier

\subsection{Clustered Single-Index and Multi-Index Models}
\label{app:experiments-3-2-3-4-details}

This subsection documents Sections~\ref{sec:local-routing} and
\ref{sec:clustered-multi-index}.

Both experiments use \(d=40\), with twenty routing coordinates and twenty
predictive coordinates and no nuisance coordinates. For
\(K\in\{1,2,10,50\}\), a sample from cluster \(c\) is
\[
    X=(s_c+\sigma_K Z_r,Z_p),
    \qquad Z_r,Z_p\stackrel{\mathrm{iid}}{\sim}\mathcal N(0,I_{20}).
\]
To generate the routing centers, we first sample \(25{,}000\) independent
standard Gaussian vectors in the \(20\)-dimensional routing subspace and
normalize each to unit norm. Starting from a uniformly selected candidate,
we construct a sequence of \(50\) centers by repeatedly selecting the unused
candidate that minimizes its largest inner product with the previously
selected centers. Equivalently, each step maximizes the minimum Euclidean
distance to the selected centers. For each \(K>1\), we use the first \(K\)
vectors of this sequence as \(s_1,\ldots,s_K\). These centers are not
orthogonalized. For \(K=1\), the routing center is set to zero.

The routing noise \(\sigma_K\) is adjusted to the closest pair of centers.
Specifically, if
\(\rho_K=\max_{c\ne r}s_c^\top s_r\), then
\(\sigma_K=(1-\rho_K)/8\); we set \(\rho_1=0\).

In the single-index experiment, each predictive direction
\(\widetilde v_c\) is generated independently by normalizing a standard
Gaussian vector in the \(20\)-dimensional predictive subspace. In the
multi-index experiment, we independently sample a \(20\times3\) standard
Gaussian matrix \(A_c\) for each cluster. Writing its reduced QR decomposition
as \(A_c=Q_cR_c\), we use the three columns of \(Q_c\) as the predictive
directions \(\widetilde v_{c,1},\widetilde v_{c,2},\widetilde v_{c,3}\).
Thus, the three directions within each cluster are orthonormal, but
directions belonging to different clusters need not be orthogonal.
Their embeddings in the full input space are \(v_c=(0,\widetilde v_c)\)
and \(v_{c,\ell}=(0,\widetilde v_{c,\ell})\).

Each cluster is assigned one link
function, sampled once and then held fixed, from
\[
    \mathcal G
    =
    \left\{
        \frac{t^2-1}{\sqrt2},
        \frac{t^3-3t}{\sqrt6},
        \sin(t),
        \tanh(t)
    \right\}.
\]
The same master link assignment is used in Experiments~\ref{sec:local-routing} and~\ref{sec:clustered-multi-index}. Its
composition for the plotted values of \(K\) is
\begin{center}
\begin{tabular}{ccccc}
    \toprule
    \(K\) & Hermite--2 & Hermite--3 & sine & \(\tanh\) \\
    \midrule
    1  & 0  & 0  & 0  & 1 \\
    2  & 0  & 0  & 0  & 2 \\
    10 & 3  & 2  & 1  & 4 \\
    50 & 17 & 13 & 12 & 8 \\
    \bottomrule
\end{tabular}
\end{center}
Experiment~\ref{sec:local-routing} uses the single-index response
\[
    Y=g_c(\widetilde v_c^\top Z_p)+\varepsilon,
\]
while Experiment~\ref{sec:clustered-multi-index} uses three orthonormal directions per cluster and the
normalized sum-of-links response
\[
    Y
    =
    \frac{1}{\sqrt3}\sum_{\ell=1}^3
    g_c(\widetilde v_{c,\ell}^\top Z_p)+\varepsilon.
\]
In both cases \(\varepsilon\sim\mathcal N(0,0.02^2)\).

For every \(K\), we save one balanced training pool of \(500{,}000\) examples,
one validation set of \(20{,}000\) examples, and one test set of \(20{,}000\)
examples. The sample sizes
\(n\in\{20{,}000,50{,}000,100{,}000,250{,}000,500{,}000\}\) are nested,
balanced prefixes of the same training pool. All MLP and kernel methods use the
same saved training, validation, and test tensors. Geometry, link assignments,
examples, and label-noise realizations are fixed across the three MLP runs;
only initialization and minibatch sampling vary.

\paragraph{MLP parameterization and initialization.}
For the MLPs in Sections~\ref{sec:coordinate-specialization},
\ref{sec:local-routing}, and~\ref{sec:clustered-multi-index}, we write
\[
    f(x)=\frac{m_0}{m}\sum_{j=1}^{m}\widetilde a_jh_j(x)+b_{\mathrm{out}},
    \qquad m_0=1024,\quad m=4096,
\]
where \(h_j\) is defined in Eqs.~\ref{eq:standard-hidden-activation}
and~\ref{eq:gated-hidden-activation}. Here \(\widetilde a_j\) is the trainable
readout parameter, and \(a_j=(m_0/m)\widetilde a_j\) is the effective readout
coefficient used in the main text.

All first-layer weight entries are initialized independently from
\(\mathcal N(0,\epsilon^2/d)\), with \(\epsilon=10^{-3}\), where \(d\) is
the input dimension of the corresponding experiment. For gated activations,
the gate and value weights are sampled independently from this distribution.
All hidden and output biases are initialized to zero.

Let \(\xi_j\) be independent signs, each uniform on \(\{-1,+1\}\).
The readout initialization is
\[
    \widetilde a_j(0)=
    \begin{cases}
        \xi_j\|w_j(0)\|_2,
        & \text{ReLU/GELU},\\[2mm]
        \xi_j\sqrt{\|w_{g,j}(0)\|_2\|w_{v,j}(0)\|_2},
        & \text{ReGLU/SwiGLU}.
    \end{cases}
\]
``Balanced'' refers to these norm-matched readout magnitudes, not to exactly
equal numbers of positive and negative readouts or an exactly zero initial
output. The output multiplier \(m_0/m=1024/4096\) multiplies the hidden sum
but not the output bias. These equations specify the parameterization
referred to here as \(\mu\)P-style.

\paragraph{MLP optimization and selection.}
We use Adam with minibatches of size \(4096\). All trainable parameters use
the same base Adam learning rate, without additional layer-specific
learning-rate multipliers. We use cosine decay from the initial learning
rate to one hundredth of that value over at most \(200{,}000\) steps.
For each activation and \(K\), we sweep
\[
    \eta\in\{10^{-4},3\!\times\!10^{-4},10^{-3},3\!\times\!10^{-3},
                 10^{-2},3\!\times\!10^{-2}\},
    \qquad
    \lambda\in\{0,10^{-6},10^{-5}\}
\]
at \(n=250{,}000\) and seed \(1000\). The pair with minimum validation MSE is
frozen for the other sample sizes and seeds. Within every final run, the
reported checkpoint is the checkpoint with minimum validation MSE. The test
set is loaded only after that checkpoint has been selected.

Validation is normally evaluated every \(500\) updates. Training is eligible
to stop after \(50{,}000\) updates when validation MSE has not improved by at
least \(0.05\%\) for \(60\) consecutive evaluations; otherwise it continues
to \(200{,}000\) updates. Regardless of when training stops, the retained
checkpoint is the one with the lowest validation MSE, including improvements
smaller than the early-stopping tolerance.

Table~\ref{tab:experiments-3-2-3-4-hyperparameters} lists the selected initial
learning rates and weight decays.
\begin{table}[htbp]
    \centering
    \small
    \setlength{\tabcolsep}{4pt}
    \resizebox{\linewidth}{!}{%
\begin{tabular}{ccllll}
        \toprule
        Experiment & \(K\) & ReLU & GELU & ReGLU & SwiGLU \\
        \midrule
        \ref{sec:local-routing} & 1  & \(10^{-3}/10^{-5}\) & \(10^{-4}/10^{-6}\) & \(10^{-4}/10^{-5}\) & \(3\!\times\!10^{-4}/10^{-6}\) \\
        \ref{sec:local-routing} & 2  & \(10^{-4}/0\) & \(3\!\times\!10^{-4}/0\) & \(10^{-2}/10^{-6}\) & \(3\!\times\!10^{-3}/10^{-6}\) \\
        \ref{sec:local-routing} & 10 & \(10^{-2}/10^{-6}\) & \(10^{-2}/10^{-6}\) & \(10^{-3}/10^{-5}\) & \(3\!\times\!10^{-3}/0\) \\
        \ref{sec:local-routing} & 50 & \(10^{-2}/0\) & \(10^{-2}/10^{-6}\) & \(3\!\times\!10^{-3}/10^{-6}\) & \(3\!\times\!10^{-3}/0\) \\
        \midrule
        \ref{sec:clustered-multi-index} & 1  & \(3\!\times\!10^{-3}/10^{-5}\) & \(3\!\times\!10^{-3}/10^{-6}\) & \(3\!\times\!10^{-2}/10^{-6}\) & \(10^{-4}/10^{-6}\) \\
        \ref{sec:clustered-multi-index} & 2  & \(10^{-4}/0\) & \(10^{-3}/0\) & \(3\!\times\!10^{-3}/10^{-6}\) & \(3\!\times\!10^{-3}/10^{-6}\) \\
        \ref{sec:clustered-multi-index} & 10 & \(10^{-2}/10^{-6}\) & \(10^{-2}/10^{-6}\) & \(10^{-3}/10^{-5}\) & \(3\!\times\!10^{-3}/0\) \\
        \ref{sec:clustered-multi-index} & 50 & \(3\!\times\!10^{-2}/0\) & \(3\!\times\!10^{-2}/10^{-6}\) & \(10^{-2}/10^{-6}\) & \(10^{-2}/10^{-6}\) \\
        \bottomrule
    \end{tabular}%
    }
    \caption{Validation-selected MLP hyperparameters. Each entry is
    initial learning rate/weight decay. Selection uses only the validation set
    at \(n=250{,}000\) and seed \(1000\).}
    \label{tab:experiments-3-2-3-4-hyperparameters}
\end{table}

For the small-sample
\(K=50\) cases in which overfitting occurred before step \(500\), we repeated
the same hyperparameters with validation evaluated every \(50\) steps. These
replacements are GELU at \(n=20{,}000,50{,}000\) and SwiGLU at
\(n=20{,}000\) in Experiment~\ref{sec:local-routing}, and SwiGLU at
\(n=20{,}000,50{,}000\) in Experiment~\ref{sec:clustered-multi-index}. These runs are eligible to stop
after \(5{,}000\) updates and use a patience of \(100\) evaluations. No test
result was used to decide whether or where to checkpoint these runs.

\paragraph{Kernel optimization and selection.}
Iteration zero is the isotropic Laplace kernel. RFM performs up to three AGOP
metric updates, estimating each AGOP from at most \(20{,}000\) training
examples. Kernel bandwidth is the median transformed pairwise distance among
up to \(512\) sampled training inputs. Fits with more than \(2{,}000\) examples
use EigenPro~\citep{ma2017diving,ma2019kernel,abedsoltan2023toward,abedsoltan2025fast}; smaller local components use a direct ridge solve with ridge
\(10^{-6}\). EigenPro uses a training-MSE target of \(4\times10^{-4}\), a
minimum of \(20\) epochs, \(10\)-epoch stages, and at most \(5{,}000\) epochs,
with adaptive learning-rate reduction when its training monitor stalls. For
global methods, one predictor is fit to all training examples.
For local methods, one predictor is fit per true cluster. The RFM iteration is
selected by minimum validation MSE; for local methods the criterion is the
validation-size-weighted MSE pooled across clusters. The test set is evaluated
only after the iteration is frozen. Because the kernel methods use one fixed
data realization and deterministic seeds, their curves have no across-run
shading.

\paragraph{Monosemanticity.}
We evaluate the monosemanticity score \(M_j\) at \(K=50\) and \(n=500{,}000\).
For a standard MLP row \(w_j=(w_{r,j},w_{p,j})\), the cluster-specific specialization scores are
\[
    S_{j,c}^{\mathrm{single}}
    =
    \frac{\sqrt{\langle w_{r,j},s_c\rangle^2
    +\langle w_{p,j},\widetilde v_c\rangle^2}}{\|w_j\|_2}
\]
for Experiment~\ref{sec:local-routing}, and
\[
    S_{j,c}^{\mathrm{multi}}
    =
    \frac{\sqrt{\langle w_{r,j},s_c\rangle^2
    +\sum_{\ell=1}^3\langle w_{p,j},\widetilde v_{c,\ell}\rangle^2}}
    {\|w_j\|_2}
\]
for Experiment~\ref{sec:clustered-multi-index}. We set \(S_j=\max_c S_{j,c}\). For gated networks, the
score for a cluster is the minimum of the corresponding gate- and value-row
scores, so both rows must align with the same cluster. Active units follow the experiment-specific rule in
Appendix~\ref{app:shared-specialization-analysis}.

For an active neuron passing a stated specialization threshold, let \(h_j(X)\)
denote its hidden activation before multiplication by its final readout
coefficient. Its activation energy within cluster \(r\) is
\(\mathbb E[h_j(X)^2\mid C=r]\), the mean squared activation conditional
on cluster membership. Let \(c\) denote the cluster to which it is specialized. Normalizing
the energy within cluster \(c\) by the sum across all clusters gives
\[
    M_j
    =
    \frac{\mathbb E[h_j(X)^2\mid C=c]}
    {\sum_{r=1}^{K}\mathbb E[h_j(X)^2\mid C=r]}.
\]
For ReLU and GELU, \(h_j(x)=\phi(w_j^\top x+b_j)\). For ReGLU and SwiGLU,
\(h_j(x)=\phi(w_{g,j}^\top x+b_{g,j})(w_{v,j}^\top x+b_{v,j})\).
For equally likely clusters, the monosemanticity score \(M_j\)
is the fraction of the neuron's total activation energy attributable to inputs
from the cluster to which it is specialized. The uniform
baseline is \(1/K\), where \(K\) is the total number of clusters, and \(M_j>1/2\) means that more than half of the
neuron's activation energy is on that cluster. Tables average these
per-neuron fractions over specialized neurons and report the percentage of
specialized neurons satisfying \(M_j>1/2\).

The results are reported in Tables~\ref{tab:single-index-specialization} and~\ref{tab:multi-index-specialization} and
Figures~\ref{fig:q2-histograms-single-index} and
\ref{fig:q2-histograms-multi-index}. Each table first computes neuron-level
summaries within each run, then reports their mean and sample standard
deviation across the three runs. Histograms pool selected neurons across
runs, so runs with more selected neurons contribute more observations.
\FloatBarrier

\paragraph{Sample complexity across all four architectures.}
Figures~\ref{fig:single-index-all-activations} and
\ref{fig:multi-index-all-activations} extend the ReLU/ReGLU comparisons in
Sections~\ref{sec:local-routing} and~\ref{sec:clustered-multi-index} to
GELU and SwiGLU, using the same datasets, baselines, and evaluation protocol.

\begin{figure}[htbp]
    \centering
    \includegraphics[width=\linewidth]
    {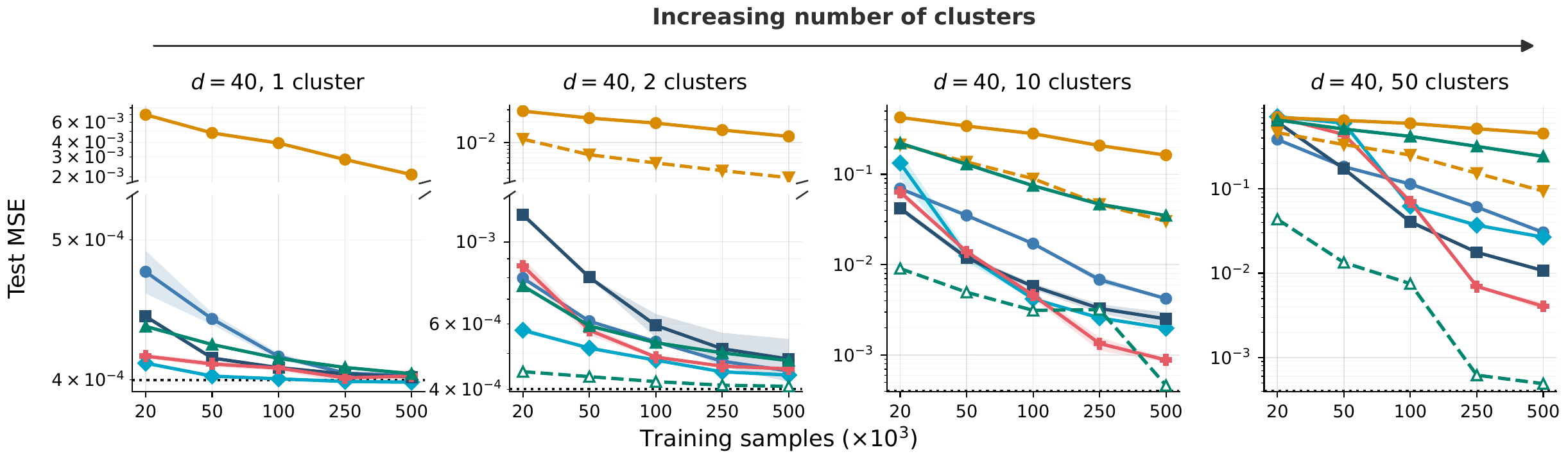}
    \par\smallskip
    \includegraphics[width=\linewidth]
    {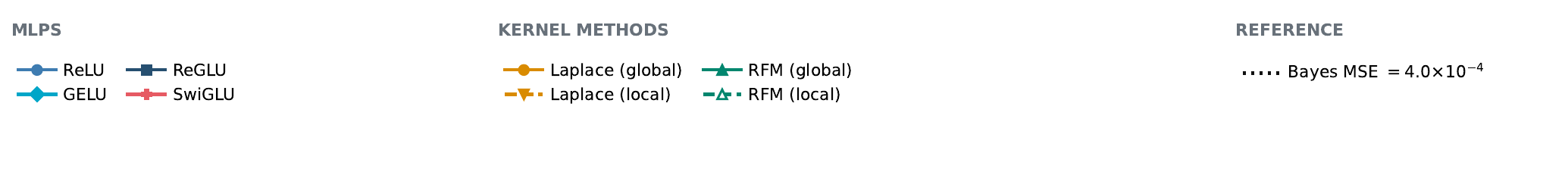}
    \caption{\textbf{Single-index sample complexity across all four architectures.}
    ReLU, GELU, ReGLU, and SwiGLU are compared with global and oracle local
    kernel methods on the fixed \(d=40\) datasets of
    Section~\ref{sec:local-routing}. MLP curves show mean test MSE over three
    training runs; shading denotes one sample standard deviation. Kernel
    curves use deterministic seeds, and local methods receive true cluster
    identities. The dotted line marks the observation-noise variance
    \(0.02^2\). Broken vertical axes separate error scales in the first two panels.}
    \label{fig:single-index-all-activations}
\end{figure}

\begin{figure}[htbp]
    \centering
    \includegraphics[width=\linewidth]
    {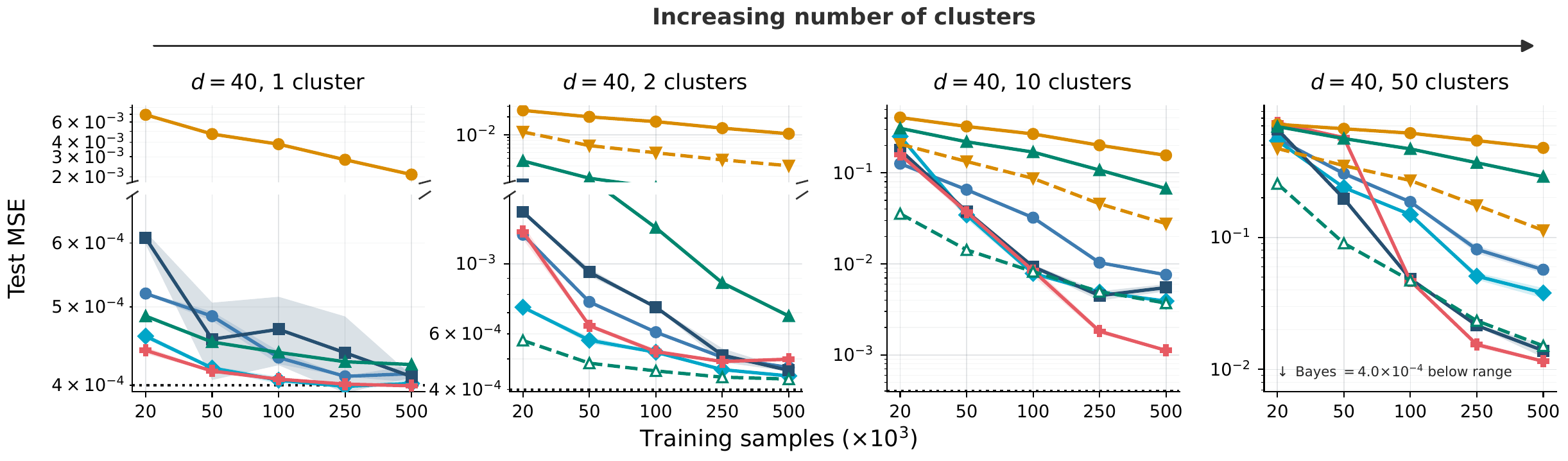}
    \par\smallskip
    \includegraphics[width=\linewidth]
    {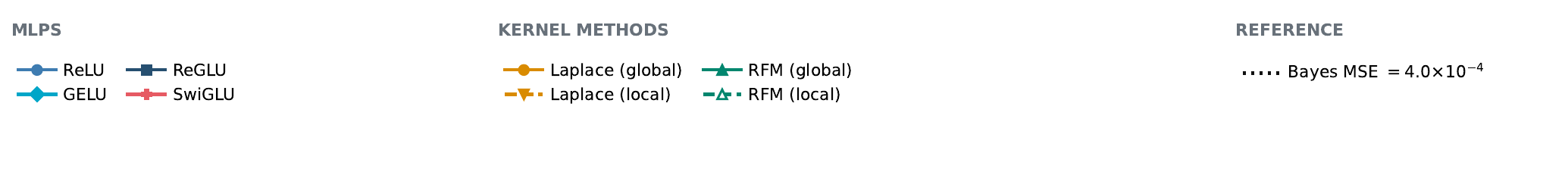}
    \caption{\textbf{Multi-index sample complexity across all four architectures.}
    ReLU, GELU, ReGLU, and SwiGLU are compared with global and oracle local
    kernel methods on the fixed \(d=40\) datasets of
    Section~\ref{sec:clustered-multi-index}, with three predictive directions
    per cluster. MLP curves show mean test MSE over three training runs;
    shading denotes one sample standard deviation. Kernel curves use
    deterministic seeds, and local methods receive true cluster identities.
    The observation-noise variance \(0.02^2\) lies below the plotted range
    in the \(K=50\) panel. Broken vertical axes separate error scales in
    the first two panels.}
    \label{fig:multi-index-all-activations}
\end{figure}

\subsection{Additional Details for Trained MLPs Encoding Cluster Structure}
\label{app:learned-gating-details}

This appendix supplements the experiment in
Section~\ref{sec:learned-gating}, which uses a separate \(d=20\)
clustered single-index dataset. We describe how cluster gates are
extracted from a trained ReLU MLP and reused to construct independent local
Laplace and RFM predictors.

\paragraph{Clustered data with mixed link functions.}
We use \(d=20\), with ten routing and ten predictive coordinates, and
\(K\in\{2,5,10,50\}\) equally represented clusters. Cluster-center
vectors \(s_c\in\mathbb S^9\) are obtained by farthest-point sampling from
\(\max\{20{,}000,500K\}\) independent Gaussian vectors normalized to unit
norm. Predictive directions \(\widetilde v_c\) are sampled independently
and uniformly from \(\mathbb S^9\). With
\[
    \rho_K=\max_{c\ne r}s_c^\top s_r,\qquad
    \sigma_K=\frac{1-\rho_K}{8},
\]
we sample
\[
    X=(s_c+\sigma_K Z_r,Z_p),\qquad
    Z_r,Z_p\stackrel{\mathrm{iid}}{\sim}\mathcal N(0,I_{10}),
    \qquad Y=g_c(\widetilde v_c^\top Z_p)+\varepsilon.
\]
Each cluster's link function is sampled independently and uniformly from
normalized Hermite-3, sine, and hyperbolic tangent, and
\(\varepsilon\sim\mathcal N(0,0.02^2)\). This data-generation and training
protocol is specific to the learned-gating experiment.

\paragraph{Datasets and repetitions.}
For every \(K\), we generate a balanced training pool containing
\(200{,}000\) examples and use nested prefixes of sizes
\[
    n\in\{20{,}000,50{,}000,100{,}000,200{,}000\}.
\]
Evaluation uses an independently generated test set containing \(10{,}000\)
examples. Within each run, exactly the same serialized training and test
tensors are used by the MLP, global kernel methods, ground-truth local
methods, and MLP-gated methods. RFM iteration selection uses the same
validation set as the MLPs.

Results are averaged over three seeds. For a fixed \(K\), the cluster-position
and predictive directions are held fixed, while each seed independently
resamples the cluster-specific link-function assignments, training and test
examples,
label noise, and MLP initialization. No test labels are used to construct the
gates.

\paragraph{Source ReLU MLP.}
For each \(K\), sample size, and seed, we train a one-hidden-layer ReLU network
of width \(4{,}096\),
\[
    \widehat f(x)
    =
    \sum_{j=1}^{4096}
    a_j\operatorname{ReLU}(w_j^\top x+b_j)
    +b_{\mathrm{out}}.
\]
All weights and biases are trainable. Optimization uses Adam with minibatches
of size \(4{,}096\), zero weight decay, and cosine learning-rate decay from
\(10^{-2}\) to \(10^{-4}\). Training runs for at least \(20{,}000\) steps and
at most \(200{,}000\) steps. It terminates when the full training MSE reaches
\(4\times10^{-4}\), or when the training loss plateaus. We restore the
checkpoint with the smallest full training MSE. Neither test MSE nor cluster
identity is used for checkpoint selection.

\paragraph{Selecting active neurons.}
After training, we freeze the ReLU MLP and select the active-neuron set
\(\mathcal J\) using the shared importance-based criterion in
Section~\ref{app:shared-specialization-analysis}.
Here, these neurons are used to construct contribution profiles for
clustering inputs and gate-specific features for the local predictors,
rather than to measure alignment with the predictive directions.

\paragraph{Contribution profiles and learned gates.}
For every training input \(x\), we compute its normalized contribution profile
over the active neurons:
\begin{equation}
    p_j(x)
    =
    \frac{
        |a_j|\operatorname{ReLU}(w_j^\top x+b_j)
    }{
        \displaystyle
        \sum_{\ell\in\mathcal J}
        |a_\ell|\operatorname{ReLU}(w_\ell^\top x+b_\ell)
    },
    \qquad j\in\mathcal J.
    \label{eq:appendix-contribution-profile}
\end{equation}
Thus \(p(x)\) describes which active first-layer neurons contribute to the
MLP prediction on input \(x\), after normalization by the total absolute readout-weighted
activation across the selected neurons.

We apply \(K\)-means++ to the training contribution profiles and use
\(M=K\) learned gates. Knowledge of the number of clusters is therefore
provided, but the true cluster assignments are never observed. We use eight
initializations and at most 30 Lloyd iterations, retaining the solution with
the smallest training inertia. Every training example is assigned to exactly
one gate through its nearest centroid.

Let \(q_c\) denote the centroid of learned gate \(c\), with coordinate
\(q_{cj}\) corresponding to active neuron \(j\). Although examples receive
hard gate assignments, the centroid coordinates provide a soft association
between neurons and gates: the same neuron may contribute to several gates.

\paragraph{Constructing gate-specific features.}
For each learned gate \(c\), we form the positive-semidefinite metric
\begin{equation}
    G_c
    =
    \frac{
        \displaystyle
        \sum_{j\in\mathcal J}q_{cj}w_jw_j^\top
    }{
        \displaystyle
        \sum_{j\in\mathcal J}q_{cj}
    }.
    \label{eq:appendix-mlp-derived-metric}
\end{equation}
The corresponding gate-specific representation is
\begin{equation}
    z_c(x)=G_c^{1/2}x.
    \label{eq:appendix-mlp-derived-projection}
\end{equation}
The MLP biases affect the contribution profiles and hence the gate
assignments, while \(G_c\) itself is constructed from the incoming
first-layer weight directions.

For every learned gate, we project only the training examples assigned to
that gate and fit an independent local Laplace or RFM predictor in the
resulting representation.

\paragraph{Test-time prediction.}
For a new test input \(x\), we:

\begin{enumerate}
    \item pass \(x\) through the frozen MLP and compute its contribution
    profile \(p(x)\);
    \item assign \(x\) to the nearest fixed training centroid;
    \item transform it using the corresponding representation
    \(z_c(x)=G_c^{1/2}x\); and
    \item evaluate the local predictor fitted for that learned gate.
\end{enumerate}

The test input is therefore assigned using only its MLP contribution profile.
Neither its response nor its true cluster identity is used.

\paragraph{Local Laplace and RFM predictors.}
The local Laplace predictor for gate \(c\) uses
\[
    k_c(x,x')
    =
    \exp\!\left(
        -\frac{
            \|G_c^{1/2}(x-x')\|_2
        }{h_c}
    \right),
\]
where \(h_c\) is initialized using the median pairwise distance within that
gate. We use ridge parameter \(10^{-6}\).

The local RFM begins from this Laplace kernel and performs three metric
updates, producing iterations \(0,1,2,3\), where iteration \(0\) is the local
Laplace predictor. After each update, the bandwidth is recomputed using
distances under the updated metric. Kernel systems with fewer than
\(11{,}000\) local training examples are solved directly; larger systems use
EigenPro~\citep{ma2017diving,ma2019kernel,abedsoltan2023toward,abedsoltan2025fast}. The same procedure is used for the global and ground-truth local
RFM baselines. We select the RFM iteration among \(0,1,2,3\) with the lowest
validation MSE. The plotted RFM results report the test MSE of the selected
iteration.

\paragraph{Compared methods.}
Figure~\ref{fig:local-sample-complexity_kmeans} compares:

\begin{itemize}
    \item the source ReLU MLP;
    \item one global Laplace predictor fitted to all training examples;
    \item one global RFM fitted to all training examples;
    \item MLP-gated local Laplace, using the gates and projections extracted
    from the ReLU MLP;
    \item MLP-gated local RFM, using the same extracted gates and projections;
    \item local Laplace, which is given the true cluster identities at
    training and test time;
    \item local RFM, defined analogously; and
    \item the irreducible noise level, whose expected MSE is
    \(0.02^2=4\times10^{-4}\).
\end{itemize}

The local methods using the true clusters are diagnostic references rather
than fair deployable predictors because they receive the true cluster
identity. By contrast, the
MLP-gated methods receive only the training input--response pairs and the
known number \(K\) of gates.

\paragraph{Aggregation.}
Every curve in Figure~\ref{fig:local-sample-complexity_kmeans} is the
arithmetic mean over three independent runs. Each run uses a newly sampled
dataset from the same generative model and an independently initialized MLP.
Cluster-specific link assignments, training and test examples, label noise,
K-means initializations, and optimization minibatch orders also vary across runs.

\FloatBarrier
\subsection{Random-Orientation Specialization Baseline}
\label{app:random-orientation-specialization}

The learned specialization rates greatly exceed a random-orientation
baseline. At the main threshold \(S_j\geq0.9\), the mean percentage of
specialized active neurons under this baseline is at most
\(2.49\times10^{-5}\%\) across the three experiments and four architectures
(Table~\ref{tab:specialization-random-orientation-null}).

\paragraph{Randomization.}
Write each active learned weight as \(w_j=(w_{r,j},w_{p,j},w_{n,j})\),
with routing, predictive, and nuisance components; the last block is absent
in Sections~\ref{sec:local-routing} and~\ref{sec:clustered-multi-index}.
We replace
\[
    w_{r,j}\longmapsto\|w_{r,j}\|_2 u_{r,j},
    \qquad
    w_{p,j}\longmapsto\|w_{p,j}\|_2 u_{p,j},
\]
where the two unit vectors are drawn independently and uniformly within
their respective blocks. Nuisance weights, total norms, cluster means, and
predictive directions remain fixed. Gate and value weights are randomized
independently. We retain the active set and specialization score definitions from the main
analysis (Appendix~\ref{app:shared-specialization-analysis}), including the
same-cluster requirement for gated branches. This controls for specialization
arising merely from the allocation of weight norm across blocks.

\paragraph{Evaluation and results.}
For each of the three trained models per setting, we generate
\(B=5{,}000\) independent randomizations and recompute the percentage of
active neurons with \(S_j\geq0.7\) or \(S_j\geq0.9\).
Table~\ref{tab:specialization-random-orientation-null} reports means and
sample standard deviations over three models. For the null, we first average
the randomizations within each model; its standard deviation is across these
three averages, not all \(15{,}000\) draws.

\begin{table}[!htbp]
    \centering
    \small
    \setlength{\tabcolsep}{4.2pt}
    \renewcommand{\arraystretch}{1.08}
    \resizebox{\textwidth}{!}{%
    \begin{tabular}{llccccc}
        \toprule
        Experiment & Model & \shortstack{Active neurons\\(out of 4096)}
        & \multicolumn{2}{c}{\shortstack{Specialization score \(S_j\)\\$\geq0.7$ (\% active) \(\uparrow\)}}
        & \multicolumn{2}{c}{\shortstack{Specialization score \(S_j\)\\$\geq0.9$ (\% active) \(\uparrow\)}} \\
        \cmidrule(lr){4-5}\cmidrule(lr){6-7}
        & & & Trained & Random orientation & Trained & Random orientation \\
        \midrule
        \ref{sec:coordinate-specialization} & ReLU
        & $3879\pm30$ & $32.36\pm0.19$ & $0.0169\pm0.0005$
        & $15.71\pm0.66$ & $0$ \\
        & GELU
        & $431\pm19$ & $87.56\pm0.44$ & $0.0682\pm0.0019$
        & $77.66\pm0.49$ & $0$ \\
        & ReGLU
        & $406\pm3$ & $47.49\pm1.12$ & $0$
        & $44.30\pm0.69$ & $0$ \\
        & SwiGLU
        & $249\pm16$ & $63.67\pm4.23$ & $0$
        & $57.68\pm4.90$ & $0$ \\
        \midrule
        \ref{sec:local-routing} & ReLU
        & $3332\pm69$ & $32.48\pm0.54$ & $0.9316\pm0.0076$
        & $22.19\pm0.62$ & $(2.21\pm1.55)\!\times\!10^{-5}$ \\
        & GELU
        & $502\pm5$ & $95.69\pm1.06$ & $0.9390\pm0.0023$
        & $74.59\pm1.38$ & $0$ \\
        & ReGLU
        & $399\pm8$ & $81.51\pm1.22$ & $(1.84\pm0.32)\!\times\!10^{-4}$
        & $75.25\pm1.77$ & $0$ \\
        & SwiGLU
        & $692\pm15$ & $54.81\pm1.74$ & $(2.32\pm0.80)\!\times\!10^{-4}$
        & $39.89\pm1.37$ & $0$ \\
        \midrule
        \ref{sec:clustered-multi-index} & ReLU
        & $3378\pm20$ & $29.04\pm0.11$ & $1.0778\pm0.0038$
        & $24.28\pm0.37$ & $(3.94\pm6.83)\!\times\!10^{-6}$ \\
        & GELU
        & $1068\pm16$ & $85.39\pm0.68$ & $1.0618\pm0.0233$
        & $70.64\pm2.09$ & $(2.49\pm1.07)\!\times\!10^{-5}$ \\
        & ReGLU
        & $387\pm5$ & $98.19\pm0.23$ & $0.00661\pm0.00022$
        & $95.18\pm0.24$ & $0$ \\
        & SwiGLU
        & $483\pm6$ & $87.17\pm0.94$ & $0.00469\pm0.00068$
        & $63.08\pm0.71$ & $0$ \\
        \bottomrule
    \end{tabular}%
    }
    \caption{\textbf{Specialization greatly exceeds a block-norm-matched
    random-orientation baseline.}
    Entries report the percentage of active neurons whose specialization score
    \(S_j\) meets the stated threshold. Both gated branches must meet the cutoff
    for the same cluster.
    \(\uparrow\) indicates a larger fraction of specialized neurons.
    Experiment~\ref{sec:coordinate-specialization} uses the axis-aligned single-index model with
    \(d=60\), \(K=25\), and \(n=500{,}000\).  Experiment~\ref{sec:local-routing} uses the dense
    single-index model with \(d=40\), \(K=50\), and \(n=500{,}000\).
    Experiment~\ref{sec:clustered-multi-index} uses the corresponding rank-three multi-index model with
    \(d=40\), \(K=50\), and \(n=500{,}000\).  All three use the mixed
    Hermite--2, Hermite--3, sine, and hyperbolic-tangent link family.  Results
    are mean $\pm$ sample standard deviation over three trained models.  The
    null uses \(5{,}000\) randomizations per model. Zero entries mean that no
    randomized active neuron met the threshold in the sampled randomizations.}
    \label{tab:specialization-random-orientation-null}
\end{table}

For each trained model and threshold, we additionally compute the one-sided
Monte Carlo value
\[
    p_{\mathrm{MC}}
    =
    \frac{
        1+\#\{b:\widehat p_b^{\mathrm{null}}
        \geq\widehat p^{\mathrm{trained}}\}
    }{B+1},
\]
where \(\widehat p\) is the percentage of active neurons exceeding the
threshold.  None of the \(5{,}000\) randomizations matched the trained-model
percentage for any activation, experiment, seed, or threshold.  Hence every
comparison attains the finite-resolution bound
\(p_{\mathrm{MC}}=1/5001\approx2.0\times10^{-4}\).

\FloatBarrier

% Requires: booktabs, graphicx
\subsection{Active Neurons Retain Predictive Information}
\label{app:active-neuron-readout-refit}

This experiment tests whether the neurons identified by the
\(99.9\%\)-importance criterion retain the predictive information learned by
the full network.  Starting from the final validation-selected checkpoints in
Experiments~\ref{sec:coordinate-specialization} and~\ref{sec:local-routing}, we remove every neuron outside the fixed active set,
freeze the retained hidden representation, and refit only the scalar readout
weights and output bias.  Thus, recovery cannot be attributed to learning new
first-layer features.

\paragraph{Main findings.}
Table~\ref{tab:active-neuron-readout-refit} summarizes test performance.
In Experiment~\ref{sec:coordinate-specialization}, readout-only refitting restores the ReLU and GELU models to
near their full-model test errors, while pruning has almost no effect on ReGLU
or SwiGLU even before refitting.  In Experiment~\ref{sec:local-routing}, ReLU, ReGLU, and SwiGLU
again match the corresponding full models after readout refitting.  GELU also
improves substantially, from test MSE \(0.0357\) before refitting to
\(0.0299\), close to the full-model value \(0.0268\).  Overall, these results support the active-neuron criterion as retaining most
of the learned predictive information, although GELU retains a gap in
Experiment~\ref{sec:local-routing}. This is evidence for approximate
sufficiency of the retained representation after readout recalibration;
it does not imply that every retained neuron is individually necessary.

\begin{table}[!htbp]
    \centering
    \small
    \setlength{\tabcolsep}{5pt}
    \renewcommand{\arraystretch}{1.08}
    \begin{tabular}{llcccc}
        \toprule
        Experiment & Model & \shortstack{Active neurons\\(out of 4096)}
        & Full model & Active only & \shortstack{Active only,\\readout refit} \\
        \midrule
        \ref{sec:coordinate-specialization} & ReLU
        & $3879\pm30$ & $0.0544\pm0.0022$ & $0.0851\pm0.0104$ & $0.0537\pm0.0020$ \\
        & GELU
        & $431\pm19$ & $0.0235\pm0.0015$ & $0.0297\pm0.0026$ & $0.0240\pm0.0018$ \\
        & ReGLU
        & $406\pm3$ & $0.0119\pm0.0010$ & $0.0119\pm0.0010$ & $0.0118\pm0.0010$ \\
        & SwiGLU
        & $249\pm16$ & $0.0051\pm0.0010$ & $0.0052\pm0.0010$ & $0.0052\pm0.0010$ \\
        \midrule
        \ref{sec:local-routing} & ReLU
        & $3332\pm69$ & $0.0305\pm0.0004$ & $0.0305\pm0.0004$ & $0.0301\pm0.0003$ \\
        & GELU
        & $502\pm5$ & $0.0268\pm0.0003$ & $0.0357\pm0.0046$ & $0.0299\pm0.0017$ \\
        & ReGLU
        & $399\pm8$ & $0.0107\pm0.0004$ & $0.0107\pm0.0004$ & $0.0106\pm0.0003$ \\
        & SwiGLU
        & $692\pm15$ & $0.0041\pm0.0004$ & $0.0041\pm0.0005$ & $0.0041\pm0.0004$ \\
        \bottomrule
    \end{tabular}
    \caption{\textbf{Readout refitting restores the active-neuron models close
    to the corresponding full models.}
    ``Active only'' evaluates the retained neurons with their original
    readout coefficients.  ``Active only, readout refit'' freezes the retained
    hidden representation and refits only its scalar readout coefficients and
    output bias.  Entries are test MSE, reported as mean $\pm$ sample standard
    deviation over three model seeds.}
    \label{tab:active-neuron-readout-refit}
\end{table}

\begin{table}[!htbp]
    \centering
    \small
    \setlength{\tabcolsep}{4pt}
    \renewcommand{\arraystretch}{1.08}
    \begin{tabular}{lllccc}
        \toprule
        Experiment & Model & Readout setting & Train & Validation & Test \\
        \midrule
        \ref{sec:coordinate-specialization} & ReLU & Full model & $.00401\pm.00039$ & $.03588\pm.00138$ & $.05436\pm.00219$ \\
         &  & Active only & $.03557\pm.00919$ & $.06728\pm.01024$ & $.08514\pm.01039$ \\
         &  & Readout refit & $.00363\pm.00028$ & $.03560\pm.00120$ & $.05369\pm.00205$ \\
        \addlinespace[3pt]
        \ref{sec:coordinate-specialization} & GELU & Full model & $.00638\pm.00073$ & $.01726\pm.00256$ & $.02349\pm.00146$ \\
         &  & Active only & $.01264\pm.00168$ & $.02333\pm.00364$ & $.02969\pm.00263$ \\
         &  & Readout refit & $.00737\pm.00096$ & $.01805\pm.00283$ & $.02400\pm.00177$ \\
        \addlinespace[3pt]
        \ref{sec:coordinate-specialization} & ReGLU & Full model & $.00136\pm.00015$ & $.00613\pm.00009$ & $.01185\pm.00099$ \\
         &  & Active only & $.00136\pm.00015$ & $.00613\pm.00009$ & $.01186\pm.00099$ \\
         &  & Readout refit & $.00134\pm.00013$ & $.00611\pm.00011$ & $.01182\pm.00096$ \\
        \addlinespace[3pt]
        \ref{sec:coordinate-specialization} & SwiGLU & Full model & $.00265\pm.00021$ & $.00537\pm.00044$ & $.00510\pm.00103$ \\
         &  & Active only & $.00276\pm.00022$ & $.00546\pm.00046$ & $.00520\pm.00104$ \\
         &  & Readout refit & $.00274\pm.00022$ & $.00542\pm.00044$ & $.00516\pm.00104$ \\
        \midrule
        \ref{sec:local-routing} & ReLU & Full model & $.00374\pm.00095$ & $.03213\pm.00132$ & $.03048\pm.00035$ \\
         &  & Active only & $.00374\pm.00095$ & $.03213\pm.00132$ & $.03048\pm.00035$ \\
         &  & Readout refit & $.00329\pm.00067$ & $.03188\pm.00124$ & $.03006\pm.00034$ \\
        \addlinespace[3pt]
        \ref{sec:local-routing} & GELU & Full model & $.01303\pm.00019$ & $.03067\pm.00063$ & $.02677\pm.00028$ \\
         &  & Active only & $.02213\pm.00457$ & $.04006\pm.00423$ & $.03575\pm.00455$ \\
         &  & Readout refit & $.01626\pm.00162$ & $.03401\pm.00112$ & $.02989\pm.00171$ \\
        \addlinespace[3pt]
        \ref{sec:local-routing} & ReGLU & Full model & $.00277\pm.00031$ & $.01008\pm.00063$ & $.01073\pm.00042$ \\
         &  & Active only & $.00277\pm.00031$ & $.01007\pm.00063$ & $.01073\pm.00042$ \\
         &  & Readout refit & $.00242\pm.00026$ & $.00980\pm.00056$ & $.01058\pm.00030$ \\
        \addlinespace[3pt]
        \ref{sec:local-routing} & SwiGLU & Full model & $.00157\pm.00006$ & $.00512\pm.00069$ & $.00406\pm.00044$ \\
         &  & Active only & $.00167\pm.00009$ & $.00522\pm.00066$ & $.00415\pm.00047$ \\
         &  & Readout refit & $.00163\pm.00007$ & $.00516\pm.00069$ & $.00411\pm.00045$ \\
        \bottomrule
    \end{tabular}
    \caption{Train, validation, and test MSE for the full model, the pruned
    active-neuron model before refitting, and the validation-selected
    readout-refitted model. Entries are mean $\pm$ sample standard deviation
    over three model seeds. Active-neuron counts are in
    Table~\ref{tab:active-neuron-readout-refit}.}
    \label{tab:active-neuron-readout-refit-all-splits}
\end{table}

\paragraph{Source models and data.}
We use the largest-sample specialization setting from each experiment.  The
data, link assignments, and train--validation--test splits are fixed across the
three model seeds.
\begin{itemize}
    \item \textbf{Experiment~\ref{sec:coordinate-specialization}:} \(d=60\), \(K=25\), and
    \(n=500{,}000\), with 25 routing coordinates, 25 predictive coordinates,
    and 10 nuisance coordinates.  The cluster means and predictive
    directions are matched coordinate axes.  The routing-noise standard
    deviation is \(0.125\).

    \item \textbf{Experiment~\ref{sec:local-routing}:} \(d=40\), \(K=50\), and
    \(n=500{,}000\), with 20 routing coordinates and 20 predictive
    coordinates.  The cluster means and predictive directions are dense
    unit vectors.  The routing-noise standard deviation is approximately
    \(0.0883\).
\end{itemize}
Both experiments use balanced clusters, observation noise
\(\varepsilon\sim\mathcal N(0,0.02^2)\), and independent validation and test
sets of \(20{,}000\) examples each.  Each cluster has a fixed link selected
from
\[
    \left\{
        t\mapsto\frac{t^2-1}{\sqrt 2},\quad
        t\mapsto\frac{t^3-3t}{\sqrt 6},\quad
        t\mapsto\sin(t),\quad
        t\mapsto\tanh(t)
    \right\}.
\]
The numbers of clusters using these four links are respectively
\((9,7,3,6)\) in Experiment~\ref{sec:coordinate-specialization} and \((17,13,12,8)\) in Experiment~\ref{sec:local-routing}.
All source MLPs have width \(4096\), and we analyze the three model seeds
\(1000,1001,1002\).

\paragraph{Pruning and readout refitting.}
The active set \(A\) is the smallest importance-ranked set accounting for
\(99.9\%\) of total importance, using the scores in
Appendix~\ref{app:shared-specialization-analysis}. It is computed once from
the original checkpoint and remains fixed. With hidden activations \(h_j\)
as defined in Eqs.~\ref{eq:standard-hidden-activation}
and~\ref{eq:gated-hidden-activation}, the pruned model is
\[
    f_A(x)=\frac{1024}{4096}\sum_{j\in A}\widetilde a_j h_j(x)
        +\widetilde b_{\mathrm{out}}.
\]
We preserve the original \(1024/4096\) output multiplier rather than
rescaling it by the number of retained neurons. All retained first-layer
weights and hidden biases are copied from the source checkpoint and frozen.
Only the readout coefficients \(\{\widetilde a_j:j\in A\}\) and output bias
are trainable, initialized from their checkpoint values.

The readout is optimized on all \(500{,}000\) training examples using Adam
with \((\beta_1,\beta_2)=(0.9,0.999)\), \(\epsilon=10^{-8}\), zero weight
decay, and minibatches of size \(4096\) sampled with replacement.  We consider
initial learning rates
\[
    \eta_0\in
    \{10^{-4},3\times10^{-4},10^{-3},3\times10^{-3},10^{-2}\}.
\]
Every candidate follows the same cosine schedule from \(\eta_0\) to
\(0.01\eta_0\) over at most \(20{,}000\) updates.  Validation MSE is measured
every \(500\) updates, including step zero.  After at least \(5{,}000\)
updates, training stops when every learning-rate candidate has gone 20
consecutive validation evaluations without a relative improvement of at
least \(10^{-5}\).  Among all learning rates and recorded checkpoints, we
select the one with the smallest validation MSE.  The test set is evaluated
only after this choice is fixed.

\begin{table}[!htbp]
    \centering
    \small
    \setlength{\tabcolsep}{5pt}
    \begin{tabular}{llccc}
        \toprule
        Experiment & Model & Seed 1000 & Seed 1001 & Seed 1002 \\
        \midrule
        \ref{sec:coordinate-specialization} & ReLU
        & $(10^{-4},4000)$ & $(10^{-4},4500)$ & $(3{\times}10^{-4},1000)$ \\
        & GELU
        & $(10^{-2},16500)$ & $(3{\times}10^{-4},8000)$ & $(10^{-2},12500)$ \\
        & ReGLU
        & $(10^{-4},4500)$ & $(10^{-4},0)$ & $(10^{-4},500)$ \\
        & SwiGLU
        & $(10^{-2},16500)$ & $(3{\times}10^{-4},7000)$ & $(10^{-2},16000)$ \\
        \midrule
        \ref{sec:local-routing} & ReLU
        & $(10^{-4},500)$ & $(10^{-2},8500)$ & $(10^{-4},500)$ \\
        & GELU
        & $(10^{-2},6000)$ & $(10^{-4},15500)$ & $(3{\times}10^{-4},6500)$ \\
        & ReGLU
        & $(10^{-2},6000)$ & $(3{\times}10^{-4},10000)$ & $(3{\times}10^{-3},9500)$ \\
        & SwiGLU
        & $(10^{-2},20000)$ & $(10^{-2},19000)$ & $(10^{-4},8500)$ \\
        \bottomrule
    \end{tabular}
    \caption{Validation-selected readout hyperparameters.  Each entry gives
    (initial learning rate, selected update step).  Step zero denotes the original retained
    readout before any update.}
    \label{tab:active-neuron-readout-selected-hyperparameters}
\end{table}

\FloatBarrier

\clearpage
\section{Proofs for Section~\ref{sec:specialization-theory}}
\label{app:specialization-theory-proofs}

\subsection{Proof of Theorem~\ref{thm:neurons-specialize}}
\label{app:proof-neurons-specialize}

We first prove Theorem~\ref{thm:neurons-specialize} invoking
\cref{lem:cubic-correlation-identity,lem:one-cluster-maxima,lem:specialized-maxima,lem:positive-signed-flow,lem:self-selected-specialization,lem:uniform-cluster-coverage,lem:early-time-tracking},
whose statements and proofs are provided after the main proof.
A roadmap of their dependencies is shown in
\cref{fig:theorem1-dependencies}.

{
Throughout this subsection, the unit means and predictive directions
\[
    \mu_1,\ldots,\mu_K,
    v_1,\ldots,v_K
\]
form an orthonormal basis of $\mathbb R^d$, where $d=2K$.
Write $\mathcal S_c=\operatorname{span}\{\mu_c,v_c\}$ and
$P_c=\mu_c\mu_c^\top+v_cv_c^\top$ as in the main text.
In the notation of
Section~\ref{sec:prelim}, we may write
\[
    c\sim\operatorname{Unif}([K]),
    \qquad
    x=\mu_c+z,
    \qquad
    z\sim\mathcal N(0,I_d),
\]
and
\[
    y=h_3\bigl(v_c^\top z\bigr),
    \qquad
    h_3(t)=\frac{t^3-3t}{\sqrt6}.
\]

For $\omega\in\mathbb S^{d-1}$, define the weight's correlations with the cluster means and predictive directions
\[
    b_c(\omega)
    =
    \langle\omega,\mu_c\rangle,
    \qquad
    \rho_c(\omega)
    =
    \langle\omega,v_c\rangle.
\]
}
Write the ReLU activation function and the Gaussian density as
\[
    \phi(t)=t_+,
    \qquad
    \varphi(t)
    =
    \frac{1}{\sqrt{2\pi}}e^{-t^2/2},
\]
and define the unsigned and signed population-correlation objectives, respectively, by
\[
    \Phi({\omega})
    =
    \mathbb E\left[{y}\,
      \phi({\omega}^\top{x})\right],
    \qquad
    \Psi_\zeta({\omega})
    =
    \zeta\Phi({\omega}),
    \qquad
    \zeta\in\{\pm1\}.
\]

\begin{proof}[Proof of Theorem~\ref{thm:neurons-specialize}]
For each initial direction ${\omega_j^0}$, let
$(\bar u_j,{\bar\omega_j})$ be the corresponding self-selected teacher-only
trajectory from Lemma~\ref{lem:self-selected-specialization}.

By Lemmas~\ref{lem:self-selected-specialization}
and~\ref{lem:specialized-maxima}, almost surely there are a cluster label
$J_j\in[K]$ and a unit vector $\omega_j^\star\in\mathcal S_{J_j}$ such that
\[
    \bar\omega_j(t)\longrightarrow\omega_j^\star
    \qquad\text{as }t\to\infty.
\]
Each selected vector has nonzero components along both $\mu_{J_j}$ and
$v_{J_j}$. Since $m$ is finite, these assertions hold simultaneously
outside a single null set of initializations.

Since $m<\infty$ and every teacher-only direction converges, for every
$\delta>0$ there is a finite time $T_\delta$ such that, simultaneously for
all $j\in[m]$,
\[
    \left\|
        {\bar\omega_j(T_\delta)-\omega_j^\star}
    \right\|_2
    \leq
    \frac{\delta}{2}.
\]

Apply Lemma~\ref{lem:early-time-tracking} on $[0,T_\delta]$. For sufficiently
small $\varepsilon$,
\[
    \left\|
        \frac{w_j(T_\delta)}{\|w_j(T_\delta)\|_2}
        -
        {\bar\omega_j(T_\delta)}
    \right\|_2
    \leq
    L_{T_\delta}m\varepsilon^2
    \leq
    \frac{\delta}{2}
\]
for every $j\in[m]$. The triangle inequality gives
\[
    \left\|
        \frac{w_j(T_\delta)}{\|w_j(T_\delta)\|_2}
        -
        \omega_j^\star
    \right\|_2
    \leq
    \delta.
\]
Since $(I_d-P_{J_j})\omega_j^\star=0$ and an orthogonal projector
has operator norm at most one, this implies the stated distance-to-span
bound. Pythagoras gives the equivalent alignment bound. The teacher-only
trajectories and their limits do not depend on $\varepsilon$, so the
chosen labels and $T_\delta$ have the required independence from
$\varepsilon$.

The independence and uniformity of the labels $J_1,\ldots,J_m$, together
with the cluster-coverage probability, follow from
Lemma~\ref{lem:uniform-cluster-coverage}.
\end{proof}

{
\begin{figure}[ht!]
\centering
\resizebox{\linewidth}{!}{%
\begin{tikzpicture}[
    node distance=7mm and 12mm,
    >=Latex,
    every node/.style={
        draw,
        rounded corners,
        align=center,
        font=\small,
        inner sep=5pt,
        text width=3.4cm
    },
    theorem/.style={
        draw,
        very thick,
        text width=4.0cm
    },
    arrow/.style={
        ->,
        thick
    }
]

\node (corr)
{
Lemma~\ref{lem:cubic-correlation-identity}\\
\textbf{Cubic correlation identity}\\
$\displaystyle\Psi_\zeta(\omega)
 = -\frac{\zeta}{K\sqrt6}\sum_c
 b_c\varphi(b_c)\rho_c^3$
};

\node (onecluster) [below left=of corr]
{
Lemma~\ref{lem:one-cluster-maxima}\\
\textbf{One-cluster landscape}\\
positive local maxima cannot
mix predictive directions
};

\node (specialmax) [below right=of corr]
{
Lemma~\ref{lem:specialized-maxima}\\
\textbf{Specialized maxima}\\
$\omega^\star\in\mathcal S_c$\\
nonzero mean and predictive components
};

\node (flow) [below=14mm of $(onecluster)!0.5!(specialmax)$]
{
Lemma~\ref{lem:positive-signed-flow}\\
\textbf{Gradient-flow convergence}\\
unstable mixed critical points
are avoided a.s.
};

\node (selfsign) [below=of flow]
{
Lemma~\ref{lem:self-selected-specialization}\\
\textbf{Self-selected sign}\\
teacher-only dynamics select a sign
and specialize a.s.
};

\node (coverage) [below left=of selfsign]
{
Lemma~\ref{lem:uniform-cluster-coverage}\\
\textbf{Uniform labels and coverage}\\
$J_j \stackrel{\mathrm{iid}}{\sim}
\operatorname{Unif}([K])$
};

\node (tracking) [below right=of selfsign]
{
Lemma~\ref{lem:early-time-tracking}\\
\textbf{Full-network tracking}\\
coupled dynamics differ from
teacher-only dynamics by
$O(m\varepsilon^2)$
};

\node[theorem] (thm) [below=15mm of $(coverage)!0.5!(tracking)$]
{
Theorem~\ref{thm:neurons-specialize}\\
\textbf{Randomly initialized neurons specialize}\\[2pt]
$\|(I_d-P_{J_j})\omega_j(T_\delta)\|_2\leq\delta$
};

\draw[arrow] (corr) -- (onecluster);
\draw[arrow] (corr) -- (specialmax);
\draw[arrow] (onecluster) -- (specialmax);
\draw[arrow] (onecluster) -- (flow);
\draw[arrow] (specialmax) -- (flow);
\draw[arrow] (flow) -- (selfsign);
\draw[arrow] (selfsign) -- (coverage);
\draw[arrow] (selfsign) -- (tracking);
\draw[arrow] (coverage) -- (thm);
\draw[arrow] (tracking) -- (thm);

\end{tikzpicture}%
}
\caption{
Dependency graph for the proof of
Theorem~\ref{thm:neurons-specialize}.
Arrows indicate the main logical dependencies. Here
$\omega_j(t)=w_j(t)/\|w_j(t)\|_2$ and
$\mathcal S_c=\operatorname{span}\{\mu_c,v_c\}$.
}
\label{fig:theorem1-dependencies}
\end{figure}
}
\clearpage

\subsubsection{The Population-Correlation Landscape}

In this subsection, we study the critical points of the optimization problem
\begin{align}\label{eq:directional-alignment-optimization-problem}
\max_{\omega \in \mathbb{S}^{d-1}} \Psi_{\zeta}(\omega)\,.
\end{align}
The use of homogeneity to characterize parameter directions
is part of a broader literature on the implicit bias of gradient methods,
beginning with homogeneous linear predictors and extending to homogeneous
neural networks
\citep{soudry2018implicit,ji2019implicit,lyu2020gradient,ji2020directional}.
These results motivate the directional viewpoint used here, but do not
directly yield the squared-loss, small-initialization approximation in our
setting; we establish the required approximation directly in
Lemma~\ref{lem:early-time-tracking}. The subsequent subsection shows how
the sign $\zeta$ is selected by each neuron. We prove that every positive
critical point of \eqref{eq:directional-alignment-optimization-problem}
either lies in one cluster's mean--predictive span or has an unstable
tangent direction.

First, we provide an explicit formula for the signed population-correlation objective in terms of a weight's correlations with the cluster means and predictive directions. This lemma is stated for directions $\omega$ on the unit sphere, but is later applied to the directions of small weights using the homogeneity of ReLU networks.

\begin{lemma}[Cubic correlation identity]
\label{lem:cubic-correlation-identity}
For every $\omega\in\mathbb S^{d-1}$ and $\zeta\in\{\pm1\}$,
\[
    \Psi_\zeta(\omega)
    =
    -\frac{\zeta}{K\sqrt6}
    \sum_{c=1}^K
    b_c(\omega)\varphi\bigl(b_c(\omega)\bigr)\rho_c(\omega)^3.
\]
\end{lemma}

\begin{proof}
Fix $\omega\in\mathbb S^{d-1}$. Conditional on $c$,
\[
    \omega^\top x
    =b_c(\omega)
      +\omega^\top z,
    \qquad
    y=h_3\bigl(v_c^\top z\bigr).
\]
Hence the contribution of cluster $c$ to $\Phi(\omega) = \mathbb E\left[y\,
      \phi(\omega^\top x)\right]$ is
\[\mathbb E\left[
        h_3\bigl(v_c^\top z\bigr)
        \bigl(b_c(\omega)
          +\omega^\top z\bigr)_+
    \right].
\]

Set
\[
    U=v_c^\top z,
    \qquad
    V=\omega^\top z,
    \qquad
    \rho=\rho_c(\omega)
    =\langle v_c,\omega\rangle.
\]
Since $z\sim\mathcal N(0,I_d)$ and
$v_c,\omega$ are unit vectors,
$U$ and $V$ are standard Gaussian, and
\[
    \mathbb E[UV]
    =
    v_c^\top
      \mathbb E[zz^\top]\omega
    =
    \langle v_c,\omega\rangle
    =
    \rho.
\]
Thus $(U,V)$ is a jointly standard Gaussian pair with correlation $\rho$.

For $|\rho|<1$, define
\[
    Z:=\frac{U-\rho V}{\sqrt{1-\rho^2}}.
\]
Because $(U,V)$ is jointly Gaussian, $(Z,V)$ is also jointly Gaussian.
Moreover,
\[
    \mathbb E[Z]=0,
    \qquad
    \mathbb E[Z^2]=1,
\]
and
\[
    \mathbb E[ZV]
    =
    \frac{\mathbb E[UV]-\rho\mathbb E[V^2]}
         {\sqrt{1-\rho^2}}
    =0.
\]
Jointly Gaussian random variables with zero covariance are independent, so
$Z\sim\mathcal N(0,1)$ is independent of $V$. Therefore
\[
    U=\rho V+\sqrt{1-\rho^2}\,Z.
\]
When $|\rho|=1$, the same representation holds with the second term equal
to zero.

Conditioning on $V$ and using the independence of $Z$ gives
\[
    \mathbb E[U\mid V]=\rho V
\]
and
\begin{align*}
    \mathbb E[U^3\mid V]
    &=
    \mathbb E\left[
        \left(\rho V+\sqrt{1-\rho^2}\,Z\right)^3
        \,\middle|\,V
    \right]\\
    &=
    \rho^3V^3+3\rho(1-\rho^2)V,
\end{align*}
since $\mathbb E[Z]=\mathbb E[Z^3]=0$ and $\mathbb E[Z^2]=1$.
Recalling that
\[
    h_3(x)=\frac{x^3-3x}{\sqrt6},
\]
we obtain
\begin{align*}
    \mathbb E[h_3(U)\mid V]
    &=
    \frac{1}{\sqrt6}
    \left(
        \mathbb E[U^3\mid V]
        -3\mathbb E[U\mid V]
    \right)\\
    &=
    \rho^3 h_3(V).
\end{align*}

Since $(b_c(\omega)+V)_+$ depends only on $V$, the tower property now gives
\begin{align*}
    &\mathbb E\left[
        h_3(U)\bigl(b_c(\omega)+V\bigr)_+
    \right]\\
    &\qquad=
    \mathbb E\left[
        \bigl(b_c(\omega)+V\bigr)_+
        \mathbb E[h_3(U)\mid V]
    \right]\\
    &\qquad=
    \rho_c(\omega)^3
    \mathbb E_{G\sim\mathcal N(0,1)}
    \left[
        h_3(G)\bigl(b_c(\omega)+G\bigr)_+
    \right].
\end{align*}

It remains to compute the one-dimensional expectation. For any $b\in\mathbb R$,
\[
    \mathbb E\left[h_3(G)(b+G)_+\right]
    =
    \frac1{\sqrt6}
    \int_{-b}^{\infty}
        (x^3-3x)(b+x)\varphi(x)\,dx.
\]
Using
\[
    \frac{d}{dx}
    \left[(x^2-1)\varphi(x)\right]
    =
    -(x^3-3x)\varphi(x),
\]
integration by parts yields
\[
    \int_{-b}^{\infty}
        (x^3-3x)(b+x)\varphi(x)\,dx
    =
    \int_{-b}^{\infty}
        (x^2-1)\varphi(x)\,dx,
\]
where the boundary term vanishes because $b+x=0$ at $x=-b$ and
$\varphi(x)\to0$ as $x\to\infty$. Since
\[
    \frac{d}{dx}\bigl(x\varphi(x)\bigr)
    =
    (1-x^2)\varphi(x),
\]
we further obtain
\[
    \int_{-b}^{\infty}
        (x^2-1)\varphi(x)\,dx
    =
    -b\varphi(b).
\]
Therefore
\[
    \mathbb E\left[h_3(G)(b+G)_+\right]
    =
    -\frac{b\varphi(b)}{\sqrt6}.
\]

Applying this with $b=b_c(\omega)$, the contribution of cluster $c$ to
$\Phi(\omega)$ is
\[
    -\frac{1}{\sqrt6}
    b_c(\omega)\varphi\bigl(b_c(\omega)\bigr)\rho_c(\omega)^3.
\]
Since $c$ is uniform on $[K]$,
\[
    \Phi(\omega)
    =
    -\frac1{K\sqrt6}
    \sum_{c=1}^K
    b_c(\omega)\varphi\bigl(b_c(\omega)\bigr)\rho_c(\omega)^3.
\]
Finally, $\Psi_\zeta(\omega)=\zeta\Phi(\omega)$, and hence
\[
    \Psi_\zeta(\omega)
    =
    -\frac{\zeta}{K\sqrt6}
    \sum_{c=1}^K
    b_c(\omega)\varphi\bigl(b_c(\omega)\bigr)\rho_c(\omega)^3,
\]
as claimed.
\end{proof}

Because the $\mu_c$ and $v_c$ vectors
form an orthonormal basis, the
sphere constraint is
\[
    \sum_{c=1}^Kb_c^2
    +
    \sum_{c=1}^K\rho_c^2
    =
    1.
\]
Next, we consider the optimization problem of signed population-correlation maximization between the value of the neuron and the ground truth:
$\max_{\omega\in\mathbb S^{d-1}}\Psi_\zeta(\omega)$. This is written and analyzed below under the linear change of
coordinates
\[
    b_c=\langle\omega,\mu_c\rangle,
    \qquad
    \rho_c=\langle\omega,v_c\rangle\,,
\]
where we show that the neuron maximizes correlation with the response by specializing to one cluster. Here, we use the cubic Hermite link function.

\begin{lemma}[Positive local maxima use one cluster]
\label{lem:one-cluster-maxima}
Consider the constrained maximization problem
\[
    \max_{b_1,\ldots,b_K,\rho_1,\ldots,\rho_K}
    F(b_1,\ldots,b_K,\rho_1,\ldots,\rho_K),
\]
where
\[
    F(b_1,\ldots,b_K,\rho_1,\ldots,\rho_K)
    :=
    -\frac{\zeta}{K\sqrt6}
    \sum_{c=1}^K
    b_c\varphi(b_c)\rho_c^3,
\]
subject to
\[
    \sum_{c=1}^Kb_c^2
    +
    \sum_{c=1}^K\rho_c^2
    =
    1.
\]

Every constrained local maximum of $F$ with positive objective value has
exactly one nonzero predictive coordinate $\rho_c$. Moreover, if
$\rho_c=0$, then the corresponding routing coordinate $b_c$ also vanishes.
Every positive critical point with at least two nonzero predictive
coordinates has a tangent direction of strictly positive Hessian
curvature and is therefore unstable under gradient ascent.
\end{lemma}

\begin{proof}
At any positive constrained critical point, first-order stationarity of the Lagrangian
\[
    F(b_1,\ldots,b_K,\rho_1,\ldots,\rho_K)
    -
    \lambda\left(
        \sum_{c=1}^Kb_c^2
        +
        \sum_{c=1}^K\rho_c^2
        -1
    \right)
\]
with respect to $b_c$ and $\rho_c$ gives, for every $c\in[K]$,
\begin{align}
    -\frac{\zeta}{K\sqrt6}
    (1-b_c^2)\varphi(b_c)\rho_c^3
    &=
    2\lambda b_c,
    \label{eq:kkt-b}\tag{B.1}\\
    -\frac{3\zeta}{K\sqrt6}
    b_c\varphi(b_c)\rho_c^2
    &=
    2\lambda\rho_c.
    \label{eq:kkt-rho}\tag{B.2}
\end{align}
Here we used
\[
    \frac{d}{db}\bigl(b\varphi(b)\bigr)
    =
    (1-b^2)\varphi(b).
\]

Multiplying~\eqref{eq:kkt-rho} by $\rho_c$ and summing over $c$ gives
\[
    3F(b_1,\ldots,b_K,\rho_1,\ldots,\rho_K)
    =
    2\lambda\sum_{c=1}^K\rho_c^2.
\]
Since the objective value is positive, at least one predictive coordinate
is nonzero. Hence both the left-hand side and
$\sum_c\rho_c^2$ are positive, so
\[
    \lambda>0.
\]

If $\rho_c=0$, then~\eqref{eq:kkt-b} reduces to
\[
    0=2\lambda b_c.
\]
Since $\lambda>0$, this implies $b_c=0$.

It remains to rule out two nonzero predictive coordinates. Suppose that
$\rho_c\neq0$ and $\rho_r\neq0$ for two distinct clusters $c\neq r$.
Keep all routing coordinates and all other predictive coordinates fixed,
and consider
\[
    \rho_c(t)
    =
    \operatorname{sign}(\rho_c)
    \sqrt{\rho_c^2+t},
    \qquad
    \rho_r(t)
    =
    \operatorname{sign}(\rho_r)
    \sqrt{\rho_r^2-t}.
\]
For sufficiently small $|t|$, this perturbation is well defined, preserves
the signs of the two coordinates, and satisfies
\[
    \rho_c(t)^2+\rho_r(t)^2
    =
    \rho_c^2+\rho_r^2.
\]
Thus it preserves the constraint exactly.

Differentiating $F$ along this feasible curve at $t=0$ gives
\[
    F'(0)
    =
    -\frac{3\zeta}{2K\sqrt6}
    \left(
        b_c\varphi(b_c)\rho_c
        -
        b_r\varphi(b_r)\rho_r
    \right).
\]
Since $\rho_c,\rho_r\neq0$, dividing~\eqref{eq:kkt-rho} by the corresponding
predictive coordinate gives
\[
    -\frac{3\zeta}{K\sqrt6}
    b_c\varphi(b_c)\rho_c
    =
    2\lambda,
    \qquad
    -\frac{3\zeta}{K\sqrt6}
    b_r\varphi(b_r)\rho_r
    =
    2\lambda.
\]
Therefore
\[
    F'(0)=0.
\]

The second derivative along the same feasible curve is
\[
    F''(0)
    =
    -\frac{3\zeta}{4K\sqrt6}
    \left(
        \frac{b_c\varphi(b_c)}{\rho_c}
        +
        \frac{b_r\varphi(b_r)}{\rho_r}
    \right).
\]
Using the same KKT identities gives
\[
    F''(0)
    =
    \frac{\lambda}{2}
    \left(
        \frac{1}{\rho_c^2}
        +
        \frac{1}{\rho_r^2}
    \right)
    >0.
\]
Thus a critical point with two nonzero predictive coordinates has a feasible
direction of positive curvature and cannot be a constrained local maximum.

Hence at a positive local maximum at most one predictive coordinate is
nonzero. Since the objective
value is positive, at least one is nonzero, so exactly one is active. As
shown above, every routing coordinate corresponding to an inactive predictive
coordinate also vanishes.
\end{proof}

Returning to our objective $\max_{\omega \in \mathbb{S}^{d-1}} \Psi_{\zeta}(\omega)$, let $c$ denote the unique active cluster and let
\[
    \tau=\operatorname{sign}(\rho_c)\in\{\pm1\}.
\]
Then
\[
    b_j=\rho_j=0
    \qquad
    \text{for every }j\neq c,
\]
and the constraint becomes
\[
    b_c^2+\rho_c^2=1.
\]
Hence
\[
    \rho_c
    =
    \tau\sqrt{1-b_c^2}.
\]
Substituting this into $F$, and writing $b=b_c$, reduces the problem to the
one-dimensional objective
\[
    -\frac{\tau\zeta}{K\sqrt6}
    b\varphi(b)
    \left(1-b^2\right)^{3/2},
    \qquad
    |b|<1.
\]

\begin{lemma}[Specialized maxima]
\label{lem:specialized-maxima}
Define the constants
\[
    \beta_\star:=\sqrt{\frac{5-\sqrt{21}}{2}},
    \qquad
    \gamma_\star:=\sqrt{1-\beta_\star^2}.
\]
For each fixed $\zeta\in\{\pm1\}$, the positive local maxima of
$\Psi_\zeta$ are exactly the $2K$ directions
\begin{equation}
\label{eq:exact-specialized-direction}
    \omega^\star_{c,\tau,\zeta}
    =\tau\bigl(\gamma_\star v_c-\zeta\beta_\star\mu_c\bigr),
    \qquad c\in[K],\quad \tau\in\{\pm1\}.
\end{equation}
They are global maxima, and each is a strict local maximum. In particular,
each lies in $\mathcal S_c$ with nonzero mean and predictive components.
\end{lemma}

\begin{proof}
By Lemma~\ref{lem:one-cluster-maxima}, every positive local maximum has
one active cluster $c$. Writing $\tau=\operatorname{sign}(\rho_c)$,
its direction has the form
\[
    \omega=b\mu_c+\tau\sqrt{1-b^2}\,v_c,
    \qquad |b|<1.
\]
Its objective is
\[
    -\frac{\tau\zeta}{K\sqrt6}\,
    b\varphi(b)(1-b^2)^{3/2}.
\]
Positivity is equivalent to $t:=-\tau\zeta b\in(0,1)$. Since
$\varphi$ is even, maximizing the positive part of any one-cluster
branch is equivalent to maximizing
\[
    H(t):=t\varphi(t)(1-t^2)^{3/2},\qquad 0<t<1.
\]
The logarithmic derivative and its derivative are
\begin{align*}
    (\log H)'(t)&=\frac1t-t-\frac{3t}{1-t^2},\\
    (\log H)''(t)&=-\frac1{t^2}-1
                   -\frac{3(1+t^2)}{(1-t^2)^2}<0.
\end{align*}
The first derivative tends to $+\infty$ as $t\downarrow0$ and to
$-\infty$ as $t\uparrow1$. Hence $H$ has exactly one critical point,
which is its strict maximum. Its equation is
\[
    t^4-5t^2+1=0,
\]
whose unique solution in $(0,1)$ is $t=\beta_\star$. Thus
$b=-\tau\zeta\beta_\star$ and $\rho_c=\tau\gamma_\star$, giving
Eq.~\ref{eq:exact-specialized-direction}.

The sphere is compact, so $\Psi_\zeta$ attains a global maximum.
Its value is positive because each candidate above has positive value.
Every global maximizer must therefore be a positive local maximum and
hence one of these candidates. All candidates have the same value
$H(\beta_\star)/(K\sqrt6)$, so all are global maximizers. They are
finitely many distinct points; each is isolated among the global
maximizers and thus a strict local maximum. Finally,
$0<\beta_\star,\gamma_\star<1$, proving the claim about both components.
\end{proof}

\subsubsection{Random Initialization, Self-Selected Signs, and Cluster Coverage}

In the previous section, we considered the signed population-correlation objective $\Psi_{\zeta}$ for each neuron. In this section, we note that the sign $\zeta$ for each neuron is learned at very early times of training and does not need to be fixed. Consider
the teacher-only joint dynamics
\begin{equation}
    \dot u=\Phi({\omega}),
    \qquad
    {\dot\omega}
    =\tanh(u)\nabla_{\mathbb S^{d-1}}\Phi({\omega}).
\end{equation}
These equations arise from the two-layer parameterization when the
second-layer weight is initially zero. Crucially, they neglect interactions between neurons, which we will later show is a fine approximation because the network is small at initialization.

\begin{lemma}[Positive signed-correlation flows specialize]
\label{lem:positive-signed-flow}
Let ${\omega^0}$ have an absolutely continuous distribution on
$\mathbb S^{d-1}$, fix $\zeta\in\{\pm1\}$, and assume
$\Psi_\zeta({\omega^0})>0$ almost surely. Then spherical gradient
ascent
\[
    {\dot\omega}
    =\nabla_{\mathbb S^{d-1}}\Psi_\zeta({\omega}),
    \qquad
    {\omega(0)}={\omega^0},
\]
converges almost surely to one of the specialized maxima in
Lemma~\ref{lem:specialized-maxima}.
\end{lemma}

\begin{proof}
Along spherical gradient ascent,
\[
    \frac{d}{dt}\Psi_\zeta({\omega(t)})
    =
    \left\|
        \nabla_{\mathbb S^{d-1}}\Psi_\zeta({\omega(t)})
    \right\|_2^2
    \geq 0.
\]
Since the sphere is compact, $\Psi_\zeta$ is bounded above. Hence
\[
    \int_0^\infty
    \left\|
        \nabla_{\mathbb S^{d-1}}\Psi_\zeta({\omega(t)})
    \right\|_2^2\,dt
    <\infty.
\]
Moreover, $\Psi_\zeta$ is smooth on the sphere, so its gradient and Hessian
are bounded. It follows that
$\|\nabla_{\mathbb S^{d-1}}\Psi_\zeta({\omega(t)})\|_2^2$ has bounded derivative.
A nonnegative integrable function with bounded derivative must converge to
zero, and therefore
\[
    \left\|
        \nabla_{\mathbb S^{d-1}}\Psi_\zeta({\omega(t)})
    \right\|_2
    \longrightarrow 0.
\]

We next note that $\Psi_\zeta$ has only finitely many positive critical
points. Indeed, at any positive critical point, the stationarity conditions
\eqref{eq:kkt-b}--\eqref{eq:kkt-rho} imply that $\lambda>0$, and hence
$b_c=0$ whenever $\rho_c=0$. For every active coordinate $\rho_c\neq0$,
eliminating $\lambda$ from the two stationarity equations gives
\[
    (1-b_c^2)\rho_c^2
    =
    3b_c^2.
\]
Thus $0<|b_c|<1$ and
\[
    \rho_c^2
    =
    \frac{3b_c^2}{1-b_c^2}.
\]
Equation~\ref{eq:kkt-rho} further shows that
\[
    |b_c|\varphi(b_c)|\rho_c|
\]
has the same value for every active coordinate. Substituting the preceding
expression for $\rho_c^2$, this quantity is proportional to
\[
    \frac{|b_c|^2\varphi(|b_c|)}
         {\sqrt{1-|b_c|^2}},
\]
which is strictly increasing for $0<|b_c|<1$, since
\[
    \frac{d}{dt}
    \log\left(
        \frac{t^2\varphi(t)}{\sqrt{1-t^2}}
    \right)
    =
    \frac{2}{t}+\frac{t^3}{1-t^2}
    >0.
\]
Hence all active $|b_c|$ are equal. If there are $s$ active coordinates,
the sphere constraint then gives
\[
    s\,\frac{b_c^2(4-b_c^2)}
             {1-b_c^2}
    =1,
\]
which has a unique solution for $b_c^2\in(0,1)$: the function
$u\mapsto u(4-u)/(1-u)$ is strictly increasing from zero to infinity
on $(0,1)$, with derivative $(u^2-2u+4)/(1-u)^2>0$.
Thus, for each choice of
the active coordinates and their signs, there is at most one positive
critical point. Since there are only finitely many such choices, the set of
positive critical points is finite.

Because $\Psi_\zeta({\omega(0)})>0$ and the objective is nondecreasing, every
accumulation point of the trajectory has positive objective value. By
compactness, accumulation points exist, and the convergence of the gradient
to zero implies that every accumulation point is a positive critical point.
The set of accumulation points is connected: it is the intersection of
the nested compact connected closures of the trajectory tails.
As a nonempty connected subset of a finite set, it is a singleton.
The trajectory therefore converges to one positive critical point.

Finally, Lemma~\ref{lem:one-cluster-maxima} gives a positive eigenvalue of
the spherical Hessian at every positive critical point with more than
one active predictive coordinate. To apply saddle avoidance precisely,
let $\mathcal T$ be the time-one map of this spherical gradient flow.
The smooth vector field on the compact sphere is complete, so
$\mathcal T$ is a smooth diffeomorphism. At a critical point $p$,
\[
    D\mathcal T(p)
    =\exp\!\left(\operatorname{Hess}_{\mathbb S^{d-1}}
                   \Psi_\zeta(p)\right).
\]
Thus every mixed positive critical point has an eigenvalue of
$D\mathcal T(p)$ greater than one. By the local center-stable manifold
theorem (see \citealp[Theorem~1]{lee2017firstorder}), the points whose
forward iterates stay in a sufficiently small neighborhood of $p$ lie
in a submanifold of codimension at least one. Every trajectory converging
to $p$ eventually stays in that neighborhood. Its initial point therefore
lies in a countable union of inverse images of this null set under
$\mathcal T^N$, $N\geq0$. Diffeomorphisms preserve null sets locally,
so this union has spherical measure zero. There are finitely many mixed
positive critical points, and their combined basin also has measure zero.
Lemma~\ref{lem:specialized-maxima} shows that the remaining positive critical
points are precisely the specialized maxima. Hence an absolutely continuous
initialization converges almost surely to a specialized maximum.
\end{proof}

For a single neuron, write
\[
    \omega(t)=\frac{w(t)}{\|w(t)\|_2}.
\]
ReLU homogeneity implies that gradient flow preserves
\[
    \|w(t)\|_2^2-a(t)^2
    =
    \|w(0)\|_2^2-a(0)^2
    =
    \varepsilon^2.
\]
We therefore define the scalar amplitude coordinate
\[
    u(t)
    :=
    \operatorname{arsinh}\!\left(\frac{a(t)}{\varepsilon}\right),
\]
so that, equivalently,
\[
    a(t)=\varepsilon\sinh u(t),
    \qquad
    \|w(t)\|_2=\varepsilon\cosh u(t).
\]
Under this parametrization, the target-only part of the directional dynamics
is
\begin{equation}
\label{eq:self-selected-teacher-flow}
    \dot u=\Phi(\omega),
    \qquad
    \dot\omega
    =
    \tanh(u)\nabla_{\mathbb S^{d-1}}\Phi(\omega).
\end{equation}

\begin{lemma}[Self-selected sign and almost-sure specialization]
\label{lem:self-selected-specialization}
Let $\omega^0$ have an absolutely continuous distribution on
$\mathbb S^{d-1}$, and let $(u(t),\omega(t))$ solve
\eqref{eq:self-selected-teacher-flow} with
\[
    u(0)=0,
    \qquad
    \omega(0)=\omega^0.
\]
For almost every $\omega^0$, define
\[
    \zeta
    =
    \operatorname{sign}(\Phi(\omega^0))
    \in\{\pm1\}.
\]
Then, for every $t>0$,
\[
    \zeta u(t)>0,
    \qquad
    \zeta\Phi(\omega(t))
    \geq
    |\Phi(\omega^0)|.
\]
Moreover, $\omega(t)$ is a positive time reparameterization of spherical
gradient ascent on $\Psi_\zeta=\zeta\Phi$, and it converges almost surely to one of the specialized maxima in
Lemma~\ref{lem:specialized-maxima}.
\end{lemma}

\begin{proof}
The cubic correlation identity makes $\Phi$ real analytic on the sphere.
It is nonzero at $(\mu_1+v_1)/\sqrt2$, so it is not identically zero.
The sphere is connected for $d=2K\geq2$; hence the zero set of $\Phi$
has spherical measure zero. Thus, for almost every ${\omega^0}$,
$\Phi({\omega^0})\neq0$.
Fix such a ${\omega^0}$ and let
\[
    \zeta=\operatorname{sign}(\Phi({\omega^0})).
\]

At $t=0$,
\[
    \frac{d}{dt}\bigl(\zeta u(t)\bigr)\bigg|_{t=0}
    =
    \zeta\Phi({\omega^0})
    =
    |\Phi({\omega^0})|
    >0,
\]
so $\zeta u(t)>0$ for all sufficiently small $t>0$.

As long as $\zeta u(t)>0$, $\tanh(u(t))$ has sign $\zeta$, and hence
\[
    \zeta\tanh(u(t))=|\tanh(u(t))|.
\]
Therefore
\[
    \frac{d}{dt}\bigl(\zeta\Phi({\omega(t)})\bigr)
    =
    |\tanh(u(t))|
    \left\|
        \nabla_{\mathbb S^{d-1}}\Phi({\omega(t)})
    \right\|_2^2
    \geq0.
\]
It follows that
\[
    \zeta\Phi({\omega(t)})
    \geq
    \zeta\Phi({\omega^0})
    =
    |\Phi({\omega^0})|.
\]
Consequently,
\[
    \frac{d}{dt}\bigl(\zeta u(t)\bigr)
    =
    \zeta\Phi({\omega(t)})
    \geq
    |\Phi({\omega^0})|,
\]
and therefore
\[
    \zeta u(t)\geq |\Phi({\omega^0})|t.
\]
In particular, $\zeta u(t)$ cannot return to zero, so the preceding
inequalities hold for every $t>0$. This proves
\[
    \zeta u(t)>0,
    \qquad
    \zeta\Phi({\omega(t)})
    \geq
    |\Phi({\omega^0})|.
\]

Since $\Psi_\zeta=\zeta\Phi$ and $\tanh(u(t))$ has sign $\zeta$, the
directional dynamics satisfy
\[
    {\dot\omega(t)}
    =
    |\tanh(u(t))|
    \nabla_{\mathbb S^{d-1}}\Psi_\zeta({\omega(t)}).
\]
Thus ${\omega(t)}$ follows the same orbit as spherical gradient ascent on
$\Psi_\zeta$, up to a positive reparameterization of time. Moreover,
$\zeta u(t)\geq|\Phi({\omega^0})|t$ implies $|u(t)|\to\infty$, and hence
$|\tanh(u(t))|\to1$. In particular,
\[
    \int_0^\infty |\tanh(u(t))|\,dt=\infty,
\]
so this reparameterization covers the entire forward gradient-ascent
trajectory.

Finally,
\[
    \Psi_\zeta({\omega^0})
    =
    \zeta\Phi({\omega^0})
    =
    |\Phi({\omega^0})|
    >0.
\]
Conditioning the initialization on either event
$\{\operatorname{sign}\Phi(\omega^0)=\zeta\}$ of positive probability
preserves absolute continuity. Apply Lemma~\ref{lem:positive-signed-flow}
separately for each sign and take the union of the two null exceptional
sets. Thus $\omega(t)$ converges almost surely to one of the specialized
maxima in Lemma~\ref{lem:specialized-maxima}.
\end{proof}

\begin{lemma}[Uniform selected labels and coverage]
\label{lem:uniform-cluster-coverage}
Let ${\omega_1^0,\ldots,\omega_m^0}$ be independent and uniform on
$\mathbb S^{d-1}$, and let $J_j$ be the cluster selected by the self-selected
flow in Lemma~\ref{lem:self-selected-specialization}. Then
$J_1,\ldots,J_m$ are independent and uniform on $[K]$. Consequently,
\[
    \mathbb P\bigl(\{J_1,\ldots,J_m\}\neq[K]\bigr)
    \leq
    K\left(1-\frac1K\right)^m
    \leq
    Ke^{-m/K}.
\]
\end{lemma}

\begin{proof}
For a permutation $\pi$ of $[K]$, let $P_\pi$ be the orthogonal map satisfying
\[
    P_\pi{\mu_c}
    =
    {\mu_{\pi(c)}},
    \qquad
    P_\pi{v_c}
    =
    {v_{\pi(c)}}.
\]
The data distribution, target, and self-selected vector field are equivariant
under $P_\pi$. Hence, if ${\omega^0}$ selects cluster $c$, then $P_\pi {\omega^0}$ selects
cluster $\pi(c)$. Uniform spherical measure is invariant under $P_\pi$, so
all selection probabilities are equal. By
Lemma~\ref{lem:self-selected-specialization}, they sum to one, and each is
therefore $1/K$. Independence follows because each teacher-only label is a
deterministic function of an independent initial direction.

A fixed cluster is missed with probability $(1-1/K)^m$. A union bound yields
the coverage estimate.
\end{proof}

\subsubsection{Tracking by the Full Small-Initialization Dynamics}

Finally, we put the above ingredients analyzing the trajectories of individual neurons together, and show that they describe the trajectory of the neurons in a neural network (up to rescaling) when the network is initialized small.
The proof strategy of studying feature learning through
effectively independent neuron dynamics in early-time or
small-initialization regimes has been put forward in prior work
\citep{abbe2022merged,min2024early,glasgow2024sgd}.
The population loss is
\[
    \mathcal L(\theta)
    =
    \frac12
    \mathbb E\left[
        \bigl(f_\theta({x})-{y}\bigr)^2
    \right],
\]
and its gradient-flow equations are
\begin{align*}
    \dot a_j
    &=
    \mathbb E\left[
        ({y}-f_\theta({x}))
        \phi(w_j^\top{x})
    \right],\\
    \dot w_j
    &=
    a_j
    \mathbb E\left[
        ({y}-f_\theta({x}))
        \mathbf 1_{\{w_j^\top{x}>0\}}{x}
    \right].
\end{align*}

\begin{lemma}[Early-time tracking from a zero output layer]
\label{lem:early-time-tracking}
Initialize
\[
    a_j(0)=0,
    \qquad
    w_j(0)=\varepsilon {\omega_j^0},
    \qquad
    {\omega_j^0}\in\mathbb S^{d-1}.
\]
Let $(\bar u_j,{\bar\omega_j})$ solve the self-selected teacher-only system
\[
    \dot{\bar u}_j
    =
    \Phi({\bar\omega_j}),
    \qquad
    {\dot{\bar\omega}_j}
    =
    \tanh(\bar u_j)
    \nabla_{\mathbb S^{d-1}}\Phi({\bar\omega_j}),
\]
with
\[
    (\bar u_j(0),{\bar\omega_j(0)})
    =
    (0,{\omega_j^0}).
\]
For every fixed $T<\infty$, there are constants $L_T<\infty$ and
$\varepsilon_T>0$ such that, for $0<\varepsilon\leq\varepsilon_T$,
\[
    \sup_{0\leq t\leq T}
    \left(
        |u_j(t)-\bar u_j(t)|
        +
        \|{\omega_j(t)-\bar\omega_j(t)}\|_2
    \right)
    \leq
    L_Tm\varepsilon^2,
\]
where
\[
    a_j(t)=\varepsilon\sinh(u_j(t)),
    \qquad
    {\omega_j(t)}=\frac{w_j(t)}{\|w_j(t)\|_2}.
\]
\end{lemma}

\begin{proof}
ReLU homogeneity gives the invariant
\[
    \frac{d}{dt}
    \left(
        \|w_j\|_2^2-a_j^2
    \right)
    =
    0.
\]
Under the stated initialization,
\[
    \|w_j(t)\|_2^2-a_j(t)^2
    =
    \varepsilon^2.
\]
Since $a_j=\varepsilon\sinh u_j$, it follows that
\[
    \|w_j\|_2
    =
    \varepsilon\cosh u_j.
\]

\Needspace{8\baselineskip}
Substituting $w_j=\varepsilon\cosh(u_j)\omega_j$ into the
gradient-flow equations and using ReLU homogeneity gives
\begin{align*}
    \dot u_j
    ={}&
    \Phi({\omega_j})
    -
    \mathbb E\left[
        f_\theta({x})
        \phi({\omega_j}^\top{x})
    \right],\\
    {\dot\omega_j}
    ={}&
    \tanh(u_j)
    \Bigg(
        \nabla_{\mathbb S^{d-1}}\Phi({\omega_j})\\
    &\hspace{25mm}
        -
        \bigl(I-{\omega_j\omega_j^\top}\bigr)
        \mathbb E\left[
            f_\theta({x})
            \mathbf 1_{\{{\omega_j}^\top
              {x}>0\}}{x}
        \right]
    \Bigg).
\end{align*}

We justify the uniform amplitude bound before estimating the interaction
terms. Put
\[
    A_x:=\bigl(\mathbb E\|x\|_2^2\bigr)^{1/2}=\sqrt{d+1}.
\]
Population gradient flow decreases the loss, and $f_{\theta(0)}=0$ and
$\mathbb E[y^2]=1$. Hence, on its interval of existence,
\[
    \|y-f_{\theta(t)}\|_{L^2}\leq1.
\]
Using the exact $u_j$ equation and Cauchy--Schwarz gives
\[
    |\dot u_j|
    =\left|\mathbb E[(y-f_\theta(x))\phi(\omega_j^\top x)]\right|
    \leq A_x.
\]
Therefore $|u_j(t)|\leq A_xT$ on $[0,T]$, uniformly in $m$ and
$\varepsilon$, and
\[
    |a_j(t)|+\|w_j(t)\|_2
    \leq\varepsilon e^{A_xT}.
\]
These bounds also justify continuation to every finite $T$. Indeed,
$\|w_j(t)\|_2\geq\varepsilon$ by the balance invariant. On the region
where hidden weights are nonzero, the population vector field is locally
Lipschitz. For the target term this follows from the smooth homogeneous
extension of $\Phi$. For the interaction terms, unit normals $\omega,\omega'$
satisfy
\[
    \mathbb E\!\left[\|x\|_2^2
       \left|\mathbf1_{\{\omega^\top x>0\}}
              -\mathbf1_{\{(\omega')^\top x>0\}}\right|\right]
    \leq C_d\|\omega-\omega'\|_2.
\]
Indeed, the mixture density is bounded above by
$(2\pi)^{-d/2}\exp(- (\|x\|_2-1)^2/2)$, and polar-coordinate
integration over the disagreement wedge gives its angle times a finite
radial moment. Combining this estimate with
$|\phi(v^\top x)|\leq\|v\|_2\|x\|_2$ and the Lipschitz property of
ReLU proves the claimed local Lipschitz continuity.
Thus the parameters stay in a compact subset of this region on each
finite interval, and the local solution extends. The teacher-only
solution likewise satisfies $|\bar u_j(t)|\leq A_xT$.

The prediction bound is consequently
\[
    \|f_{\theta(t)}\|_{L^2(P_x)}
    \leq A_x\sum_{j=1}^m|a_j(t)|\|w_j(t)\|_2
    \leq A_xm\varepsilon^2 e^{2A_xT}.
\]
The two interaction terms in the $(u_j,\omega_j)$ equations have norms
at most $A_x\|f_{\theta(t)}\|_{L^2(P_x)}$, since
$|\tanh u_j|\leq1$ and $\|I_d-\omega_j\omega_j^\top\|_{\mathrm{op}}\leq1$.
They are therefore bounded by $C_Tm\varepsilon^2$ on $[0,T]$.

By Lemma~\ref{lem:cubic-correlation-identity}, the teacher-only vector
field is Lipschitz on $[-A_xT,A_xT]\times\mathbb S^{d-1}$, with a
constant $\ell_T<\infty$. Let
\[
    D(t):=\max_{j\in[m]}
       \left(|u_j(t)-\bar u_j(t)|
               +\|\omega_j(t)-\bar\omega_j(t)\|_2\right).
\]
The identical initial conditions and the interaction bounds give
\[
    D(t)\leq\ell_T\int_0^tD(s)\,ds+C_Tm\varepsilon^2t,
    \qquad 0\leq t\leq T.
\]
Gr\"onwall's inequality now yields
\[
    \sup_{0\leq t\leq T}D(t)
    \leq C_TT e^{\ell_TT}m\varepsilon^2.
\]
Taking $L_T=C_TT e^{\ell_TT}$ proves the claim. These constants may
depend on $K$ and $T$, but not on the initialization scale.
\end{proof}

\subsection{Proof of Theorem~\ref{thm:rfm-mlp-gap}}
\label{app:proof-mlp-rfm-separation}

We first prove the positive MLP result in a more general fixed-dimensional
multi-index setting. We then recover the single-index MLP result in
Theorem~\ref{thm:rfm-mlp-gap} as a corollary and prove the kernel and RFM lower bounds in that
single-index setting.

\subsubsection{A General Multi-Index Specialized MLP Construction}

Fix an integer $r_0\geq1$ and a constant $R\geq1$. For each $K$, let
$d=K+Kr_0$. Let $E_c\in\mathbb R^{Kr_0\times r_0}$ embed
$\mathbb R^{r_0}$ into the $c$th block of $\mathbb R^{Kr_0}$, so that
\[
    E_c^\top E_{c'}=\mathbf 1_{\{c=c'\}}I_{r_0}.
\]
Define
\[
    \mu_c=
    \begin{bmatrix}
        Re_c\\
        0
    \end{bmatrix}
    \in\mathbb R^d,
    \qquad
    V_c=
    \begin{bmatrix}
        0\\
        E_c
    \end{bmatrix}
    \in\mathbb R^{d\times r_0}.
\]
The data distribution is
\[
    C\sim\operatorname{Unif}([K]),
    \qquad
    X=\mu_C+
    \begin{bmatrix}
        0\\
        Z
    \end{bmatrix},
    \qquad
    Z\sim\mathcal N(0,I_{Kr_0}).
\]
Let $g:\mathbb R^{r_0}\to\mathbb R$ be fixed, bounded, and Lipschitz, with
\[
    \|g\|_\infty\leq G_0,
    \qquad
    |g(u)-g(u')|\leq L\|u-u'\|_2,
\]
and define
\[
    f_{r_0}^\star(X)=g(V_C^\top X).
\]
Thus, within cluster $c$, the target may depend on all $r_0$ coordinates in
the cluster-specific predictive block $V_c^\top X$.

For a two-layer ReLU network, define the path norm
\[
    \|f\|_{\mathcal P}
    =\inf\left\{
        \sum_j|a_j|\|w_j\|_2:
        f(x)=\sum_j a_j\phi(w_j^\top x)
    \right\}.
\]
For a width-$m$ parameterization
$\theta=((a_j,w_j))_{j=1}^m$, define
\[
    \|\theta\|_{\mathrm F}^2
    :=\frac12\sum_{j=1}^m
        \left(a_j^2+\|w_j\|_2^2\right),
\]
and let
\[
    \mathcal F_{B,m}
    :=
    \left\{
        f_\theta(x)=\sum_{j=1}^m a_j\phi(w_j^\top x):
        \|\theta\|_{\mathrm F}^2\leq B
    \right\}.
\]

Define the ramp function
\[
    \psi(t):=\phi(t)-\phi(t-1).
\]
Every occurrence of $\psi$ below is implemented by two ordinary ReLU
neurons; the network itself has neither hidden biases nor output clipping.

\begin{lemma}[Ridge-ramp approximation]
\label{lem:ridge-ramp-density}
Let $\gamma_{r_0}=\mathcal N(0,I_{r_0})$. Finite ridge-ramp expansions of the form
\[
    p(u)=b+\sum_{j=1}^J a_j\psi(q_j^\top u-t_j)
\]
are dense in $L^2(\gamma_{r_0})$. Consequently, there is a sequence
$(p_s)_{s\geq1}$ of such functions satisfying
\[
    \|p_s-g\|_{L^2(\gamma_{r_0})}\longrightarrow0.
\]
\end{lemma}

\begin{proof}
Suppose $h\in L^2(\gamma_{r_0})$ is orthogonal to every ridge-ramp function,
and define the finite signed measure
\[
    \nu(A):=\int_A h(u)\,d\gamma_{r_0}(u).
\]
For every $q\neq0$, $t\in\mathbb R$, and $\lambda>0$,
\[
    \int h(u)\psi\!\left(\lambda(q^\top u-t)\right)
        \,d\gamma_{r_0}(u)=0.
\]
Because $0\leq\psi\leq1$, dominated convergence gives, as
$\lambda\to\infty$,
\[
    \nu\bigl(\{u:q^\top u>t\}\bigr)=0.
\]
Halfspaces determine finite Borel measures on $\mathbb R^{r_0}$, so
$\nu=0$ and therefore $h=0$ in $L^2(\gamma_{r_0})$. This proves density.
\end{proof}

\begin{lemma}[Routing a multi-index approximant]
\label{lem:routed-multi-index-approximant}
\label{lem:exact-specialized-mlp}
Let
\[
    p(u)=b+\sum_{j=1}^J a_j\psi(q_j^\top u-t_j).
\]
There is a two-layer ReLU network
\[
    F_{K,p}=\sum_{c=1}^K F_{K,p,c}
\]
with no hidden biases such that, on an input from cluster $c_0$,
\[
    F_{K,p,c}(x)
    =\mathbf 1_{\{c=c_0\}}p(V_c^\top x).
\]
In particular, each cluster-specific group cancels on all other clusters.
Every hidden weight in $F_{K,p,c}$ lies in
\[
    \operatorname{span}\bigl(\{\mu_c\}\cup\operatorname{range}(V_c)\bigr).
\]
The network has width at most $K(2J+1)$ and path norm at most
\[
    K\,\mathcal C_R(p),
\]
where
\begin{align*}
    \mathcal C_R(p)
    &:={|b|\over R}\\
    &\quad+
    \sum_{j=1}^J |a_j|
    \left[
        \sqrt{\|q_j\|_2^2+{t_j^2\over R^2}}
        +
        \sqrt{\|q_j\|_2^2+{(t_j+1)^2\over R^2}}
    \right].
\end{align*}
It also admits a parameterization satisfying
\[
    \|\theta\|_{\mathrm F}^2\leq K\,\mathcal C_R(p).
\]
\end{lemma}

\begin{proof}
Define
\begin{align*}
    F_{K,p,c}(x)
    :=\;&
        b\,\phi\!\left({\mu_c^\top x\over R^2}\right)\\
    &+\sum_{j=1}^J a_j
        \phi\!\left(
            \left(V_cq_j-{t_j\over R^2}\mu_c\right)^\top x
        \right)\\
    &-\sum_{j=1}^J a_j
        \phi\!\left(
            \left(V_cq_j-{t_j+1\over R^2}\mu_c\right)^\top x
        \right).
\end{align*}
On the support of the data distribution,
\[
    {\mu_c^\top x\over R^2}=\mathbf1_{\{C=c\}}.
\]
Suppose $x$ belongs to cluster $c_0$. If $c\neq c_0$, the two ReLUs in
each pair have the same argument and cancel, while the routing neuron is
zero. If $c=c_0$, then
\[
    F_{K,p,c_0}(x)
    =b+\sum_{j=1}^J a_j
      \left[
        \phi(q_j^\top V_{c_0}^\top x-t_j)
        -\phi(q_j^\top V_{c_0}^\top x-t_j-1)
      \right]
    =p(V_{c_0}^\top x).
\]
This proves the specialization claim.

The routing weight has norm $1/R$. The two hidden weights associated with
$(a_j,q_j,t_j)$ have norms
\[
    \sqrt{\|q_j\|_2^2+{t_j^2\over R^2}}
    \quad\text{and}\quad
    \sqrt{\|q_j\|_2^2+{(t_j+1)^2\over R^2}}.
\]
Summing the corresponding path-norm contributions over all clusters gives
the stated bound. Finally, positive homogeneity allows each neuron to be
balanced so that $|a_j|=\|w_j\|_2$. For this parameterization, its
contribution to $\|\theta\|_{\mathrm F}^2$ equals its contribution to the
path norm.
\end{proof}

The finite-network path norm is the discrete analogue of the Barron or
variation norm for two-layer networks
\citep{bach2017breaking,e2022barron}.
The following sample-dependent estimate is a version of standard
path-norm Rademacher-complexity bounds
\citep{e2022barron,golowich2018size}.

\begin{lemma}[Path-norm generalization without clipping]
\label{lem:path-rademacher}
For any deterministic sample $x_1,\ldots,x_n\in\mathbb R^d$,
\[
    \widehat{\mathfrak R}_n(\mathcal F_{B,m})
    \leq
    {2B\over n}
    \left(\sum_{i=1}^n\|x_i\|_2^2\right)^{1/2}.
\]
For the multi-index distribution above, let
$\widehat f_{\mathrm{MLP}}^{(r_0)}$ be an empirical squared-loss minimizer
over $\mathcal F_{B,m}$. There is a universal constant $C<\infty$ such that
\begin{align*}
    \mathbb E\left[
        \|\widehat f_{\mathrm{MLP}}^{(r_0)}-f_{r_0}^\star
        \|_{L^2(P_x)}^2
    \right]
    &\leq
    \inf_{f\in\mathcal F_{B,m}}
    \|f-f_{r_0}^\star\|_{L^2(P_x)}^2\\
    &\quad+
    CB^2\sqrt{
        {(Kr_0+R^2)(Kr_0+R^2+\log(en))\over n}
    }\\
    &\quad+
    CBG_0\sqrt{{Kr_0+R^2\over n}}.
\end{align*}
\end{lemma}

\begin{proof}
By the arithmetic--geometric mean inequality, every
$f\in\mathcal F_{B,m}$ has path norm at most $B$. By positive homogeneity,
every such function can be represented as
\[
    f(x)=\sum_j c_j\phi(\omega_j^\top x),
    \qquad
    \|\omega_j\|_2=1,
    \qquad
    \sum_j|c_j|\leq B.
\]
Thus $\mathcal F_{B,m}$ is contained in $B$ times the absolutely convex hull
of unit-norm ReLU atoms. Contraction and Cauchy--Schwarz give
\begin{align*}
    \widehat{\mathfrak R}_n(\mathcal F_{B,m})
    &\leq
    {B\over n}\mathbb E_\epsilon
    \sup_{\|\omega\|_2\leq1}
    \left|
        \sum_{i=1}^n\epsilon_i\phi(\omega^\top x_i)
    \right|\\
    &\leq
    {2B\over n}\mathbb E_\epsilon
    \left\|\sum_{i=1}^n\epsilon_i x_i\right\|_2\\
    &\leq
    {2B\over n}
    \left(\sum_{i=1}^n\|x_i\|_2^2\right)^{1/2}.
\end{align*}

For the squared-loss bound, let
\[
    M_n:=\max_{1\leq i\leq n}\|x_i\|_2.
\]
The path-norm bound implies
\[
    |f(x_i)|\leq BM_n
    \qquad\text{for every }f\in\mathcal F_{B,m}.
\]
Since $|y_i|\leq G_0$, the squared loss is
$2(BM_n+G_0)$-Lipschitz in the prediction over the range attained on the
sample. The standard ERM symmetrization and contraction argument therefore
gives
\begin{align*}
    \mathbb E\left[
        \|\widehat f_{\mathrm{MLP}}^{(r_0)}-f_{r_0}^\star
        \|_{L^2(P_x)}^2
    \right]
    &\leq
    \inf_{f\in\mathcal F_{B,m}}
    \|f-f_{r_0}^\star\|_{L^2(P_x)}^2\\
    &\quad+
    C\,\mathbb E\left[
        {B(BM_n+G_0)\over n}
        \left(\sum_{i=1}^n\|x_i\|_2^2\right)^{1/2}
    \right].
\end{align*}
Under the multi-index distribution,
\[
    \|x_i\|_2^2=R^2+\|Z_i\|_2^2,
    \qquad
    Z_i\sim\mathcal N(0,I_{Kr_0}).
\]
Hence
\[
    \mathbb E\sum_{i=1}^n\|x_i\|_2^2=n(Kr_0+R^2),
\]
and Gaussian norm concentration gives
\[
    \mathbb E[M_n^2]
    \leq C\bigl(Kr_0+R^2+\log(en)\bigr).
\]
The result follows from Cauchy--Schwarz.
\end{proof}

\begin{theorem}[Specialized multi-index MLP approximation and consistency]
\label{thm:multi-index-mlp-consistency}
Fix $r_0$, $R$, and a bounded Lipschitz function
$g:\mathbb R^{r_0}\to\mathbb R$ as above. There is a sequence of
cluster-specialized two-layer ReLU networks $(F_K)$ satisfying
\[
    m(F_K)\leq \left\lceil K^{9/8}\right\rceil,
    \qquad
    \|\theta(F_K)\|_{\mathrm F}^2\leq K^{9/8},
\]
and
\[
    \|F_K-f_{r_0}^\star\|_{L^2(P_x)}^2\longrightarrow0.
\]
Moreover, if
\[
    m_K=\left\lceil K^{9/8}\right\rceil,
    \qquad
    B_K=K^{9/8},
    \qquad
    n_K=\left\lceil K^7\right\rceil,
\]
then any empirical squared-loss minimizer over
$\mathcal F_{B_K,m_K}$ satisfies
\[
    \lim_{K\to\infty}
    \mathbb E\left[
        \|\widehat f_{\mathrm{MLP}}^{(r_0)}-f_{r_0}^\star
        \|_{L^2(P_x)}^2
    \right]
    =0.
\]
\end{theorem}

\begin{proof}
By Lemma~\ref{lem:ridge-ramp-density}, choose ridge-ramp networks
$(p_s)_{s\geq1}$ such that
\[
    \|p_s-g\|_{L^2(\gamma_{r_0})}\longrightarrow0.
\]
Write $J_s$ for the number of ridge-ramp units in $p_s$, and define
\[
    D_s:=\max\left\{2J_s+1,\mathcal C_R(p_s),1\right\},
    \qquad
    A_s:=\max_{1\leq j\leq s}D_j.
\]
Let $p_0=0$ and set
\[
    s_K
    :=\max\left(
        \{0\}\cup
        \{s\leq K:A_s\leq K^{1/8}\}
      \right).
\]
Every $A_s$ is finite, so $s_K\to\infty$.

Apply Lemma~\ref{lem:routed-multi-index-approximant} to $p_{s_K}$. The
resulting network $F_K$ has width at most
\[
    K(2J_{s_K}+1)\leq K^{9/8}
\]
and admits a parameterization satisfying
\[
    \|\theta(F_K)\|_{\mathrm F}^2
    \leq K\mathcal C_R(p_{s_K})
    \leq K^{9/8}.
\]
It is cluster-specialized in the sense of
Lemma~\ref{lem:routed-multi-index-approximant}, and
\[
    \|F_K-f_{r_0}^\star\|_{L^2(P_x)}^2
    =\|p_{s_K}-g\|_{L^2(\gamma_{r_0})}^2
    \longrightarrow0.
\]

Applying Lemma~\ref{lem:path-rademacher} with
$B_K=K^{9/8}$ and $n_K=\lceil K^7\rceil$ gives
\begin{align*}
    \mathbb E\left[
        \|\widehat f_{\mathrm{MLP}}^{(r_0)}-f_{r_0}^\star
        \|_{L^2(P_x)}^2
    \right]
    &\leq
    \|F_K-f_{r_0}^\star\|_{L^2(P_x)}^2\\
    &\quad+
    CB_K^2\sqrt{
        {(Kr_0+R^2)(Kr_0+R^2+\log(en_K))\over n_K}
    }\\
    &\quad+
    CB_KG_0\sqrt{{Kr_0+R^2\over n_K}}.
\end{align*}
Because $r_0$ and $R$ are fixed, the two estimation terms are respectively
$O(K^{-1/4})$ and $O(K^{-15/8})$. Both vanish, proving consistency.
\end{proof}

\begin{corollary}[Single-index specialized MLP upper bound]
\label{cor:single-index-mlp-consistency}
Under \cref{data:rfm-mlp-gap}, the target $f^\star$ admits approximation by
cluster-specialized two-layer ReLU networks with polynomially growing width
and squared Frobenius norm and with approximation error tending to zero.
Moreover, the MLP estimator with
\[
    m_K=\left\lceil K^{9/8}\right\rceil,
    \qquad
    B_K=K^{9/8},
    \qquad
    n_K=\left\lceil K^7\right\rceil
\]
is consistent.
\end{corollary}

\begin{proof}
Take $r_0=1$ and $R=1$ in
Theorem~\ref{thm:multi-index-mlp-consistency}. Then $V_c$ reduces to the
single predictive direction $v_c$, and the multi-index distribution reduces
to \cref{data:rfm-mlp-gap}.
\end{proof}

\subsubsection{The Single-Index Kernel and RFM Lower Bounds}

For the remainder of the proof, we return to the single-index setting.
Throughout, $d=2K$ and $R\geq1$ is fixed. We use
\[
    \mu_c=[Re_c;0],
    \qquad
    v_c=[0;e_c],
    \qquad
    \Sigma_c=\operatorname{diag}(0_{K\times K},I_K).
\]
Equivalently,
\[
    c\sim\operatorname{Unif}([K]),
    \qquad
    x=\mu_c+
    \begin{bmatrix}
        0\\
        z
    \end{bmatrix},
    \qquad
    z\sim\mathcal N(0,I_K),
\]
and
\[
    f^\star(x)=g(\langle x,v_c\rangle),
\]
where $g:\mathbb R\to\mathbb R$ is fixed, nonconstant, bounded, and
Lipschitz.

\subsubsection{The Ground-Truth AGOP Removes the Routing Block}

Define the ground-truth population AGOP across all input
coordinates by
\[
    M
    :=\mathbb E\left[
        \nabla_x f^\star(x)\nabla_x f^\star(x)^\top
      \right].
\]

\begin{lemma}[The AGOP kernel loses the cluster identity]
\label{lem:ground-truth-agop-loses-routing}
Let $P_V:=\sum_{c=1}^K v_cv_c^\top$ be the orthogonal projector onto the
predictive subspace, let $G\sim\mathcal N(0,1)$, and set
\[
    \alpha:=\mathbb E[g'(G)^2].
\]
Then $0<\alpha\leq L^2$, and
\[
    M=\frac{\alpha}{K}P_V.
\]
Moreover, every measurable function $h$ of $\sqrt Mx$ satisfies
\[
    \mathbb E\left[\bigl(h(\sqrt Mx)-f^\star(x)\bigr)^2\right]
    \geq
    \left(1-\frac1K\right)\operatorname{Var}(g(G)).
\]
\end{lemma}

\begin{proof}
On the support of cluster $c$, the regression function is
$f^\star(x)=g(\langle x,v_c\rangle)$ and hence, almost everywhere,
\[
    \nabla_x f^\star(x)
    =g'(\langle x,v_c\rangle)v_c.
\]

Since $g$ is Lipschitz, $|g'|\leq L$ almost everywhere, so
$\alpha\leq L^2$. If $\alpha=0$, then $g'=0$ almost everywhere with respect
to Gaussian measure and therefore almost everywhere with respect to Lebesgue
measure. Since a Lipschitz function is absolutely continuous, this would make
$g$ constant. Hence $\alpha>0$.

Averaging the corresponding outer product over the input and the uniform
cluster index gives
\[
    M
    =\frac1K\sum_{c=1}^K
      \mathbb E\left[g'(\langle x,v_c\rangle)^2\mid c\right]v_cv_c^\top
    =\frac{\alpha}{K}\sum_{c=1}^K v_cv_c^\top
    =\frac{\alpha}{K}P_V.
\]
Therefore $\sqrt M=\sqrt{\alpha/K}\,P_V$. In particular, $\sqrt Mx$
contains the predictive coordinates
$S_j:=\langle x,v_j\rangle$, $j\in[K]$, but none of the routing coordinates.
Under the data distribution, $S=(S_1,\ldots,S_K)\sim\mathcal N(0,I_K)$ and is
independent of the uniform cluster index $c$, while
$f^\star(x)=g(S_c)$. Thus, conditional on $\sqrt Mx$,
\[
    \mathbb E[f^\star(x)\mid\sqrt Mx]
    =\frac1K\sum_{j=1}^K g(S_j).
\]
The conditional-expectation characterization of squared-loss regression
implies that every $h(\sqrt Mx)$ has risk at least
\begin{align*}
    \mathbb E\!\left[\operatorname{Var}
        \bigl(f^\star(x)\mid\sqrt Mx\bigr)\right]
    &=\mathbb E\!\left[
        \frac1K\sum_{j=1}^K g(S_j)^2
        -\left(\frac1K\sum_{j=1}^K g(S_j)\right)^2
      \right]\\
    &=\left(1-\frac1K\right)\operatorname{Var}(g(G)).
\end{align*}

Because $g$ is continuous and nonconstant and the Gaussian distribution has
full support, $\operatorname{Var}(g(G))>0$.

\end{proof}

\subsubsection{An All-Orders Lower Bound for Kernel Methods and Regularized RFM}

We next study the full family in \cref{train:rfm-gap}. Let
$P_R:=I_d-P_V$ denote the projector onto the routing subspace. By
Lemma~\ref{lem:ground-truth-agop-loses-routing},
\[
    M_\rho
    =\rho P_R+\left(\rho+\frac{\alpha}{K}\right)P_V.
\]
Thus both the standard metric $I_d$ and the regularized RFM metric $M_\rho$
commute with every rotation of the predictive subspace. The lower bound below
is uniform over $\rho\geq0$, applies to any rotationally invariant base kernel
$\mathcal K$, and is independent of the empirical ridge parameter. 

This degree-by-degree rotational-invariance obstruction is closely related to
high-dimensional results showing that rotationally invariant kernel methods
with polynomially many samples recover only bounded-degree components
\citep{ghorbani2021linearized,ghorbani2020outperform}.
Related group-based analyses of invariant kernel and random-feature models
appear in \citet{mei2021invariances}. More broadly,
\citet{abbe2022nonuniversality} develop symmetry-based limitations for
equivariant learning algorithms.

Let $\operatorname{He}_r$ denote the probabilists' Hermite polynomial and
$h_r=\operatorname{He}_r/\sqrt{r!}$ its normalized version. Write
\[
    g(t)=\sum_{r=0}^\infty\widehat g_r h_r(t),
    \qquad
    \widehat g_r=\mathbb E[g(G)h_r(G)].
\]

\begin{lemma}[Bounded nonconstant links have infinite Hermite support]
\label{lem:infinite-hermite-support}
\label{lem:clipped-ramp-hermite}
The Hermite coefficients of $g$ are not eventually zero. Equivalently,
\[
    \sum_{r>r_0}\widehat g_r^{\,2}>0
    \qquad\text{for every finite }r_0.
\]
\end{lemma}

\begin{proof}
Suppose instead that, for some finite $r_0$,
\[
    \widehat g_r=0
    \qquad\text{for all }r>r_0.
\]
Then, in $L^2(\mathcal N(0,1))$,
\[
    g(t)=\sum_{r=0}^{r_0}\widehat g_rh_r(t)=:p(t),
\]
where $p$ is a polynomial. Both $g$ and $p$ are continuous, and the Gaussian
distribution has full support. Hence equality almost everywhere implies
$g(t)=p(t)$ for every $t\in\mathbb R$. But $g$ is bounded, so $p$ is a
bounded polynomial and must be constant, contradicting the assumption that
$g$ is nonconstant. Parseval's identity then gives the stated positive-tail
property.
\end{proof}

For $r\geq2$, let $\mathcal V_{r,K}$ be the traceless order-$r$ Gaussian-chaos
space. Equivalently, under the standard isometry between the $r$th Gaussian
chaos and symmetric order-$r$ tensors, $\mathcal V_{r,K}$ corresponds to the
kernel of the tensor trace map. It is the irreducible $O(K)$ representation
of harmonic degree-$r$ polynomials.

\begin{lemma}[Harmonic dimension and directional energy]
\label{lem:harmonic-order-r}
The dimension of $\mathcal V_{r,K}$ is
\[
    H_{r,K}
    =\binom{K+r-1}{r}-\binom{K+r-3}{r-2}.
\]
Let $P_{r,K}$ be orthogonal projection onto $\mathcal V_{r,K}$. For every
unit vector $q\in\mathbb R^K$,
\[
    \left\|P_{r,K}h_r(q^\top z)\right\|_{L^2}^2
    =\alpha_{r,K},
\]
where, writing $r=2m$ or $r=2m+1$,
\[
    \alpha_{2m,K}
    =\prod_{j=0}^{m-1}
      \frac{K+2j-1}{K+2m+2j-2},
\]
and
\[
    \alpha_{2m+1,K}
    =\prod_{j=0}^{m-1}
      \frac{K+2j-1}{K+2m+2j}.
\]
For every fixed $r$, $H_{r,K}=\Theta_r(K^r)$ and
$\alpha_{r,K}\to1$ as $K\to\infty$.
\end{lemma}

\begin{proof}
The dimension of symmetric order-$r$ tensors is
$\binom{K+r-1}{r}$. The trace map onto symmetric order-$(r-2)$ tensors is
surjective, and its kernel is the traceless subspace. This gives the stated
dimension.

Under the Gaussian-chaos/tensor isometry,
$h_r(q^\top z)$ corresponds to $q^{\otimes r}$. By rotational
invariance, take $q=e_1$. The orthogonal projection of the homogeneous
polynomial $x_1^r$ onto harmonic degree $r$ is
\[
    \mathcal H_r[x_1^r]
    =\sum_{j=0}^{\lfloor r/2\rfloor}
      \frac{(-1)^j\Gamma(r-j+K/2-1)}
           {4^j j!\,\Gamma(r+K/2-1)}
      \|x\|^{2j}\Delta^j x_1^r.
\]
Since
\[
    \Delta^j x_1^r=\frac{r!}{(r-2j)!}x_1^{r-2j},
\]
the coefficient of $x_1^r$ in this harmonic projection is
\[
    \sum_{j=0}^{\lfloor r/2\rfloor}
      \frac{(-1)^j r!\,\Gamma(r-j+K/2-1)}
           {4^j j!(r-2j)!\,\Gamma(r+K/2-1)}.
\]
In the symmetric-tensor inner product, this coefficient equals
$\langle e_1^{\otimes r},P_{r,K}e_1^{\otimes r}\rangle$, which equals the
squared norm of the projection. Simplifying the finite sum gives the two
product formulas displayed above. The asymptotic statements follow
immediately from those formulas and the dimension expression.
\end{proof}

\begin{lemma}[Invariant random-subspace bound]
\label{lem:random-subspace-bound}
Let $V$ be a $D$-dimensional irreducible orthogonal representation of a compact
group. Let $S\subseteq V$ be a random subspace with invariant law and
$\dim S\leq N$ almost surely. Then, for every fixed $v\in V$,
\[
    \mathbb E\operatorname{dist}(v,S)^2
    \geq\left(1-\frac ND\right)_+\|v\|_2^2.
\]
\end{lemma}

\begin{proof}
Let $\Pi_S$ be orthogonal projection onto $S$. Invariance implies that
$\mathbb E\Pi_S$ commutes with the group action. Schur's lemma gives
$\mathbb E\Pi_S=\beta I_V$. Taking traces yields
$\beta D=\mathbb E\dim S\leq N$. Therefore
\[
    \mathbb E\operatorname{dist}(v,S)^2
    =\|v\|_2^2-\mathbb E\|\Pi_Sv\|_2^2
    =(1-\beta)\|v\|_2^2
    \geq\left(1-\frac ND\right)_+\|v\|_2^2.
\]
\end{proof}

\begin{lemma}[All-orders lower bound for regularized AGOP kernel methods]
\label{lem:all-orders-krr-lower-bound}
For every rotationally invariant kernel $\mathcal K$, every sample size $n$,
every $\rho\geq0$, and every empirical ridge parameter $\lambda_n\geq0$,
both estimators from \cref{train:rfm-gap} satisfy
\[
    \mathbb E\left[
        \|\widehat f-f^\star\|_{L^2(P_x)}^2
    \right]
    \geq
    \sum_{r=2}^\infty
    \widehat g_r^{\,2}\alpha_{r,K}
    \left(1-\frac{n}{H_{r,K}}\right)_+,
    \qquad
    \widehat f\in
    \left\{\widehat f_{\mathrm{Kernel}},\widehat f_{\mathrm{RFM}}\right\}.
\]
\end{lemma}

\begin{proof}
Let $A=I_d$ for $\widehat f_{\mathrm{Kernel}}$ and $A=M_\rho$ for
$\widehat f_{\mathrm{RFM}}$.
By the representer theorem, for a fixed training sample the predictor belongs
to the span of the $n$ transformed-kernel sections centered at the training
inputs. For a test point in cluster $c$, restrict these sections to that
cluster's support and write them as functions of its predictive argument
$\xi\in\mathbb R^K$:
\[
    \psi_{c,i}^{A}(\xi)
    :=\mathcal K\!\left(
        \sqrt A x_i,
        \sqrt A
        \left(\mu_c+\begin{bmatrix}0\\\xi\end{bmatrix}\right)
      \right),
    \qquad 1\leq i\leq n.
\]
At Hermite order $r\geq2$, define
\[
    S_{c,r}
    :=\operatorname{span}\left\{
        P_{r,K}\psi_{c,i}^{A}:1\leq i\leq n
      \right\}
    \subseteq\mathcal V_{r,K}.
\]
Then $\dim S_{c,r}\leq n$. The predictive coordinates of the training inputs
are standard Gaussian, $A$ commutes with predictive rotations, and
$\mathcal K$ is rotationally invariant. Hence the law of $S_{c,r}$ is
invariant under the $O(K)$ action on $\mathcal V_{r,K}$ for both estimators
and every $\rho\geq0$.

The order-$r$ target component in cluster $c$ is
\[
    \gamma_{c,r}
    =\widehat g_rP_{r,K}h_r(e_c^\top\xi),
\]
whose squared norm is $\widehat g_r^{\,2}\alpha_{r,K}$ by
Lemma~\ref{lem:harmonic-order-r}. The projected predictor lies in $S_{c,r}$.
Since distinct Gaussian-chaos orders are orthogonal, for every finite $L_0$,
\[
    \|\widehat f_c-g(e_c^\top\!\cdot)\|_{L^2}^2
    \geq
    \sum_{r=2}^{L_0}\operatorname{dist}(\gamma_{c,r},S_{c,r})^2.
\]
Apply Lemma~\ref{lem:random-subspace-bound} and then let $L_0\to\infty$ by
monotone convergence. This gives, for every cluster $c$,
\[
    \mathbb E\left[
        \|\widehat f_c-g(e_c^\top\!\cdot)\|_{L^2}^2
    \right]
    \geq
    \sum_{r=2}^\infty
    \widehat g_r^{\,2}\alpha_{r,K}
    \left(1-\frac{n}{H_{r,K}}\right)_+.
\]
Averaging over the uniform test cluster proves the result.
\end{proof}

% Keep the short closing lemma and its proof on the same page.
\Needspace{24\baselineskip}
\begin{lemma}[Polynomial sample budgets leave a nonzero kernel/RFM error]
\label{lem:polynomial-kernel-rfm-failure}
Fix $A<\infty$, let $n_K=O(K^A)$, and let $\rho_K\geq0$ be any sequence.
Then, for each
\[
    \widehat f\in
    \left\{\widehat f_{\mathrm{Kernel}},\widehat f_{\mathrm{RFM}}\right\},
\]
\[
    \liminf_{K\to\infty}
    \mathbb E\left[
        \|\widehat f-f^\star\|_{L^2(P_x)}^2
    \right]
    \geq
    \sum_{\substack{r\geq2\\r>A}}\widehat g_r^{\,2}>0.
\]
\end{lemma}

\begin{proof}
Fix an order $r>A$. By Lemma~\ref{lem:harmonic-order-r},
$H_{r,K}=\Theta_r(K^r)$ and $\alpha_{r,K}\to1$, so
\[
    \frac{n_K}{H_{r,K}}\longrightarrow0.
\]
Apply Lemma~\ref{lem:all-orders-krr-lower-bound}. For any finite collection
of orders satisfying $r>A$, each corresponding summand converges to
$\widehat g_r^{\,2}$. Taking the lower limit and then increasing the finite
collection gives, by monotone convergence,
\[
    \liminf_{K\to\infty}
    \mathbb E\left[
        \|\widehat f-f^\star\|_{L^2(P_x)}^2
    \right]
    \geq
    \sum_{\substack{r\geq2\\r>A}}\widehat g_r^{\,2}.
\]

The right-hand side is strictly positive by
Lemma~\ref{lem:infinite-hermite-support}.

\end{proof}

\begin{proof}[Proof of Theorem~\ref{thm:rfm-mlp-gap}]
The specialized MLP approximation and consistency conclusions follow from
Corollary~\ref{cor:single-index-mlp-consistency}. In particular, the choices
\[
    m_K=\left\lceil K^{9/8}\right\rceil,
    \qquad
    B_K=K^{9/8},
    \qquad
    n_K=\left\lceil K^7\right\rceil
\]
are polynomial in $K$ and give
\[
    \mathbb E\left[
        \|\widehat f_{\mathrm{MLP}}-f^\star\|_{L^2(P_x)}^2
    \right]
    \longrightarrow0.
\]

For the lower bound, fix any $A<\infty$, any sequence $n_K=O(K^A)$, and any
sequence $\rho_K\geq0$. Lemma~\ref{lem:polynomial-kernel-rfm-failure} gives
\[
    \liminf_{K\to\infty}
    \mathbb E\left[
        \|\widehat f-f^\star\|_{L^2(P_x)}^2
    \right]
    \geq
    \sum_{\substack{r\geq2\\r>A}}\widehat g_r^{\,2}>0,
    \qquad
    \widehat f\in
    \left\{\widehat f_{\mathrm{Kernel}},\widehat f_{\mathrm{RFM}}\right\}.
\]
The lower bound is uniform over $\rho_K$ and over the empirical ridge
parameters allowed in \cref{train:rfm-gap}. This proves the theorem.
\end{proof}

\end{document}